\documentclass[pmlr,twocolumn,10pt]{jmlr}

\usepackage{booktabs}
\usepackage{siunitx}

\usepackage{amsmath}
\usepackage{amssymb}
\usepackage{mathtools}
\usepackage{dsfont}
\newcommand{\ind}{\mathds{1}}
\usepackage{enumitem}
\usepackage{footmisc}
\usepackage{bm}
\usepackage{multirow}
\usepackage{rotating}
\usepackage{float}
\usepackage{placeins}
\newcommand{\anchor}{\text{A}}
\newcommand{\viz}{\text{V}}
\usepackage{adjustbox}
\usepackage{microtype}

\newcommand{\abovesectionskip}{\vspace{-.3em}}
\newcommand{\belowsectionskip}{\vspace{-.3em}}
\newcommand{\abovesubsectionskip}{\vspace{-.4em}}
\newcommand{\belowsubsectionskip}{\vspace{-.3em}}

\theorembodyfont{\upshape}
\theoremheaderfont{\scshape}
\theorempostheader{:}
\theoremsep{\newline}

\jmlrpages{}
\jmlrvolume{}
\jmlryear{2023}
\jmlrsubmitted{}
\jmlrpublished{}
\jmlrworkshop{Conference on Health, Inference, and Learning (CHIL) 2023}

\title[Visualizing Embedding Spaces of Neural Survival Analysis Models]{A General Framework for Visualizing Embedding Spaces of\titlebreak Neural Survival Analysis Models Based on Angular Information}

\author{
\Name{George H.~Chen}
\Email{georgechen@cmu.edu}\\
\addr Carnegie Mellon University, USA
}

\begin{document}

\maketitle

\begin{abstract}
We propose a general framework for visualizing any intermediate embedding representation used by any neural survival analysis model. Our framework is based on so-called \emph{anchor directions} in an embedding space. We show how to estimate these anchor directions using clustering or, alternatively, using user-supplied ``concepts'' defined by collections of raw inputs (e.g., feature vectors all from female patients could encode the concept ``female''). For tabular data, we present visualization strategies that reveal how anchor directions relate to raw clinical features and to survival time distributions. We then show how these visualization ideas extend to handling raw inputs that are images. Our framework is built on looking at angles between vectors in an embedding space, where there could be ``information loss'' by ignoring magnitude information. We show how this loss results in a ``clumping'' artifact that appears in our visualizations, and how to reduce this information loss in practice.
\end{abstract}

\paragraph*{Data and Code Availability}
We use the publicly available datasets on predicting time until death from the Study to Understand Prognoses, Preferences, Outcomes, and Risks of Treatment (SUPPORT) \citep{knaus1995support}, the Rotterdam tumor bank (Rotterdam) \citep{foekens2000urokinase}, and the German Breast Cancer Study Group (GBSG) \citep{schumacher_1994}. The SUPPORT dataset is on severely ill hospitalized patients with various diseases whereas the Rotterdam and GBSG datasets are both on breast cancer. We also use the MNIST handwritten digits dataset \citep{lecun2010mnist} modified by \citet{polsterl2019survival} to be for survival analysis. Our code is publicly available (links to the datasets we use are in our code): \\
\url{https://github.com/georgehc/anchor-vis/}

\paragraph*{Institutional Review Board (IRB)}
Our research does not require IRB approval as we are conducting secondary analyses of existing publicly available datasets that do not have access restrictions.

\sloppy

\abovesectionskip
\section{Introduction}
\label{sec:intro}
\belowsectionskip

Survival analysis models regularly arise in health applications in reasoning about %
how much time will elapse
before a critical event happens, such as death, disease relapse, and hospital readmission. Across many health-related datasets, state-of-the-art survival analysis models commonly use neural networks
(e.g., \citealt{ranganath2016deep,chapfuwa2018adversarial,katzman2018deepsurv,lee2018deephit,kvamme2019time,nagpal2021deep,chen2020deep,chen2022survival,li2020neural,zhong2021deep,zhong2022deep,manduchi2022deep}). However, most neural survival analysis models have been developed with a focus on prediction accuracy, often without examining what these models have learned internally. In more detail, these models typically represent individual patients in terms of ``embedding vectors''. How do these embedding vectors relate to patient characteristics? How do they relate to survival (or time-to-event) outcomes?

Some existing neural survival analysis models have been designed to have interpretable components. For instance, the model by \citet{zhong2022deep} uses a partially linear Cox model: variables that the modeler wants to easily reason about are captured by a linear component of the model whereas the rest of the variables are modeled by a neural network. Meanwhile, \citet{chapfuwa2020survival}, \citet{nagpal2021deep}, \citet{manduchi2022deep}, and \citet{chen2022survival} all represent a data point (i.e., a patient) in terms of clusters, where we can summarize each cluster's patient characteristics and survival distributions. Along similar lines, \citet{li2020neural} introduced a neural topic model with survival supervision, which represents each patient as a combination of ``topics'', where each topic corresponds to specific patient characteristics being more probable and to either higher or lower survival times.

In this paper, rather than developing a new survival analysis model that aims to be in some sense interpretable, our main contribution is instead to propose a general framework for visualizing intermediate representations of \emph{any} neural survival analysis model. Crucially, our framework is based on analyzing \emph{angular information}. Specifically, our framework uses what we refer to as \emph{anchor directions} in an intermediate representation space. We specifically show:
\begin{itemize}[leftmargin=*,itemsep=0pt,parsep=0pt,topsep=0pt,partopsep=0pt]
\item how to estimate anchor directions based on clustering or, alternatively, based on a ``concept'' that the user provides, where the concept is represented as a collection of data points (e.g., a set of feature vectors all for female patients could represent the concept ``female''; this is the same definition of ``concepts'' as used by \citet{kim2018interpretability});
\item how to visualize raw features vs anchor directions, and survival time distributions vs anchor directions; %
\item how to tell if our visualizations are ``losing too much information'' by focusing on angular information (and ignoring magnitude information), and how to reduce this information loss.
\end{itemize}
Our framework could be thought of as a suite of visualization tools with accompanying statistical tests that can help quantify the strength of associations related to the embedding space under examination.

To showcase our framework, we first focus on a tabular dataset on predicting time until death of patients from the Study to Understand Prognoses, Preferences, Outcomes, and Risks of Treatment (\mbox{SUPPORT}) \citep{knaus1995support}. We then show how our visualization ideas extend to working with images, where we use the semi-synthetic Survival MNIST dataset \citep{lecun2010mnist,polsterl2019survival} with known ground truth structure. Our visualizations help us see to how well a neural network's embedding space captures this known structure. We provide a second tabular data example on survival times of breast cancer patients in Appendix~\ref{sec:rotterdam-gbsg}, using data from the Rotterdam tumor bank (Rotterdam) \citep{foekens2000urokinase} and the German Breast Cancer Study Group (GBSG) \citep{schumacher_1994}.

The only baseline visualization strategy we are aware of that works with any intermediate representation of any neural survival analysis model is to apply a dimensionality reduction method (such as PCA \citep{pearson1901lines} or t-SNE \citep{van2008visualizing}) to transform the intermediate representation of interest (which could be high-dimensional) into a 1D, 2D, or 3D representation that is displayed in a scatter plot. Points in the scatter plot could be colored based on, for instance, their median survival times as predicted by the neural survival analysis model. We provide examples of such scatter plots in Appendix~\ref{sec:baseline-visualization}.
While this baseline visualization strategy can provide some rough intuition of the geometry of the embedding space, it does not---on its own---do clustering or provide quantitative metrics relating the embedding space to raw features or to predicted survival time distributions. %
Our framework could be used in addition to this baseline visualization strategy and does
relate the embedding space to raw features and to survival time distributions,
with the help of anchor directions based on clustering or on concepts.

Separately, even though many visualization tools have been developed for neural network models for classification or for predicting a single scalar output (e.g., \citealt{selvaraju2016grad,zhou2016learning,dabkowski2017real,lundberg2017unified,shrikumar2017learning,smilkov2017smoothgrad,sundararajan2017axiomatic,kim2018interpretability}), these existing visualization tools do not easily extend to the survival analysis setting. A key reason is that these tools aim to quantify how important different input features are in affecting a single output neuron's scalar value. However, in general, the prediction target in survival analysis for a single test data point is not a single scalar value and is instead a probability distribution over time, where time could either be continuous or discrete. When the time is discrete, the number of time steps used is up to the modeler and could even scale with the number of training data points. Quantifying the importance of different input features in predicting such a survival time distribution is not straightforward.

\abovesectionskip
\section{Background}
\belowsectionskip

We review the standard survival analysis setup in Section~\ref{sec:survival-analysis-background}, and then we provide an example of a neural survival analysis model in Section~\ref{sec:deepsurv}. For the latter, we specifically review the now-standard DeepSurv model \citep{katzman2018deepsurv}.
We emphasize that our visualization framework is not limited to only working with DeepSurv and works with any neural survival analysis model.
For ease of exposition, we use DeepSurv throughout the paper.

\abovesubsectionskip
\subsection{Survival Analysis}
\label{sec:survival-analysis-background}
\belowsubsectionskip

We assume that we have $n$ training patients with data points $\{(x_i,y_i,\delta_i)\}_{i=1}^n$, where the $i$-th training patient has raw input $x_i\in\mathcal{X}$ (e.g., tabular data, images), observed time $y_i\in[0,\infty)$, and event indicator $\delta_i\in\{0,1\}$; if $\delta_i=1$, then this means that the critical event of interest (e.g., death) happened for the $i$-th training patient so $y_i$ is the time until the critical event happened (also called the ``survival time''), whereas if $\delta_i=0$, then this means that the critical event did not happen, so $y_i$ is the time until we stopped collecting information on the $i$-th patient (commonly called the ``censoring time'').

To model how training points are sampled, we first introduce some notation. We denote the random variable for a generic raw input as $X$, which has marginal distribution $\mathbb{P}_X$. We denote the random variable for the survival time as $T$, which depends on raw input~$X$; in particular, there is a conditional distribution $\mathbb{P}_{T | X}$. We also denote the random variable for the censoring time as $C$, which depends on raw input $X$ via a conditional distribution~$\mathbb{P}_{C | X}$. Note that $T$ and $C$ are assumed to be conditionally independent given $X$. Then each training point $(x_i,y_i,\delta_i)$ is assumed to be generated i.i.d.~as follows:
\begin{enumerate}[leftmargin=*,itemsep=0pt,parsep=0pt,topsep=1pt]
\item Sample raw input $x_i$ from $\mathbb{P}_X$.
\item Sample survival time $t_i$ from $\mathbb{P}_{T\mid X=x_i}$.
\item Sample censoring time $c_i$ from $\mathbb{P}_{C\mid X=x_i}$.
\item If $t_i \le c_i$ (the critical event happens before censoring): set $y_i = t_i$ and $\delta_i=1$. Otherwise, set $y_i = c_i$ and $\delta_i=0$.
\end{enumerate}
For a patient with raw input $x\in\mathcal{X}$, a survival analysis model estimates a distribution over survival times for~$x$. Specifically, this survival time distribution is specified in terms of the \emph{conditional survival function}
\[
S(t|x) \triangleq \mathbb{P}(T>t\mid X=x)\quad\text{for }t\ge0,
\]
or a transformed version of this function, such as the \emph{hazard function} $h(t|x) \triangleq -\frac{\partial}{\partial t}\log S(t|x)$ (by negating, integrating, and exponentiating, one can show that $S(t|x) = \exp(-\int_0^t h(\tau|x)d\tau)$, i.e., having an estimate for $h(\cdot|x)$ yields an estimate for $S(\cdot|x)$).

\abovesubsectionskip
\subsection{Neural Survival Analysis: DeepSurv}
\label{sec:deepsurv}
\belowsubsectionskip

The DeepSurv model \citep{katzman2018deepsurv} %
estimates the hazard function $h(t|x)$ under the standard \emph{proportional hazards assumption} \citep{cox1972regression}:
\begin{equation}
h(t|x) = h_0(t) \exp( f(x; \theta) ),
\label{eq:proportional-hazards}
\end{equation}
where $h_0$ is called the \emph{baseline hazard function} (which takes as input a nonnegative time $t\ge0$ and outputs a nonnegative value), and $f(\cdot;\theta)$ is a user-specified neural network with parameters $\theta$ (specifically, $f(\cdot;\theta)$ maps a raw input from $\mathcal{X}$ to a single real number that could be thought of as a ``risk score'', where higher values correspond to the critical event of interest likely happening earlier for feature vector $x$). For example, when working with tabular data, $f$ could be a multilayer perceptron, and when working with images, $f$ could be a convolutional neural network.

To train a DeepSurv model, we first learn the neural network parameters~$\theta$ by minimizing the loss %
\[
L(\theta) \triangleq
-\sum_{i=1}^n
\delta_i
\bigg[ f(x_i; \theta) - \log \!\!
\sum_{\substack{j=1,\dots,n\\\text{s.t.~}y_j \ge y_i}}\!\!
\exp(f(x_j; \theta)) \bigg].
\]
After learning $\theta$, we then estimate $h_0$ using a standard approach such as Breslow's method \citep{breslow1972discussion}. Specifically, let $\widehat{\theta}$ denote the learned neural network parameters, let $t_{(1)}<t_{(2)}<\cdots<t_{(\tau)}$ denote the sorted unique times when the critical event happened in the training data (so that the total number of these unique times is $\tau$), and let $d_{(\ell)}$ denote the number of times the critical event happened at time $t_{(\ell)}$ for $\ell=1,\dots,\tau$. Then Breslow's method estimates a discretized version of $h_0$ at time $t_{(\ell)}$ using
\[
\widehat{h}_0(t_{(\ell)}) \triangleq \frac{d_{(\ell)}}{\sum_{j=1,\dots,n\text{ s.t.~}y_j\ge t_{(\ell)}}\exp(f(x_j; \widehat{\theta}~\!))}.
\]
The conditional survival function $S(t|x)$ can then be estimated by
\begin{align}
\widehat{S}(t|x) \triangleq \exp\bigg\{-\exp(f(x; \widehat{\theta}~\!)) \sum_{\substack{\ell=1,\dots,\tau\\\text{ s.t.~}t_{(\ell)}\le t}} \widehat{h}_0(t_{(\ell)}) \bigg\}. \nonumber \\[-1.5em]
\label{eq:deepsurv-surv}
\end{align}
Note that here, the conditional survival function is estimated along a time grid with a total of $\tau$ discretized time points, where $\tau$ could scale with the number of training points.

\abovesectionskip
\section{Visualization Framework}
\label{sec:framework}
\belowsectionskip

We now present our visualization framework based on anchor directions. Our framework can work with any neural survival analysis model with a base neural network $f$ and an estimate $\widehat{S}(t|x)$ of the conditional survival function $S(t|x)$. For the rest of the paper, we treat the neural survival analysis model that we aim to provide visualizations for as fixed, meaning that it has already been learned using the training data $\{(x_i,y_i,\delta_i)\}_{i=1}^n$ and its learned parameters~$\widehat{\theta}$ will not be modified. Thus, for notational convenience, we now write ``$f(x)$'' instead of ``$f(x;\widehat{\theta}~\!)$''.

At a high level, our framework works as follows. %
First, we need to decide what intermediate ``embedding space'' of the base neural network $f$ to visualize (Section~\ref{sec:choose-embedding-space}). Next, we estimate anchor directions in the chosen embedding space (Section~\ref{sec:anchor-directions}). Our visualizations will then be based on how well embedding vectors align with anchor directions in terms of cosine similarity (Section~\ref{sec:anchor-projections}). We explain our visualization strategies via real data examples, first with tabular data (Section~\ref{sec:tabular}) and then with images (Section~\ref{sec:images}). Our visualization strategy focuses on using angular and not magnitude information within the embedding space, which could result in ``information loss'' in our visualizations. We provide details on this information loss and how to mitigate it (Section~\ref{sec:information-loss}).

\smallskip
\noindent
\textbf{Statistical assumptions and sample splitting.}
Our exposition will be clear about statistical assumptions, which ensure that the different steps of our visualization framework are theoretically sound (and we will be clear when a step is a heuristic and lacks theoretical justification). For example, we occasionally use statistical tests, which commonly require that the input data to the tests are i.i.d. With such considerations in mind, we assume that we have access to additional data separate from the training data $\{(x_i,y_i,\delta_i)\}_{i=1}^n$: specifically, we assume that we also have ``anchor direction estimation'' data $\{(x_i^{\anchor},y_i^{\anchor},\delta_i^{\anchor})\}_{i=1}^{n^{\anchor}}$ %
sampled i.i.d.~in the same manner as the training data, %
and also ``visualization raw inputs'' $\{x_i^{\viz}\}_{i=1}^{n^{\viz}}$ sampled i.i.d.~from $\mathbb{P}_X$ that are separate from both training and anchor direction estimation data.\footnote{Our framework works even if the training data were sampled differently from the rest of the data; see Appendix~\ref{sec:rotterdam-gbsg}.} The basic workflow is as follows: we learn the neural survival analysis model using training data. Afterward, we estimate anchor directions using anchor direction estimation data. Finally,
we produce visualizations based on the estimated anchor directions with the help of visualization raw inputs.

\abovesubsectionskip
\subsection{Which ``Embedding Space'' to Visualize?}
\label{sec:choose-embedding-space}
\belowsubsectionskip

First, we need to specify which space we will try to visualize. To go with the DeepSurv example from earlier, the neural network $f$ used with DeepSurv could be a multilayer perceptron. In this case, we could visualize the representation at one of the intermediate layers such as the representation right before the last fully-connected layer (this last layer outputs a single number corresponding to the risk score). Specifically,
\begin{equation}
f(x) = g(\phi(x)),
\label{eq:neural-net-decomposition}
\end{equation}
where the function $\phi$ maps from the raw input space~$\mathcal{X}$ to some intermediate Euclidean space $\mathbb{R}^d$, and the function $g$ maps from $\mathbb{R}^d$ to $\mathbb{R}$. %
For the rest of the paper, we refer to $\phi$ as the \emph{encoder}, which converts a raw input into an \emph{embedding vector}.\footnote{
Note that for equation~\eqref{eq:neural-net-decomposition}, our framework does \emph{not} require the function $g$ to output a single real number. This happens to be the case for DeepSurv. For example, DeepHit \citep{lee2018deephit} has a base neural network $f$ that outputs a Euclidean vector instead.  Our visualization framework trivially also supports these other base neural network functions.} %

An embedding vector by itself is not very informative. Instead, our visualizations depend on the distribution of these embedding vectors. Recall from Section~\ref{sec:survival-analysis-background} that the distribution of raw inputs is given by $\mathbb{P}_X$. We define the \emph{embedding space} as follows. \vspace{-.45em}
\begin{definition}\label{def:embedding-space}
Let $\mathbb{P}_X$ denote the distribution of raw input data. Suppose that we sample $X\sim\mathbb{P}_X$. Then we set $U=\phi(X)$, where $\phi$ is the encoder. Then we define the \emph{embedding space} as the distribution of $U$, denoted as $\mathbb{P}_U$. This distribution is over $\mathbb{R}^d$.
\end{definition} \vspace{-.45em}
Crucially, we define the embedding space as a distribution over $\mathbb{R}^d$---not just $\mathbb{R}^d$ without a distribution.

To visualize the embedding space $\mathbb{P}_U$, our framework involves first choosing ``interesting'' \emph{anchor directions} in the embedding space to look at, which we discuss in detail in the next section. Importantly, we argue that in practice, the $d$ axis-aligned directions of $\mathbb{R}^d$ are not necessarily the ``interesting'' directions for the application at hand. For example, if $\mathbb{P}_{U}$ is well-approximated by a clustering model with $k$ clusters, then the ``meaningful'' directions to consider could be the directions that point toward the $k$ different cluster centers. These directions need not be axis-aligned.

Moreover, the embedding dimension $d$ is often a hyperparameter to be tuned (as part of the neural net architecture of $f$). If the problem of interest has an underlying ground truth distribution that is based on $k$ clusters, then a ``good'' choice of neural survival analysis model would be one where for a wide range of values for hyperparameter $d$ (where $d\ge k$), after learning the neural survival analysis model, the resulting learned embedding space $\mathbb{P}_U$ should consist of $k$ sufficiently separated clusters within $\mathbb{R}^d$.

\abovesubsectionskip
\subsection{Choosing Anchor Directions}
\label{sec:anchor-directions}
\belowsubsectionskip

As stated previously, we estimate anchor directions using the anchor direction estimation data $\{(x_i^{\anchor},y_i^{\anchor},\delta_i^{\anchor})\}_{i=1}^{n^{\anchor}}$.
We denote the embedding vectors of these points by $u_i^{\anchor} \triangleq \phi(x_i^{\anchor})$. Note that $\phi$ is estimated as part of learning the neural survival analysis model using training data $\{(x_i,y_i,\delta_i)\}_{i=1}^{n}$. 
Conditioned on the original training data and on the encoder $\phi$, the embedding vectors $u_1^{\anchor},\dots,u_{n^{\anchor}}^{\anchor}$ appear as i.i.d.~draws from the embedding space $\mathbb{P}_U$ (this i.i.d.~property would fail to hold if the anchor direction estimation data were the same as the training data; see Appendix~\ref{sec:statistical-considerations-anchor-direction-estimation} for details). We now provide two approaches for estimating anchor directions. %

\abovesubsectionskip
\subsubsection{Clustering Approach}
\label{sec:anchor-directions-via-clustering}
\belowsubsectionskip

We first estimate anchor directions using clustering. %
This approach is agnostic to the choice of clustering algorithm but assumes that the clustering algorithm chosen only has access to the embedding vectors $u_1^{\anchor},\dots,u_{n^{\anchor}}^{\anchor}$. In particular, we ignore the survival information $\{(y_i^{\anchor},\delta_i^{\anchor})\}_{i=1}^{n^{\anchor}}$ for the moment. Note that many clustering algorithms, such as the Expectation-Maximization algorithm for Gaussian mixture models \citep[Section 9.2]{bishop2006pattern}, %
are derived under the assumption that the data points to be clustered are i.i.d. Again, this assumption holds since $u_1^{\anchor},\dots,u_{n^{\anchor}}^{\anchor}$ appear as i.i.d.~samples from $\mathbb{P}_U$ after we condition on the original training data and~$\phi$. %

Once we have learned the clustering model of our choice, we obtain a cluster assignment for each embedding vector. In particular, we denote the cluster assignment of the $i$-th embedding vector $u_i^{\anchor}$ by $z_i^{\anchor} \in \{1,2,\dots,k\}$, where $k$ is the number of clusters.
Then we define the \mbox{$j$-th} cluster's anchor direction as
\begin{equation}
\mu_{\text{cluster~\!}j}
\triangleq
\underbrace{
  \frac{\sum_{i=1}^{n^{\anchor}} u_i^{\anchor} \ind\{z_i^{\anchor} = j\}}
       {\sum_{i=1}^{n^{\anchor}} \ind\{z_i^{\anchor} = j\}}}_{
  \substack{\text{mean direction of}\\\text{embedding vectors in}\\\text{the }j\text{-th cluster}}}
-\!\!\!\!\!\!
\underbrace{\frac{1}{n^{\anchor}}\sum_{i=1}^{n^{\anchor}} u_i^{\anchor}}_{
  \substack{\text{mean direction across}\\\text{embedding vectors }\\\triangleq~\overline{u}^{\anchor}}}\!\!\!\!\!\!\!\!,
\label{eq:cluster-vs-all}
\end{equation}
where $\ind\{\cdot\}$ is the indicator function that is~1 when its input is true and 0 otherwise.
The second
term $\overline{u}^{\anchor}$ is %
the ``center of mass'' of the embedding vectors. %
Thus, by subtracting $\overline{u}^{\anchor}$, we focus on how the clusters differ from the center of mass. This is essentially changing the center of the coordinate system so that directions are measured treating $\overline{u}^{\anchor}$ as the ``origin'' of the new coordinate system. A similar idea is commonly used in principal component analysis, in which data are first centered \citep{abdi2010principal}.

\smallskip
\noindent
\textbf{Choosing the number of clusters based on survival information.}
Clustering algorithms often have a hyperparameter (or multiple) that affects the number of clusters~$k$. For example, a Gaussian mixture model has a hyperparameter for the number of mixture components, which could be thought of as clusters. %
We now propose a heuristic for choosing $k$ using survival information $\{(y_i^{\anchor},\delta_i^{\anchor})\}_{i=1}^{n^{\anchor}}$.

\begin{figure}[t]
\centering
\includegraphics[width=.9\linewidth]{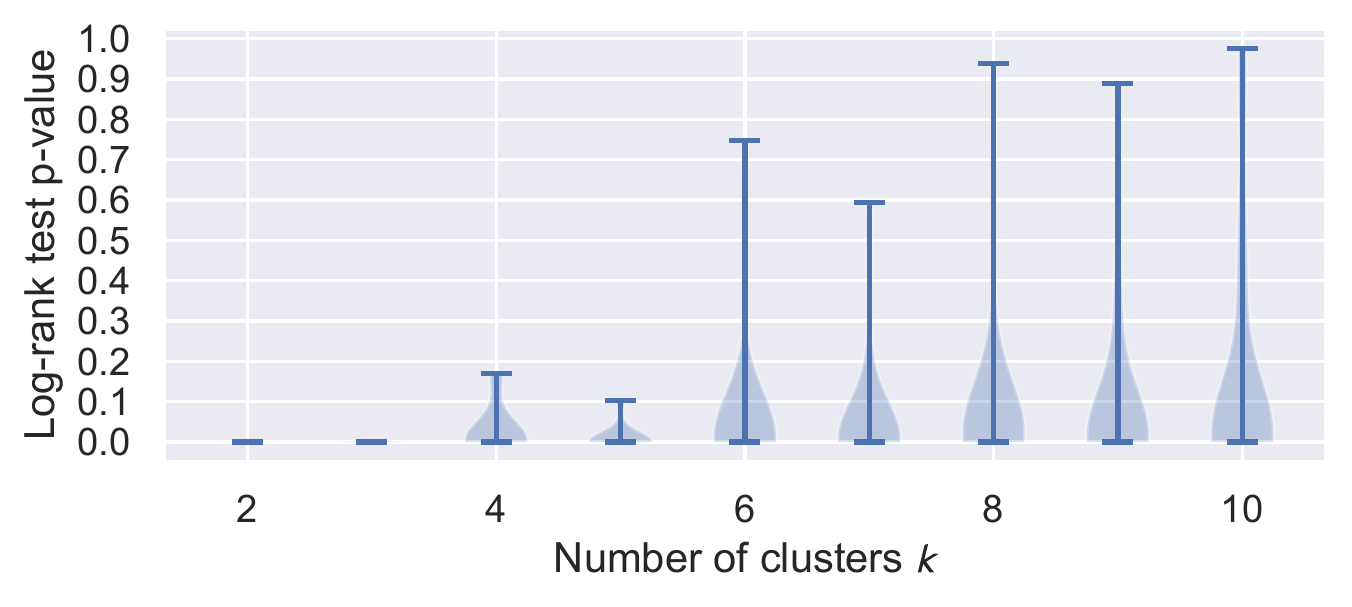}
\vspace{-1em}
\caption{A violin plot to help select the number of clusters~$k$ to use. Here, the encoder is from a DeepSurv model trained on the SUPPORT dataset, and the clustering model is a mixture of von Mises-Fisher distributions. We plot the set $\Psi(k)$ (consisting of log-rank test p-values; see equation~\eqref{eq:logrank-p-values}) vs~$k$.
Details on the dataset, encoder, and clustering model are in Section~\ref{sec:tabular}.}
\label{fig:support-logrank-helper}
\vspace{-1em}
\end{figure}

Suppose that we have grouped the data into $k$ clusters. For any two clusters $j,j'\in\{1,\dots,k\}$ with $j\ne j'$, we can run the log-rank test \citep{mantel1966evaluation} to quantify how different these two clusters' survival outcomes are (running this test requires using the $y_i^{\anchor}$ and $\delta_i^{\anchor}$ variables of the points in clusters~$j$ and~$j'$). We denote the test's resulting p-value as~$\psi_{j,j'}(k)$.
Then the set of p-values across all pairs of clusters found is
\begin{align}
\Psi(k) \triangleq \big\{ \psi_{j, j'}(k)~&\text{for all }j\in\{1,\dots,k\}, \nonumber \\
&\text{and }j'\in\{j+1,\dots,k\} \big\}.
\label{eq:logrank-p-values}
\end{align}
We can re-run the clustering algorithm to get cluster assignments that have different numbers of clusters~$k$ so that for each~$k$, we have a different set of p-values~$\Psi(k)$. We plot the entire set $\Psi(k)$ as a function of $k$ in \figureref{fig:support-logrank-helper} using a violin plot (i.e., for each~$k$, we see the distribution of $\Psi(k)$ as a ``violin''). We can choose $k$ to be a value before the p-values in $\Psi(k)$ become ``too large''. For example, in \figureref{fig:support-logrank-helper}, the ``violins'' in the violin plot get much taller (so the p-values in $\Psi(k)$ are overall getting much larger) after 5 clusters, so we could choose $k=5$.

This procedure for choosing $k$ is a heuristic: we have found it work well in practice but we currently lack theory to justify when it recovers the ``correct'' number of clusters. We comment on when the log-rank test is theoretically sound to apply (so that the p-values are valid) in Appendix~\ref{sec:statistical-considerations-logrank}. We further discuss a heuristic for choosing which clusters to focus on when the number of clusters is large in Appendix~\ref{sec:handling-many-anchor-directions}; this heuristic is based on estimating a survival time for each anchor direction, ranking anchor directions based on these survival time estimates, and focusing only on anchor directions with particular ranks (e.g., anchor directions with the highest and lowest survival times). %

\abovesubsectionskip
\subsubsection{User-supplied ``Concepts''}
\label{sec:anchor-directions-via-concepts}
\belowsubsectionskip

Next, we consider a scenario where the user provides a ``concept'' in terms of a collection of example raw input data $\mathcal{C}\triangleq\{x_1^\dagger,\dots,x_{n^\dagger}^\dagger\}$, where all $n^\dagger$ of these examples exhibit the concept (e.g., to convey a concept corresponding to ``female'', the user can set $\mathcal{C}$ to be a collection of raw inputs from female patients). This is the same notion of concepts as used by \citet{kim2018interpretability}. %
We assume that the set $\mathcal{C}$ is collected in a manner that is independent of how the visualization raw inputs $\{x_i^{\viz}\}_{i=1}^{n^{\viz}}$ are collected (this assumption will be needed later for statistical tests).
Specifically, we use the anchor direction
\begin{equation}
\mu_{\text{concept }\mathcal{C}}
\triangleq
  \frac{1}{|\mathcal{C}|}\sum_{x^{\dagger}\in\mathcal{C}} \phi(x^{\dagger})
  -
  \overline{u}^{\anchor},
\label{eq:concept-vs-all}
\end{equation}
where the center of mass $\overline{u}^{\anchor}$ is defined in equation~\eqref{eq:cluster-vs-all}.

\abovesubsectionskip
\subsection{Key Visualization Quantity: Projections onto Anchor Directions}
\label{sec:anchor-projections}
\belowsubsectionskip

In this section, we use $\mu\in\mathbb{R}^d$ to denote any specific anchor direction that we aim to analyze, such as the ones from equations~\eqref{eq:cluster-vs-all} or~\eqref{eq:concept-vs-all}. %
Our visualization framework focuses on angular information in the embedding space. Specifically, for any raw input $x\in\mathcal{X}$, we define the following projection onto $\mu$:
\begin{equation}
\text{proj}_\mu(x)
\triangleq
\frac{\langle \phi(x) - \overline{u}^{\anchor}, \mu \rangle}
     {\|\phi(x) - \overline{u}^{\anchor}\| \|\mu\|},
\label{eq:projection}
\end{equation}
where $\langle\cdot,\cdot\rangle$ denotes the Euclidean dot product, $\|\cdot\|$ denotes taking the Euclidean norm, and $\overline{u}^{\anchor}$ is defined in equation~\eqref{eq:cluster-vs-all}. This projection is precisely the cosine similarity between vectors $(\phi(x) - \overline{u}^{\anchor})$ and $\mu$, which looks at how well-aligned these two vectors are, disregarding their magnitudes (since we divide by their norms in equation~\eqref{eq:projection}). We discuss implications of ignoring magnitude information and how to reduce the amount of ``information loss'' in Section~\ref{sec:information-loss}.
Since the cosine similarity of two vectors is always between $-1$ (the two vectors point in exactly opposite directions) and $1$ (the two vectors poin in exactly the same direction), we are guaranteed that $\text{proj}_\mu(x)\in[-1,1]$. %

To create visualizations, we plug in the visualization raw inputs $\{x_i^{\viz}\}_{i=1}^{n^{\viz}}$ into the projection operator $\text{proj}_\mu$. For notation, we write $p_i^{\viz}\triangleq\text{proj}_\mu(x_i^{\viz})$ to be the projection of the $i$-th visualization input along anchor direction $\mu$. Our 2D visualizations commonly have the x-axis correspond to these projection values while the y-axis will track quantities related to either raw input feature values or time-to-event outcomes.

\abovesubsectionskip
\subsection{Visualization Strategies for Tabular Data}
\label{sec:tabular}
\belowsubsectionskip

We begin by considering tabular data, where the raw inputs are feature vectors from $\mathcal{X}=\mathbb{R}^D$ and each of the $D$ features is assumed to be ``easy to interpret'' (e.g., age, gender, cancer status). We show how to relate projection values along an anchor direction to, at first, a single continuous feature (such as age) via a scatter plot (Section~\ref{sec:tabular-one-continuous-feature}). %
Next, we show how to relate projection values to any continuous or discrete raw feature using a heatmap (Section~\ref{sec:tabular-multiple-features}). The heatmap from Section~\ref{sec:tabular-multiple-features} reveals that some raw features might be more ``important'' for an anchor direction than others. We discuss statistical tests that can identify or help rank variables that are ``important'' for a specific anchor direction (Section~\ref{sec:tabular-statistical-tests}). Lastly, we show how to relate projection values to the neural survival analysis model's predicted survival time distributions (Section~\ref{sec:tabular-survival}).

\smallskip
\noindent
\textbf{Data and setup.}
As we progress through our visualization strategies, we apply them to the SUPPORT dataset \citep{knaus1995support}. This dataset has 8,873 data points (patients) and 14 features and is on predicting time until death for severely ill hospitalized patients with various diseases. We use a 70\%/30\% train/test split. We train a DeepSurv model, where the base neural network $f$ is a multilayer perceptron consisting of the following sequence of layers:
\begin{itemize}[leftmargin=*,itemsep=0pt,parsep=0pt,topsep=1pt]
\item Fully-connected layer (mapping $\mathbb{R}^D$ to $\mathbb{R}^d$)
\item Nonlinear activation: ReLU
\item Fully-connected layer (mapping $\mathbb{R}^d$ to $\mathbb{R}^d$)
\item Nonlinear activation: ReLU\vspace{-.75em}
\begin{center}$\bm{\vdots}$\end{center}\vspace{-.75em}
\item Fully-connected layer (mapping $\mathbb{R}^d$ to $\mathbb{R}^d$)
\item Nonlinear activation: Divide each vector by its Euclidean norm %
\item Fully-connected layer (mapping $\mathbb{R}^d$ to $\mathbb{R}$)
\end{itemize}
The second-to-last bullet point does not use ReLU activation and instead normalizes vectors to have Euclidean norm~1. We design the architecture in this manner since we shall set the embedding space that we visualize to be the representation immediately after this second-to-last bullet point's layer, i.e., the encoder $\phi$ consists of all layers except for the last fully-connected layer. Thus, the output of $\phi$ has information stored purely in terms of angles and not magnitudes (since $\|\phi(x)\|=1$ for all $x\in\mathcal{X}$). This helps reduce information loss in our visualizations (more details are in Section~\ref{sec:information-loss}). We also show visualizations in Appendix~\ref{sec:support-additional-results-no-hypersphere} where this second-to-last bullet point uses ReLU activation instead.

When training DeepSurv, we hold out 20\% of the training set to treat as validation data for hyperparameter tuning. The hyperparameter grid used (e.g., number of fully-connected layers, embedding space dimension $d$, optimizer learning rate and batch size) is stated in Appendix~\ref{sec:support-hyperparameter-grid}. After training DeepSurv (including hyperparameter tuning), the learned model achieves a test set concordance index \citep{harrell1982evaluating} of 0.617. We provide this accuracy metric for reference; our visualization framework can be used regardless of the accuracy of the model (ideally, an accurate model should have an embedding space that ``captures'' application-specific structure).

As stated above, %
we take the encoder~$\phi$ to be all the layers of $f$ prior to the last fully-connected layer. We split the test set so that 25\% of it is used as the anchor direction estimation data $\{(x_i^{\anchor},y_i^{\anchor},\delta_i^{\anchor})\}_{i=1}^{n^{\anchor}}$; the rest is used to obtain visualization raw inputs $\{x_i^{\viz}\}_{i=1}^{n^{\viz}}$.

Anchor directions are estimated via clustering as described in Section~\ref{sec:anchor-directions-via-clustering}. Since we have designed the neural network architecture so that all outputs of encoder $\phi$ have Euclidean norm~1, clustering can be done using a mixture of von Mises-Fisher distributions \citep{von1981uber}, which is the analogue of the Gaussian mixture model for Euclidean vectors with norm~1. We specifically fit this mixture model using the Expectation-Maximization algorithm implementation by \citet{kim2021pytorch} and choose the number of clusters $k$ based on the heuristic we presented in Section~\ref{sec:anchor-directions-via-clustering}. In fact, \figureref{fig:support-logrank-helper} is precisely the plot we get. Throughout this section, we set the number of clusters to be $k=5$ and, for illustrative purposes, we only show visualizations for the first cluster found. Visualizations for all~5 clusters and interpretations of these visualizations are in Appendix~\ref{sec:support-additional-results}, where we also discuss results when using other numbers of clusters and a different clustering algorithm altogether.

We separately also apply our visualization framework to tabular data on survival times of breast cancer patients. Specifically, we visualize an embedding space of a DeepSurv model trained using the Rotterdam dataset \citep{foekens2000urokinase}. We treat the GBSG dataset \citep{schumacher_1994} as the test data (that we split into anchor estimation and visualization data). Due to space contraints, we defer the visualization results for this setup to Appendix~\ref{sec:rotterdam-gbsg}, where we also explain why our framework remains valid when training and test data are independent of each other and come from different distributions.

\subsubsection{Visualizing an Anchor Direction with a Single Continuous Raw Feature}
\label{sec:tabular-one-continuous-feature}

Let $\mu$ be one of the anchor directions estimated, and let $j\in\{1,2,\dots,D\}$ be the raw feature that we want to visualize, where we assume that this feature is continuous-valued. Using the visualization raw inputs $\{x_i^{\viz}\}_{i=1}^{n^{\viz}}$, we compute the projections $p_1^{\viz}\triangleq\text{proj}_{\mu}(x_1^{\viz}),~\!\dots,~\!p_{n^{\viz}}^{\viz}\triangleq\text{proj}_{\mu}(x_{n^{\viz}}^{\viz})$, and we write $(x_i^{\viz})_j$ to mean the \mbox{$j$-th} coordinate of vector $x_i^{\viz}$. Then we can make a scatter plot of $(p_1^{\viz}, (x_1^{\viz})_j), (p_2^{\viz}, (x_2^{\viz})_j), \dots, (p_{n^{\viz}}^{\viz}, (x_{n^{\viz}}^{\viz})_j)$. As a concrete example of this, for the DeepSurv model trained on the SUPPORT dataset, using the anchor direction corresponding to the first cluster found, and using ``age'' as the raw continuous feature to be visualized, we obtain the plot in \figureref{fig:support-age-vs-proj-anchor1-hypersphere}. We see that as the projection value increases (where a value of~1 maximally aligns with the anchor direction), the age distribution tends to shift upward, suggesting that this anchor direction is associated with older patients. %

\begin{figure}[p]
\centering
\vspace{-1em}
\includegraphics[width=.68\linewidth]{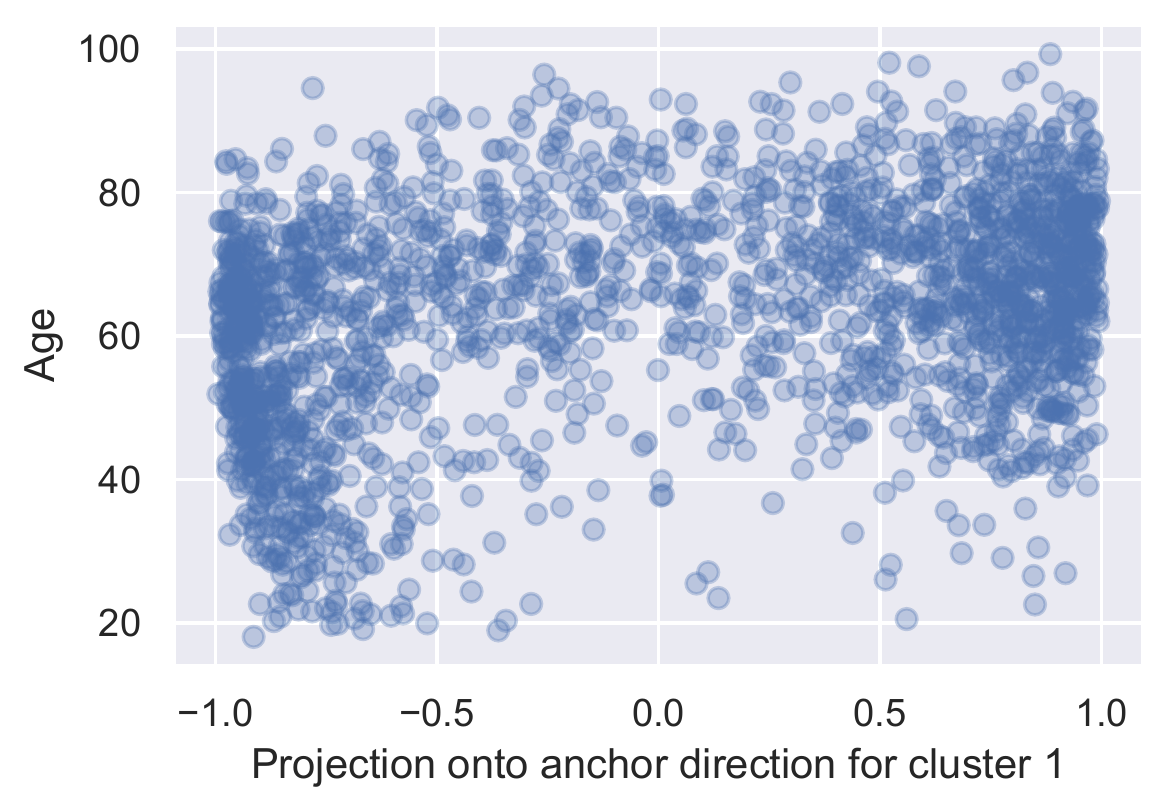}
\vspace{-1.5em}
\caption{Scatter plot of a single continuous raw feature (age) vs projection values along cluster~1's anchor direction. Plots for all clusters are in Appendix~\ref{sec:support-additional-results} (\figureref{fig:support-age-vs-proj-anchor1through5-hypersphere}).}
\label{fig:support-age-vs-proj-anchor1-hypersphere}
\vspace{-3.4em}
\end{figure}

\begin{figure}[p]
\centering
\includegraphics[width=.99\linewidth, trim={7pt 0pt 7pt 0pt}, clip]{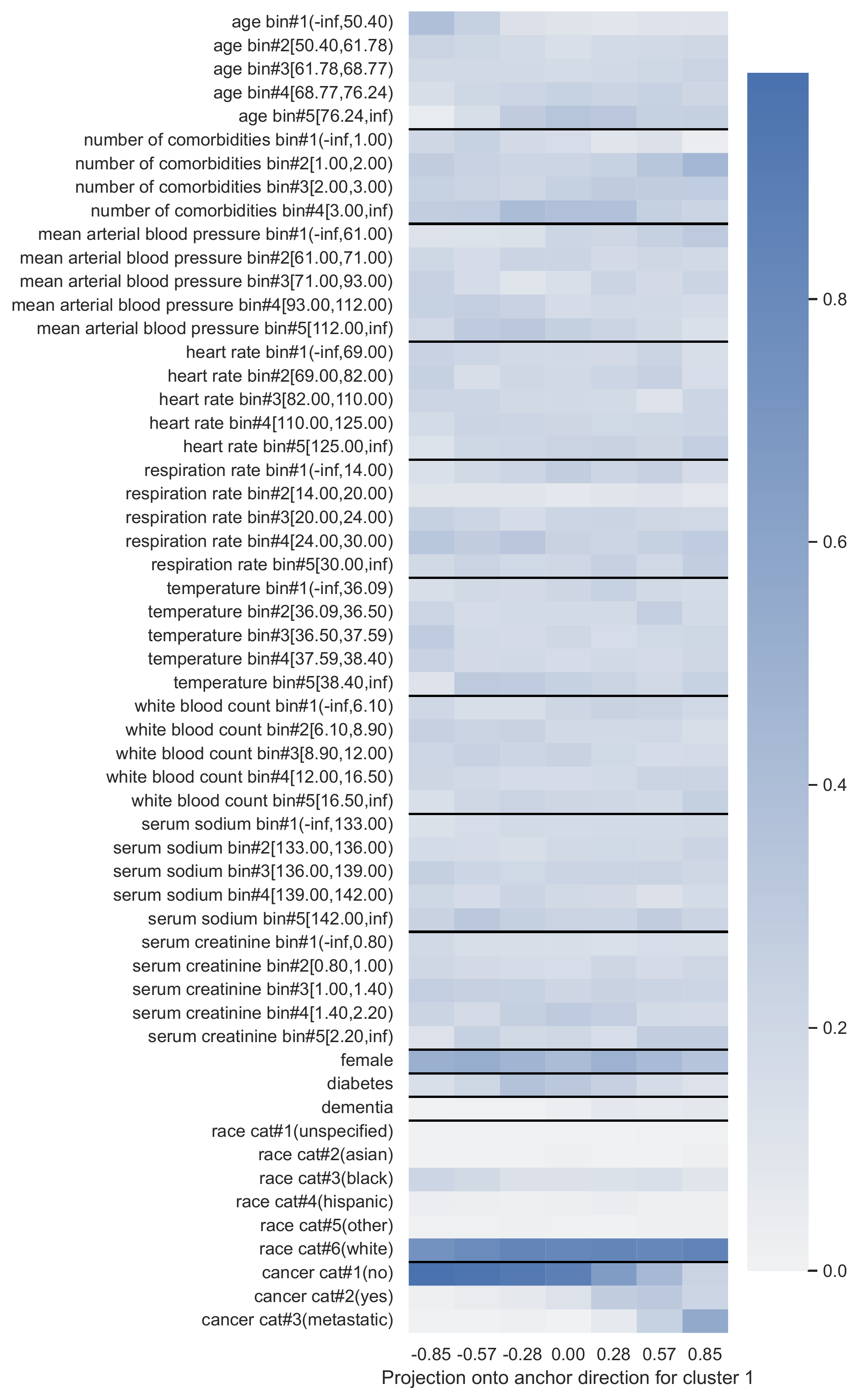}
\vspace{-2.6em}
\caption{Raw feature probability heatmap: %
the intensity of the $i$-row, \mbox{$j$-th} column indicates the fraction of visualization data in the \mbox{$j$-th} projection bin that have the $i$-th row's feature value.
Heatmaps for all clusters are in Appendix~\ref{sec:support-additional-results} (\figureref{fig:support-anchor1through5-raw-feature-heatmap}).\label{fig:support-anchor1-raw-feature-heatmap}}
\vspace{-2.5em}
\end{figure}

\subsubsection{Visualizing an Anchor Direction with Discrete or Continuous Raw Features}
\label{sec:tabular-multiple-features}

We can modify the above visualization idea to handle a discrete feature by having the y-axis of the plot correspond to specific discrete values that the feature can take on, and separately discretizing the x-axis into a user-specified number of bins (e.g., evenly spaced bins from the minimum to maximum observed projection values). In doing so, we replace the scatter plot with what we call a \emph{raw feature probability heatmap}, where the intensity at the $i$-th row and \mbox{$j$-th} column is the fraction of visualization data patients in the \mbox{$j$-th} projection value bin who have the $i$-th row's feature value. Even for a continuous feature, we can discretize it based on a user-specified discretization strategy (e.g., based on quartiles) so that all features (continuous or discrete) can be visualized together as a large heatmap. Using the same anchor direction as in \figureref{fig:support-age-vs-proj-anchor1-hypersphere}, we get the resulting heatmap in \figureref{fig:support-anchor1-raw-feature-heatmap}, where the different underlying features are separated by black horizontal lines in the heatmap.

Along the x-axis of the heatmap, the projection bins in this case are 7 evenly spaced bins between the minimum and maximum observed projection values $-0.99$ and $0.99$ respectively. The first (leftmost) bin corresponds to the interval $\mathcal{P}_1 = [-0.99, -0.71)$ with midpoint value $-0.85=\frac{1}{2}(-0.99-0.71)$, the second bin corresponds to the interval $\mathcal{P}_2 = [-0.71, -0.43)$ with midpoint value $-0.57$, and so forth. The last (rightmost) bin corresponds to the interval $\mathcal{P}_7 = [0.71, 0.99]$ with midpoint value $0.85$; only this last bin's interval includes the right endpoint.

We can readily see some trends in \figureref{fig:support-anchor1-raw-feature-heatmap}. As already revealed in \figureref{fig:support-age-vs-proj-anchor1-hypersphere}, age tends to increase as the projection value increases for cluster 1's anchor direction. However, we also see other trends as the projection value increases, such as the number of comorbidities tending to be at least~1 or cancer status tending to be ``metastatic''. In particular, this cluster seems to correspond to patients who are more ill. Indeed, these patients tend to have shorter predicted survival times, as we show later in Section~\ref{sec:tabular-survival}. %

\subsubsection{Statistical Tests to Find ``Important'' Variables for an Anchor Direction}
\label{sec:tabular-statistical-tests}

In \figureref{fig:support-anchor1-raw-feature-heatmap}, some features have noticeable trends as the projection value increases whereas others do not. %
We may want to focus on features that are the ``most important'' as they relate to anchor direction $\mu$ (especially for datasets with a large number of features, displaying all features would be impractical).
We now show how to rank raw features using statistical tests of association between two variables, for which one of the variables we take to be the projection value along $\mu$ (or a discretized version of this projection value), and the other variable will be one of the $D$ raw features (or a discretized version of it).

The first test we could use to rank features is based
on the heatmap visualization from \figureref{fig:support-anchor1-raw-feature-heatmap}: consider a single raw feature (such as ``white blood count'') and note that the heatmap restricted to that raw feature (i.e., the heatmap that only looks at the discretized white blood count vs discretized projections) is a contingency table, for which we can run Pearson's chi-squared test to assess whether white blood count and the projection value are independent. We could repeat this for all the different raw features (note that for raw features that are indicator random variables, we would have to add a row corresponding to one minus the indicator before running the statistical test) and rank the raw features based on the p-values obtained. Doing so, we obtain the ranking shown in Table~\ref{tab:support-anchor1-feature-ranking}. %
Again, this ranking could be used in constructing raw feature probability heatmaps, where we choose to only visualize a few top-ranked features.

A limitation of this chi-squared approach is that it requires projection values and raw features to be discretized. If the raw features are all continuous or ordinal, then one could use a different statistical test to compare projection values (without discretization) to each raw feature (also without discretization), such as using Kendall's tau test \citep{kendall1938new} (which checks for a monotonic relationship between a pair of variables). If the raw features are all categorical, then instead the Kruskal-Wallis test by ranks could be used \citep{kruskal1952use}.

Importantly, when using any of these statistical tests mentioned above, we suggest that the modeler check the assumptions of the test to see whether the test is appropriate for the particular dataset under examination. One of the common assumptions the tests have are that the input data points given to the tests are independent, which we have taken care to achieve by having the visualization data be separate from the training and anchor direction estimation~data. %

Separately, note that we have not stated how to pick a threshold for how small a p-value should be to flag a raw feature as ``important''. From a visualization standpoint, we think that providing a ranking is sufficient; the modeler can arbitrarily decide on how many top features to focus on or to visualize, which is equivalent to choosing an arbitrary p-value threshold. We discuss how to choose a threshold that controls for a desired false discovery rate in Appendix~\ref{sec:p-value-thresholding}.

\begin{table}[t!]
\vspace{-1em}
\caption{Ranking of raw features based on the \mbox{p-value} of Pearson's chi-squared test (for cluster~1, the same cluster visualized in Figures~\ref{fig:support-age-vs-proj-anchor1-hypersphere} and~\ref{fig:support-anchor1-raw-feature-heatmap}); rankings for all clusters are provided in Appendix~\ref{sec:support-additional-results} (Table~\ref{tab:support-anchor1through5-feature-ranking}).\label{tab:support-anchor1-feature-ranking}}
\vspace{-.75em}
\centering
\adjustbox{width=.68\linewidth}{
\begin{tabular}{rll}
\toprule
Rank & Feature & p-value \\ \midrule
1 & cancer & $2.86\times 10^{-225}$ \\
2 & age & $1.04\times 10^{-45}$ \\
3 & number of comorbidities & $1.91\times 10^{-24}$ \\
4 & mean arterial blood pressure & $9.25\times 10^{-19}$ \\
5 & diabetes & $2.46\times 10^{-14}$ \\
6 & dementia & $4.15\times 10^{-11}$ \\
7 & temperature & $1.58\times 10^{-10}$ \\
8 & heart rate & $3.34\times 10^{-9}$ \\
9 & female & $1.07\times 10^{-6}$ \\
10 & serum creatinine & $1.21\times 10^{-6}$ \\
11 & race & $3.42\times 10^{-5}$ \\
12 &white blood count & $7.02\times 10^{-5}$ \\
13 & respiration rate & $4.54\times 10^{-4}$ \\
14 & serum sodium & $6.14\times 10^{-2}$ \\
\bottomrule
\end{tabular}
}
\vspace{-1em}
\end{table}

Lastly, an important limitation of the general strategy we have stated for ranking raw features is that each statistical test is applied in a manner where we do \emph{not} account for possible interactions between different raw features. We mention two methods that could help determine interactions in Appendix~\ref{sec:interactions}.

\subsubsection{Visualizing an Anchor Direction with Time-to-Event Outcomes}
\label{sec:tabular-survival}

To relate an anchor direction to time-to-event outcomes, we again use a heatmap. We set the x-axis to be the same as in \figureref{fig:support-anchor1-raw-feature-heatmap}. As stated in Section~\ref{sec:tabular-multiple-features}, the x-axis (corresponding to projection values along the anchor direction for cluster~1) has been discretized into projection value bins that are intervals. We specifically had $\mathcal{P}_1 = [-0.99, -0.71), \mathcal{P}_2 = [-0.71, -0.43), \dots, \mathcal{P}_7=[0.71, 0.99]$. Then note that the \mbox{$j$-th} projection interval $\mathcal{P}_j$ corresponds to the following visualization data points:
\begin{equation}
{\mathcal{I}_j
\triangleq\{ i \in \{1,2,\dots,n^{\viz}\}\text{ s.t.~}\text{proj}_\mu(x_i^{\viz}) \in \mathcal{P}_j\}}.
\label{eq:projection-bin-points}
\end{equation}
Then letting $\widehat{S}(t|x)$ denote the neural survival analysis model's prediction of the conditional survival function for raw input $x$ (e.g., for DeepSurv, we use equation~\eqref{eq:deepsurv-surv}), we can compute the following average predicted survival function for the \mbox{$j$-th} projection~bin:
\begin{equation}
{\widehat{S}_j(t)
\triangleq \frac{1}{|\mathcal{I}_j|}\sum_{i\in\mathcal{I}_j} \widehat{S}(t|x_i^{\viz})}.
\label{eq:survival-curve-bin}
\end{equation}
We plot the function $\widehat{S}_j$ as the \mbox{$j$-th} column of the heatmap, where the y-axis uses a discretized time grid (e.g., evenly spaced time points between the minimum and maximum observed times in the training data). The resulting heatmap (which we call a \emph{survival probability heatmap}) is in~\figureref{fig:support-anchor1-survival-heatmap}. %
For high projection values (e.g., looking at the rightmost column), the survival probability decays quickly as time increases (starting from the bottom and going upward in the heatmap), suggesting that the patients whose embedding vectors align the most with this anchor direction tend to have short survival times. %

\begin{figure}[t]
\centering
~~\includegraphics[scale=.5]{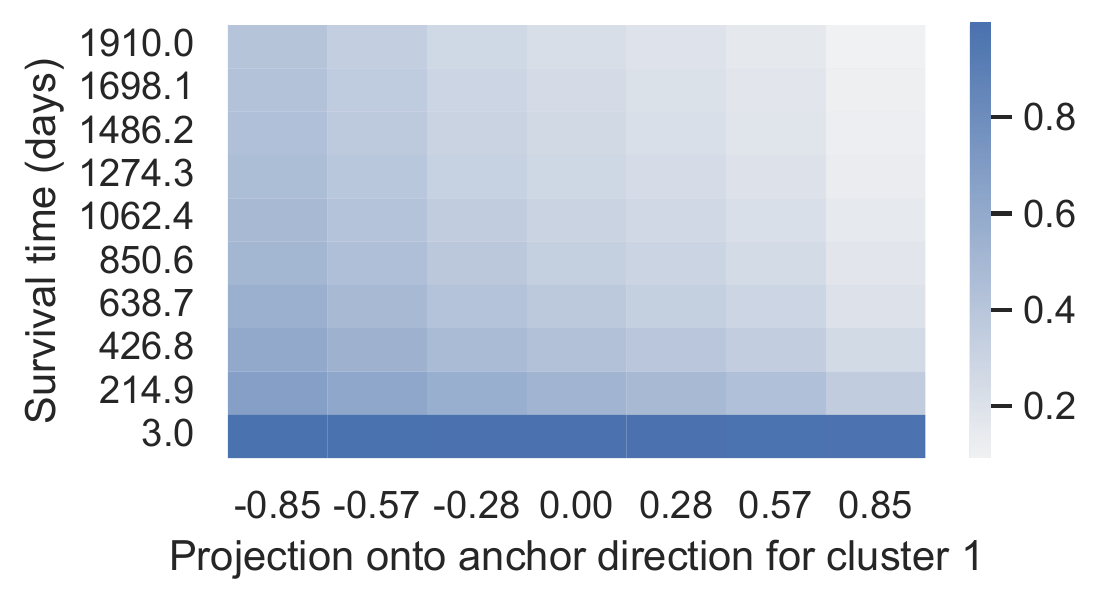}
\vspace{-1em}
\caption{Survival probability heatmap: using the same anchor direction as in \figureref{fig:support-anchor1-raw-feature-heatmap}, we show how projection values along this anchor direction relate to predicted survival probabilities over time; the intensity at the $i$-th row and \mbox{$j$-th} column is the survival probability for the \mbox{$j$-th} projection bin at the $i$-th discretized time. Heatmaps for all clusters are in Appendix~\ref{sec:support-additional-results} (\figureref{fig:support-anchor1through5-survival-heatmap}).\label{fig:support-anchor1-survival-heatmap}}
\vspace{-1em}
\end{figure}

\abovesubsectionskip
\subsection{Handling Images as Raw Inputs}
\label{sec:images}
\belowsubsectionskip

We now turn to working with images as raw inputs, %
where we intentionally examine a dataset that has known ground truth structure for what the embedding space should capture. This helps us see to what extent the learned embedding space we examine recovers this ground truth structure.
Specifically, we use the Survival MNIST dataset, which is the MNIST handwritten digit dataset \citep{lecun2010mnist} modified by \citet{polsterl2019survival} to have survival labels, i.e., observed times and event indicators. %
How P\"{o}lsterl generates these labels depends on some distributional settings, for which we use the same settings as \citet{goldstein2020x}. 
In particular, each of the 10 digits has a mean survival time shown in \figureref{fig:survival-mnist-ground-truth-mean-survival-times}. %
Each image's true survival time is sampled based on the digit of the image (e.g., all images of digit~0 have true survival times sampled i.i.d.~from a Gamma distribution with mean 11.25 and variance $10^{-3}$). The censoring times are sampled from a uniform distribution so that the overall censoring rate is roughly 50\%. Also, all images of digit~0 have observed times that are censored. %
More dataset details are in Appendix~\ref{sec:survival-mnist-dataset}. %

Using the training set %
(60,000 data points: each consists of an image, observed time, and event indicator),
we learn a DeepSurv model where the base neural network is a convolutional neural network (architecture and training details are in Appendix~\ref{sec:survival-mnist-encoder}). Importantly, the digit labels (which digit each image corresponds to) are \emph{not} available during training. We split the test set (10,000 data points), using 25\% of it as anchor direction estimation data and the rest as visualization data.
For test data, we assume that we have access to their digit labels, which helps us assess whether the embedding space learns what digits are.

First, we use anchor directions defined by concepts of digits. For example, anchor direction estimation data corresponding to digit~0 represents the concept ``digit~0''. For digit 0's anchor direction (computed using equation~\eqref{eq:concept-vs-all}), we can discretize the projection values of visualization raw inputs into projection bins $\mathcal{P}_1,\dots,\mathcal{P}_m$ where $m$ is the number of bins, just as we did with tabular data. %
Each projection bin $\mathcal{P}_j$ corresponds to a subset $\mathcal{I}_j$ of visualization data (see equation~\eqref{eq:projection-bin-points}). For projection bin $\mathcal{P}_j$, we can randomly sample, for instance, 10 raw input images corresponding to points in $\mathcal{I}_j$ and display these images along the y-axis. The resulting \emph{random input vs projection plot} is shown in \figureref{fig:survival-mnist-concept0-random-samples-vs-projection}.\footnote{When the encoder uses a convolutional layer, it is also possible to make a variant of \figureref{fig:survival-mnist-concept0-random-samples-vs-projection} where instead of displaying random raw inputs per projection bin, we display a convolution filter's outputs of these random raw inputs instead.} We see that as the projection values get large, the images that get sampled tend to be of digits 0, 7, and 9. These digits have the highest ground truth mean survival times (see \figureref{fig:survival-mnist-ground-truth-mean-survival-times}).

\begin{figure}[t]
\centering
\includegraphics[width=.9\linewidth, trim={0pt 7em 0pt 7em}, clip]{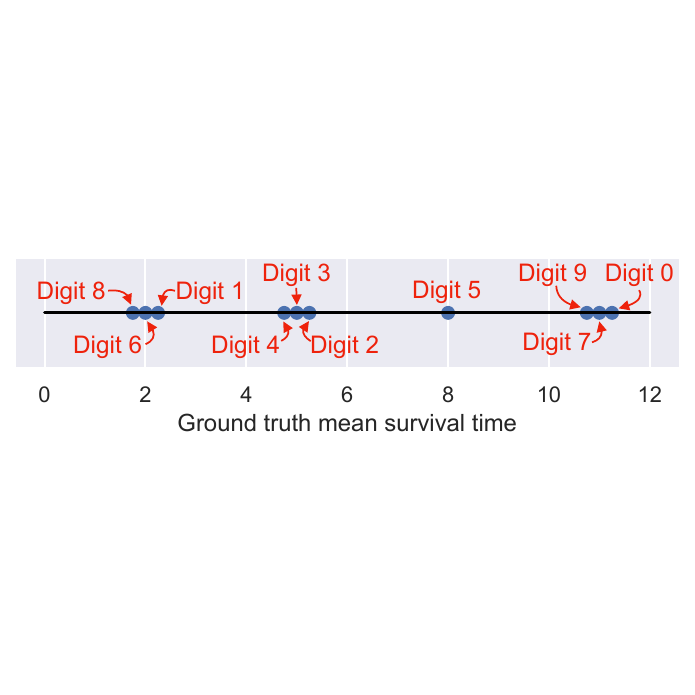}
\vspace{-1em}
\caption{Survival MNIST dataset: each digit has a different ground truth mean survival time.}
\label{fig:survival-mnist-ground-truth-mean-survival-times}
\vspace{-1em}
\end{figure}

\begin{figure}[t]
\centering
~\includegraphics[width=.97\linewidth, trim={0pt 7pt 0pt 0pt}, clip]{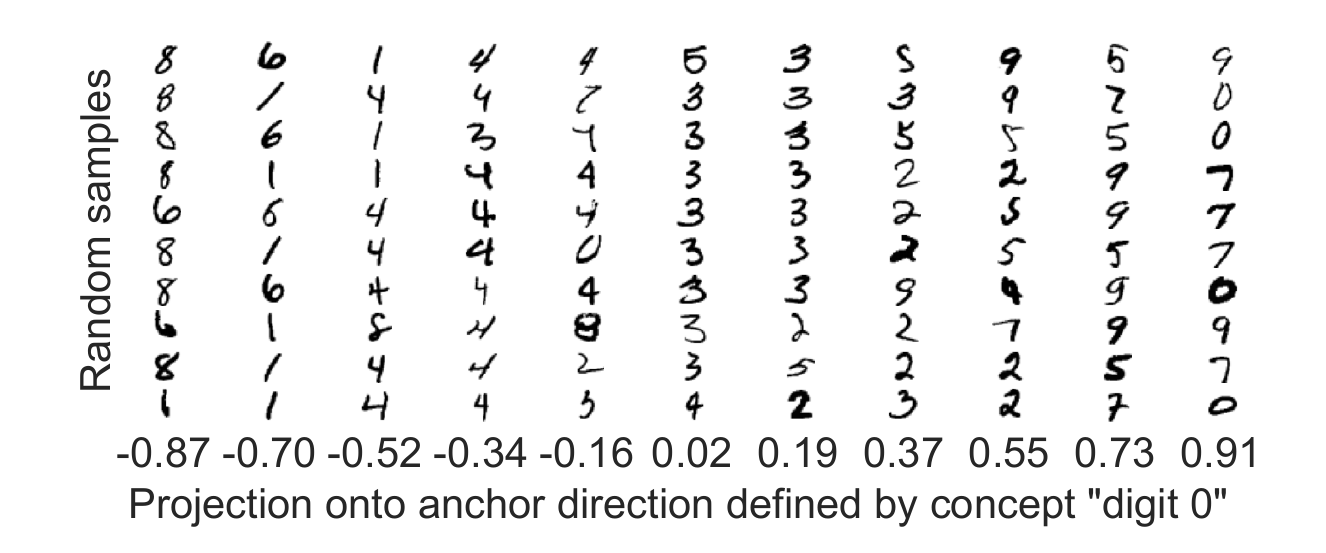}
\vspace{-.5em}
\caption{Survival MNIST random input vs projection plot: we display 10 random visualization raw inputs per projection bin.%
\label{fig:survival-mnist-concept0-random-samples-vs-projection}} %
\vspace{-1em}
\end{figure}

Due to space constraints, we defer additional Survival MNIST visualizations to Appendix~\ref{sec:survival-mnist-additional-results}. The key findings are as follows.
First, note that the digits have mean survival times ranked as 8, 6, 1, 4, 3, 2, 5, 9, 7, 0 (see \figureref{fig:survival-mnist-ground-truth-mean-survival-times}). We refer to two digits as ``adjacent'' if they are ranked next to each other (e.g., digits 1 and 4 are adjacent). We find that the learned embedding space tends to have the \mbox{$j$-th} digit's anchor direction align well with embedding vectors of the \mbox{$j$-th} digit's images as well as those of other adjacent digits (e.g., digit~1 images tend to have high projection values for digit~4's anchor direction). Because the embedding space is not learned in a manner that knows what the different digits are, the 10 digits do not get ``disentangled'' in the embedding space. Treating data that are censored as a ``concept'', we find that the embedding space recognizes which digits are more censored than others. Meanwhile, we also show that if we estimate anchor directions using clustering, the violin plot we use to help select the number of clusters sharply increases in p-values after 9 clusters, as expected (digit~0 is the only one that is always censored making it difficult to learn).

We point out that survival probability heatmaps like the one in \figureref{fig:support-anchor1-survival-heatmap} are not specific to tabular data as they do not depend on raw features; they can be created the same manner when raw inputs are images. %
Furthermore, if raw images are converted into a tabular format (e.g., by running object detectors and representing each image as a feature vector specifying how often each object appears), then our tabular data visualization strategies could of course be used.

\abovesubsectionskip
\subsection{Loss of Magnitude Information}
\label{sec:information-loss}
\belowsubsectionskip

Our visualization framework is based on angular information: projection values are cosine similarities, which measure angles between vectors, disregarding their norms (or magnitudes).
In the extreme case where the ``information content'' of the embedding vectors are all in magnitudes and not angles, the only possible projection values are $-1$ and $1$; projection values within the open interval $(-1,1)$ are not possible. Thus, all our visualizations where the x-axis is based on projection values would only need two projection bins for $-1$ and $1$. We formally state this theoretical result and provide its proof in Appendix~\ref{sec:proposition-proof}.

When working with real data, the information in the embedding space will typically not be entirely in angles or entirely in magnitudes. As more information is stored in magnitudes, our visualizations based on projection values will start exhibiting this phenomenon where %
the projection values ``clump up'' at $-1$ and~$1$. An example of this visualization artifact for DeepSurv trained on the SUPPORT dataset (without the nonlinear activation that normalizes the embedding vectors to have norm~1) is in Appendix~\ref{sec:support-additional-results-no-hypersphere}. %

\smallskip
\noindent
\textbf{Reducing information loss.}
Since the modeler can often choose the encoder $\phi$ when designing the neural network architecture of $f$, the encoder~$\phi$ could be chosen as to avoid storing information in magnitudes, which would reduce the information lost by using our framework. We had precisely done this %
in Section~\ref{sec:tabular} when we constrained the output of $\phi$ to have Euclidean norm~1.
This ``norm~1'' constraint alone could still sometimes not lead to enough angular information stored (e.g., if the embedding space $\mathbb{P}_U$ is highly concentrated so that nearly all embedding vectors randomly sampled from it point in almost the same direction). A recent theoretical and empirical insight when working with a ``norm~1'' constraint (technically referred to as working with vectors on a hypersphere)
is to add regularization that encourages the embedding vectors to have angles that are ``diverse'' (closer to uniformly distributed in all directions). Adding this regularization improves prediction accuracy for a variety of neural network architectures \citep{wang2020understanding,liu2021learning}. This diversity would, in our visualization context, lead to to having the projections onto anchor directions be more dispersed across the interval $[-1,1]$ instead of being ``clumped up'' around specific points within~$[-1,1]$.

\section{Discussion}

We have presented a visualization framework that is meant to help developers of neural survival analysis models better understand what their models have learned. Importantly, %
the visualizations we have proposed only reveal possible associations related to an embedding space. No causal claims are made. In focusing on examining one embedding space at a time,
our framework %
is not designed
to explain how the overall neural survival analysis model actually makes predictions.
We believe that our visualization strategies would be helpful in assessing whether a model has internally learned associations that agree with existing clinical literature, or to see if the model surfaces new associations that warrant further investigation.

While we have developed our framework for survival analysis, it can be modified to support other prediction tasks. For example, to support classification, it suffices to make two changes. First, the clustering approach for estimating anchor directions would be unnecessary since we could take the anchor directions to be the average embedding vector for each class, minus the center of mass across embedding vectors. Second, to relate an embedding space to predicted class distributions instead of survival time distributions, we could estimate and visualize the probability of different classes per projection bin instead of using our proposed survival probability heatmaps. Although our framework can easily be adapted to classification, whether it offers any advantages over the many existing visualization tools for classification (e.g., \citealt{selvaraju2016grad,zhou2016learning,dabkowski2017real,shrikumar2017learning,smilkov2017smoothgrad,kim2018interpretability}) is unclear. Better understanding the advantages and disadvantages of our framework in prediction tasks beyond survival analysis would be an interesting direction for future research.

\acks{This work was supported by NSF CAREER award \#2047981. The author thanks the anonymous reviewers for very helpful feedback.}

\bibliography{anchor-vis}

\appendix
\counterwithin{figure}{section}
\counterwithin{table}{section}

\abovesectionskip
\section{Statistical Considerations}
\belowsectionskip

\subsection{What Happens if Anchor Direction Estimation Data were the Same as the Training Data}
\label{sec:statistical-considerations-anchor-direction-estimation}
\belowsubsectionskip

Suppose that the anchor direction estimation data $\{(x_i^{\anchor},y_i^{\anchor},\delta_i^{\anchor})\}_{i=1}^{n^{\anchor}}$ were actually the same as the training data $\{(x_i,y_i,\delta_i)\}_{i=1}^{n}$, so that $n^{\anchor}=n$ and $x_i^{\anchor}=x_i$ for $i=1,\dots,n$. Since the training data were used to learn $\phi$, then $\phi$ itself depends on all the training data. Thus, the embedding vectors of the anchor direction estimation data would be $u_i^{\anchor}=\phi(x_i^\anchor)=\phi(x_i)$ for each $i=1,\dots,n$. This means that $u_1^{\anchor},\dots,u_{n^{\anchor}}^{\anchor}$ would no longer be guaranteed to be independent since they each depend on $\phi$ which in turn depends on all of the training data (and thus all of the anchor direction estimation data).

\abovesubsectionskip
\subsection{Comments on the Log-rank Test}
\label{sec:statistical-considerations-logrank}
\belowsubsectionskip

The log-rank test, like other statistical hypothesis tests, is designed under certain assumptions, where checking these assumptions is important if one wants the resulting p-values computed to be statistically valid.
One case these assumptions hold is if in addition to the survival analysis setup stated in Section~\ref{sec:survival-analysis-background}, we further assume that:
\begin{itemize}[leftmargin=1.8em,itemsep=0pt,parsep=0pt,topsep=1pt]
\item[(i)] the conditional censoring time distribution $\mathbb{P}_{C|X}$ is independent of raw inputs and is thus equal to the marginal censoring time distribution $\mathbb{P}_C$;
\item[(ii)] the true underlying hazard function $h(t|x)$ satisfies the proportional hazards assumption \eqref{eq:proportional-hazards}.
\end{itemize}
To provide some intuition, condition (i) ensures that when comparing any two clusters using the log-rank test, the censoring patterns for the two clusters ``look the same''. To see why this is important, consider an extreme example where two clusters truly have the same underlying survival time distribution but their censoring patterns are so different that one of the clusters always has all its observations censored whereas the other has no observations censored. In this case, we would not be able to tell that these two clusters have the same survival time distribution.

The justification for condition (ii) is more technical. One observation is that the log-rank test for comparing two clusters can be shown to be equivalent to a statistical test for the Cox proportional hazards model (specifically the so-called ``score test'') that checks for association between the survival time and an indicator variable stating which of the two clusters a data point is in \citep[Section 20.4]{harrell2001regression}. For more theoretical justification, see the books by \citet[Chapter~7]{fleming1991counting} and \citet[Chapter~V]{andersen1993statistical}.

\abovesubsectionskip
\subsection{P-value Thresholding to Control for a Desired False Discovery Rate}
\label{sec:p-value-thresholding}
\belowsubsectionskip

For a specific anchor direction $\mu$, we rank raw features by computing p-values of a statistical test (such as the chi-squared test of independence) that quantifies the strength of the association between each raw feature and projections along $\mu$. These tests are not independent of each other because all tests use the same projection values along the same direction $\mu$. To determine a p-value threshold that appropriately controls for false discovery rate across multiple statistical tests with arbitrary dependence between them, one could use, for instance, the method by \citet{benjamini2001control}.

\abovesectionskip
\section{SUPPORT Dataset Experiment}
\belowsectionskip

\subsection{Hyperparameter Grid and Optimization Details}
\label{sec:support-hyperparameter-grid}
\belowsubsectionskip

We train the neural network using minibatch gradient descent with at most 100 epochs and early stopping (no improvement in the validation concordance index after 10 epochs). We specifically use the Adam optimizer \citep{kingma2014adam}. We sweep over the following hyperparameters:
\begin{itemize}[leftmargin=*,itemsep=0pt,parsep=0pt,topsep=1pt]
\item Batch size: 64, 128
\item Learning rate: 0.01, 0.001
\item Number of fully-connected layers in encoder $\phi$: 1, 2, 3, 4
\item Embedding dimension $d$: 5, 6, 7, 8, 9, 10
\end{itemize}
Our code is written using PyTorch \citep{paszke2019pytorch}.

\smallskip
\noindent
\textbf{Compute instance.}
We ran our code on a Ubuntu 22.04.1 LTS machine with an Intel Core i9-10900K CPU (3.7GHz, 10 cores, 20 threads) with 64GB RAM and a Quadro RTX 4000 GPU (with 8GB GPU RAM).

\abovesubsectionskip
\subsection{Additional Visualizations Using an Encoder With the Euclidean Norm~1 Constraint}
\label{sec:support-additional-results}
\belowsubsectionskip

\begin{figure}[p!]
\centering
\includegraphics[width=.6\linewidth]{figures/support-age-vs-proj-anchor1-hypersphere}
\\
\includegraphics[width=.6\linewidth]{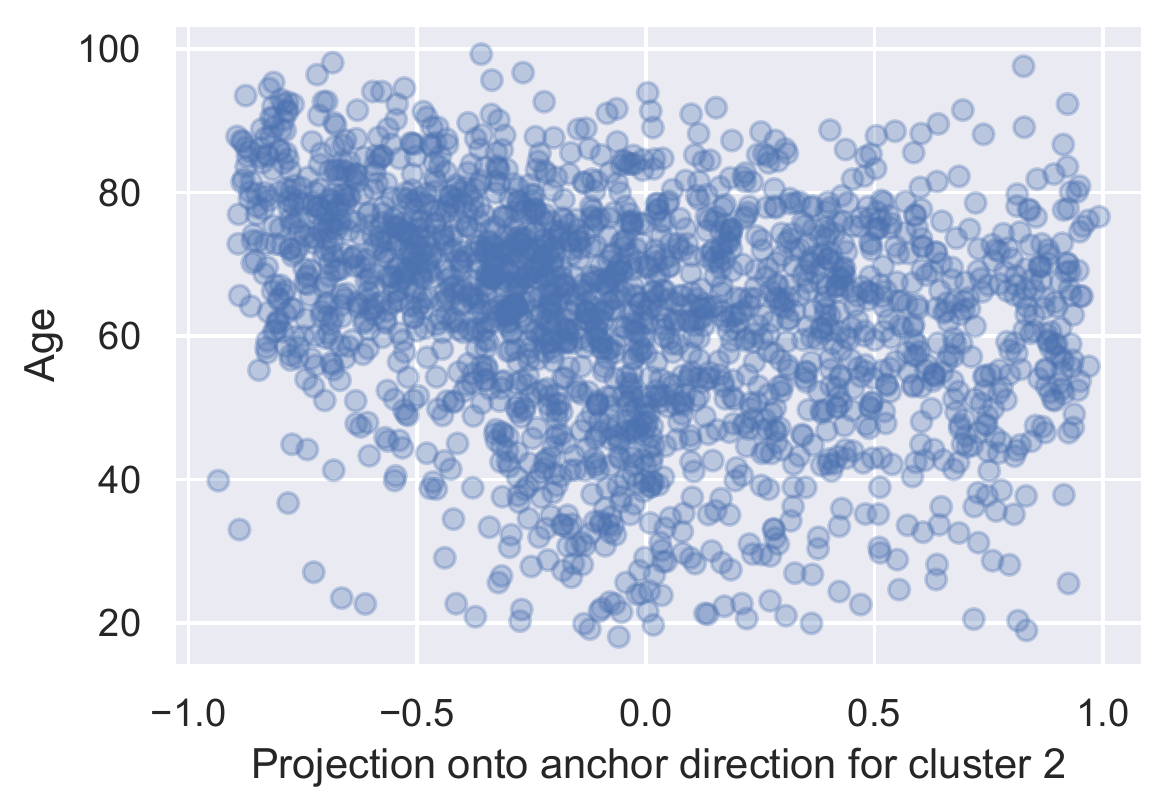}
\\
\includegraphics[width=.6\linewidth]{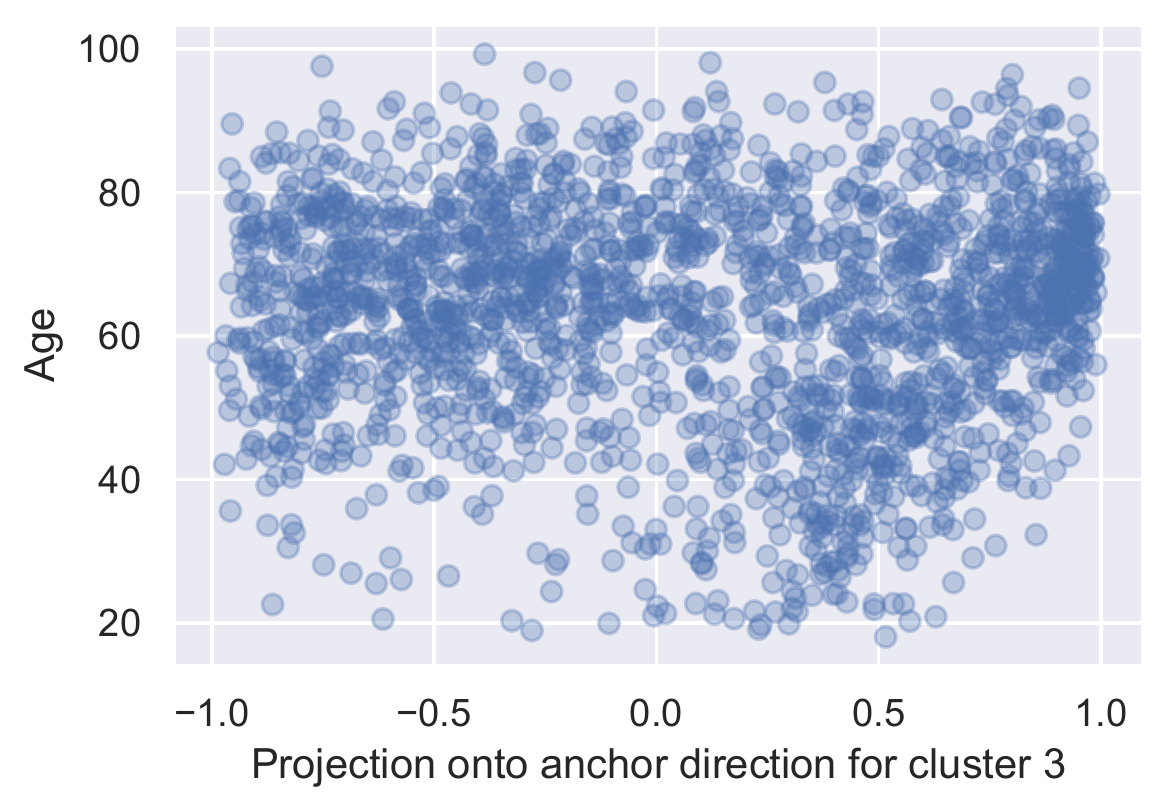}
\\
\includegraphics[width=.6\linewidth]{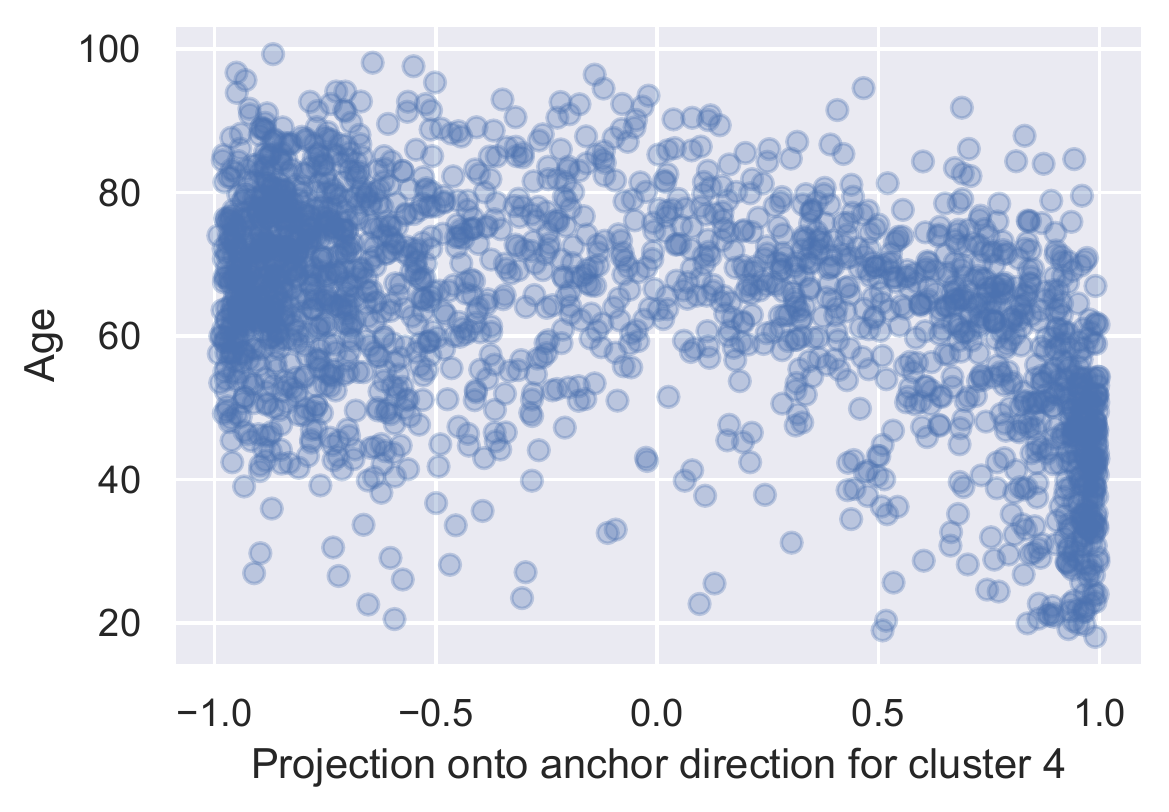}
\\
\includegraphics[width=.6\linewidth]{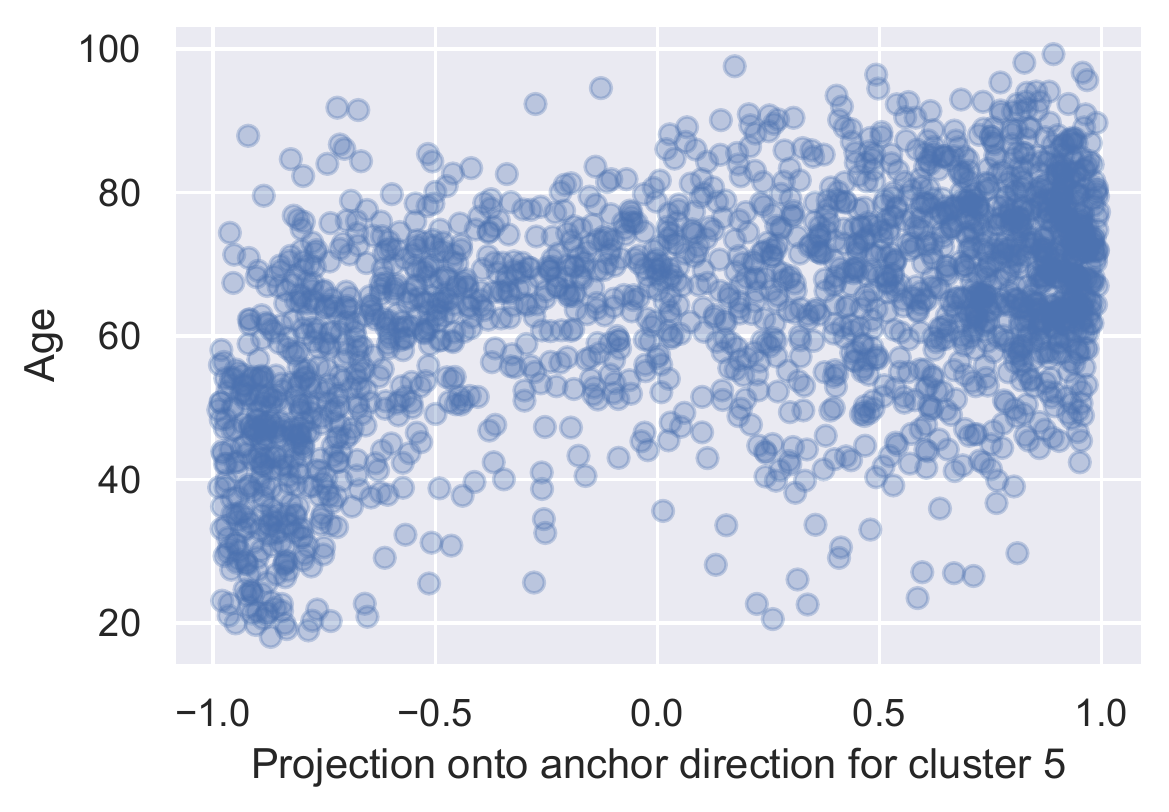}
\vspace{-1em}
\caption{SUPPORT dataset: scatter plots of age vs projection values for a DeepSurv model where the encoder has a Euclidean norm~1 constraint. The projection values are along each cluster's anchor direction (for all 5 clusters in a 5-component mixture of von Mises-Fisher distributions). The plot only for cluster~1 is in \figureref{fig:support-age-vs-proj-anchor1-hypersphere}.}
\label{fig:support-age-vs-proj-anchor1through5-hypersphere}
\end{figure}

In the main paper, we only showed visualizations for the first cluster found out of the 5 clusters used in the 5-component mixture of von Mises-Fisher distributions. We include scatter plots of age vs anchor projections for all 5 clusters in \figureref{fig:support-age-vs-proj-anchor1through5-hypersphere}, raw feature probability heatmaps for all 5 clusters in \figureref{fig:support-anchor1through5-raw-feature-heatmap}, raw feature rankings for all 5 clusters in Table~\ref{tab:support-anchor1through5-feature-ranking}, and survival probability heatmaps for all 5 clusters in \figureref{fig:support-anchor1through5-survival-heatmap}.

\smallskip
\noindent
\textbf{Interpreting the visualizations.}
From the visualizations, we can see some patterns, where for simplicity we mention only a few per cluster (e.g., looking at the top few features per cluster in Table~\ref{tab:support-anchor1through5-feature-ranking} already provides insight); we list these clusters in order of how fast their survival probability heatmap's rightmost column decays (starting from the fastest decay, indicative of survival times that tend to be the shortest):
\begin{itemize}[leftmargin=*,itemsep=0pt,parsep=0pt,topsep=1pt]
\item (Fastest survival probability decay) Cluster 1 is, as already stated in Sections~\ref{sec:tabular-multiple-features} and~\ref{sec:tabular-survival}, associated with patients being more elderly and having metastatic cancer and at least one comorbidity. %
\item Cluster 5 is associated with patients who are elderly, have low or normal temperatures, and often have cancer (non-metastatic or metatstatic). Similar to cluster 1, cluster 5 is largely also associated with patients having at least one comorbidity.
\item Cluster 2 is associated with patients having high temperatures (indicative of a fever), high sodium levels, and lower ages.
\item Cluster 3 is associated with patients without cancer, with low or normal temperatures, and ages that are neither low nor high.
\item (Slowest survival probability decay) Cluster 4 is associated with patients who are young, do not have cancer, and (compared to patients with high projection values for the other clusters) often do not have any comorbidities.
\end{itemize}
These interpretations are not surprising in that being elderly, having cancer, and having at least one comorbidity intuitively should be associated with a patient being more ill and tending to have shorter survival times. Similar findings for the same dataset but using a different neural survival analysis model have been reported previously by \citet{chen2022survival}. Note that the above ranking of clusters was determined qualitatively by looking at the survival probability heatmaps. In fact, an approach we suggest for ranking clusters/anchor directions by median survival time estimates in Appendix~\ref{sec:handling-many-anchor-directions} yields the same ranking.

\smallskip
\noindent
\textbf{Using different numbers of clusters and a different clustering algorithm.} We have also separately tried using different numbers of clusters (aside from $k=5$) with the mixture of von Mises-Fisher distributions. As the resulting visualizations do not convey much more insight than what we have already presented, we defer these to our code repository. Qualitatively, we found the following: using $k<5$ results in ``coarser-grain'' clusters, each of which look like a combination of the clusters we found with $k=5$, and using $k>5$ results in ``finer-grain'' clusters although some of these finer-grain clusters have raw feature probability (and, separately, survival probability heatmaps) that look very similar (so that some of these clusters should probably be merged into a single cluster as they correspond to similar raw feature patterns and survival time distributions).

We also repeated this exercise of trying different numbers of clusters where we cluster using a Gaussian mixture model instead. Again, we defer the resulting visualizations to our code repository except for the violin plot, which we show in \figureref{fig:support-logrank-helper-hypersphere-gmm}. Qualitatively, the clusters found for the same choice of $k$ was somewhat similar to what we get using a mixture of von Mises-Fisher distributions, although of course the clusters are not entirely the same. The violin plot ends up looking a bit different: we see in \figureref{fig:support-logrank-helper-hypersphere-gmm} that the p-values tend to be very small for $k=2$ and $k=3$, and then they increase a bit at $k=4$, and even more at $k=5$ (the highest point in the violin plot significantly increases from $k=4$ to $k=5$ and then it stays very high for all values of $k>5$ that we tried).

Ultimately, we have left the choice of which clustering algorithm to use up to the user. We suspect that a ``good'' choice of clustering algorithm would be able to find a larger number of clusters while keeping the p-values low in the violin plot. For instance, using a mixture of von Mises-Fisher distributions, the violin plot has very low p-values for up to $k=5$ in \figureref{fig:support-logrank-helper}. In contrast, the violin plot we get using a Gaussian mixture model for the same embedding vectors has low p-values only up to $k=3$ as shown in \figureref{fig:support-logrank-helper-hypersphere-gmm}. This suggests that the clusters found for the Gaussian mixture model do not distinguish the embedding vectors as well in terms of survival outcomes compared to the mixture of von Mises-Fisher distributions.

\begin{figure*}[p!]
\centering
\includegraphics[width=.95\linewidth]{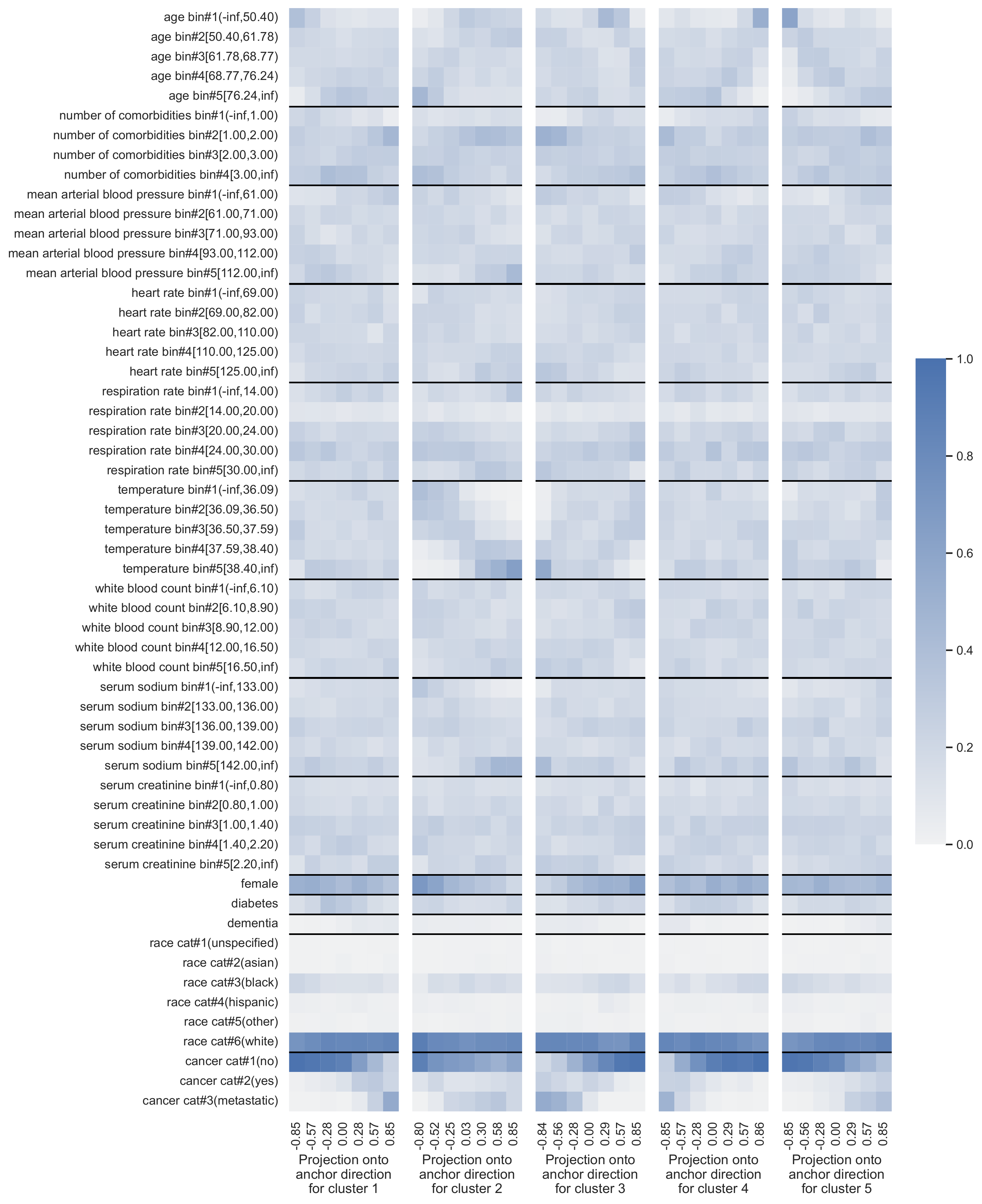}
\vspace{-1.5em}
\caption{SUPPORT dataset: raw feature probability heatmaps for a DeepSurv model where the encoder has a Euclidean norm~1 constraint. These heatmaps are for all 5 clusters' anchor directions (clusters are from a 5-component mixture of von Mises-Fisher distributions). The heatmap only for cluster~1 is also shown in \figureref{fig:support-anchor1-raw-feature-heatmap}.\label{fig:support-anchor1through5-raw-feature-heatmap}}
\vspace{-1em}
\end{figure*}

\begin{table*}[p!]
\caption{SUPPORT dataset: rankings of raw features based on the \mbox{p-value} of Pearson's chi-squared test for a DeepSurv model where the encocder has a Euclidean norm~1 constraint. These tables are for all 5 clusters in a 5-component mixture of von Mises-Fisher distributions. The ranking table only for cluster~1 is also in Table~\ref{tab:support-anchor1-feature-ranking}.\label{tab:support-anchor1through5-feature-ranking}}
\vspace{-.75em}
\centering
\adjustbox{width=.425\linewidth}{
\begin{tabular}{rll}
\toprule
\multicolumn{3}{c}{Cluster 1} \\
\midrule
Rank & Feature & p-value \\ \midrule
1 & cancer & $2.86\times 10^{-225}$ \\
2 & age & $1.04\times 10^{-45}$ \\
3 & number of comorbidities & $1.91\times 10^{-24}$ \\
4 & mean arterial blood pressure & $9.25\times 10^{-19}$ \\
5 & diabetes & $2.46\times 10^{-14}$ \\
6 & dementia & $4.15\times 10^{-11}$ \\
7 & temperature & $1.58\times 10^{-10}$ \\
8 & heart rate & $3.34\times 10^{-9}$ \\
9 & female & $1.07\times 10^{-6}$ \\
10 & serum creatinine & $1.21\times 10^{-6}$ \\
11 & race & $3.42\times 10^{-5}$ \\
12 &white blood count & $7.02\times 10^{-5}$ \\
13 & respiration rate & $4.54\times 10^{-4}$ \\
14 & serum sodium & $6.14\times 10^{-2}$ \\
\bottomrule
\end{tabular}
}
~
\adjustbox{width=.425\linewidth}{
\begin{tabular}{rll}
\toprule
\multicolumn{3}{c}{Cluster 2} \\
\midrule
Rank & Feature & p-value \\ \midrule
1 & temperature & $6.73\times 10^{-181}$ \\
2 & age & $2.42\times 10^{-40}$ \\
3 & serum sodium & $4.06\times 10^{-37}$ \\
4 & female & $6.42\times 10^{-32}$ \\
5 & cancer & $2.09\times 10^{-27}$ \\
6 & mean arterial blood pressure & $3.01\times 10^{-18}$ \\
7 & respiration rate & $2.43\times 10^{-12}$ \\
8 & number of comorbidities & $3.80\times 10^{-8}$ \\
9 & heart rate & $4.04\times 10^{-8}$ \\
10 & white blood count & $2.87\times 10^{-5}$ \\
11 & serum creatinine & $2.89\times 10^{-3}$ \\
12 & race & $1.78\times 10^{-2}$ \\
13 & diabetes & $3.56\times 10^{-2}$ \\
14 & dementia & $6.41\times 10^{-1}$ \\
\bottomrule
\end{tabular}
} \\[.75em]
\adjustbox{width=.425\linewidth}{
\begin{tabular}{rll}
\toprule
\multicolumn{3}{c}{Cluster 3} \\
\midrule
Rank & Feature & p-value \\ \midrule
1 & cancer & $6.26\times 10^{-183}$ \\
2 & temperature & $2.93\times 10^{-75}$ \\
3 & age & $1.75\times 10^{-37}$ \\
4 & female & $1.35\times 10^{-24}$ \\
5 & number of comorbidities & $3.05\times 10^{-20}$ \\
6 & white blood count & $6.74\times 10^{-16}$ \\
7 & heart rate & $1.58\times 10^{-12}$ \\
8 & respiration rate & $1.04\times 10^{-10}$ \\
9 & mean arterial blood pressure & $1.28\times 10^{-9}$ \\
10 & serum sodium & $1.39\times 10^{-9}$ \\
11 & serum creatinine & $1.19\times 10^{-8}$ \\
12 & race & $2.44\times 10^{-7}$ \\
13 & dementia & $6.06\times 10^{-5}$ \\
14 & diabetes & $6.16\times 10^{-3}$ \\
\bottomrule
\end{tabular}
}
~
\adjustbox{width=.425\linewidth}{
\begin{tabular}{rll}
\toprule
\multicolumn{3}{c}{Cluster 4} \\
\midrule
Rank & Feature & p-value \\ \midrule
1 & cancer & $6.12\times 10^{-185}$ \\
2 & age & $1.48\times 10^{-99}$ \\
3 & number of comorbidities & $1.09\times 10^{-18}$ \\
4 & mean arterial blood pressure & $4.24\times 10^{-15}$ \\
5 & dementia & $1.29\times 10^{-13}$ \\
6 & diabetes & $1.71\times 10^{-8}$ \\
7 & temperature & $1.85\times 10^{-6}$ \\
8 & serum sodium & $1.37\times 10^{-4}$ \\
9 & serum creatinine & $1.65\times 10^{-4}$ \\
10 & respiration rate & $1.69\times 10^{-4}$ \\
11 & race & $4.96\times 10^{-4}$ \\
12 & female & $7.54\times 10^{-4}$ \\
13 & white blood count & $4.61\times 10^{-3}$ \\
14 & heart rate & $1.53\times 10^{-2}$ \\
\bottomrule
\end{tabular}
} \\[.75em]
\adjustbox{width=.425\linewidth}{
\begin{tabular}{rll}
\toprule
\multicolumn{3}{c}{Cluster 5} \\
\midrule
Rank & Feature & p-value \\ \midrule
1 & age & $4.47\times 10^{-130}$ \\
2 & cancer & $4.88\times 10^{-114}$ \\
3 & temperature & $1.09\times 10^{-42}$ \\
4 & serum sodium & $2.35\times 10^{-17}$ \\
5 & mean arterial blood pressure & $6.81\times 10^{-17}$ \\
6 & number of comorbidities & $9.90\times 10^{-14}$ \\
7 & dementia & $1.66\times 10^{-13}$ \\
8 & race & $6.59\times 10^{-4}$ \\
9 & heart rate & $2.06\times 10^{-3}$ \\
10 & respiration rate & $2.05\times 10^{-2}$ \\
11 & serum creatinine & $7.25\times 10^{-2}$ \\
12 & diabetes & $8.99\times 10^{-2}$ \\
13 & white blood count & $2.60\times 10^{-1}$ \\
14 & female & $3.98\times 10^{-1}$ \\
\bottomrule
\end{tabular}
}
\end{table*}

\begin{figure*}[t!]
\centering
\includegraphics[width=\linewidth]{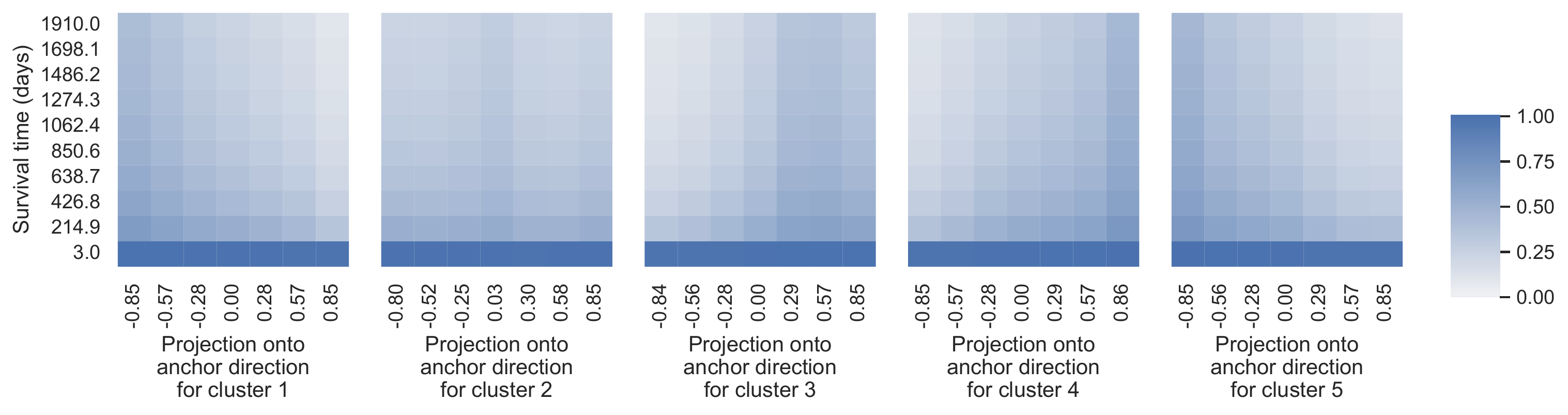}
\vspace{-2.5em}
\caption{SUPPORT dataset: survival probability heatmaps for a DeepSurv model where the encoder has a Euclidean norm~1 constraint. These heatmaps are for all 5 clusters' anchor directions (clusters are from a 5-component mixture of von Mises-Fisher distributions). The heatmap only for cluster~1 is also in \figureref{fig:support-anchor1-survival-heatmap}.\label{fig:support-anchor1through5-survival-heatmap}}
\vspace{-1.5em}
\end{figure*}

\begin{figure}[t!]
\centering
\includegraphics[width=.825\linewidth]{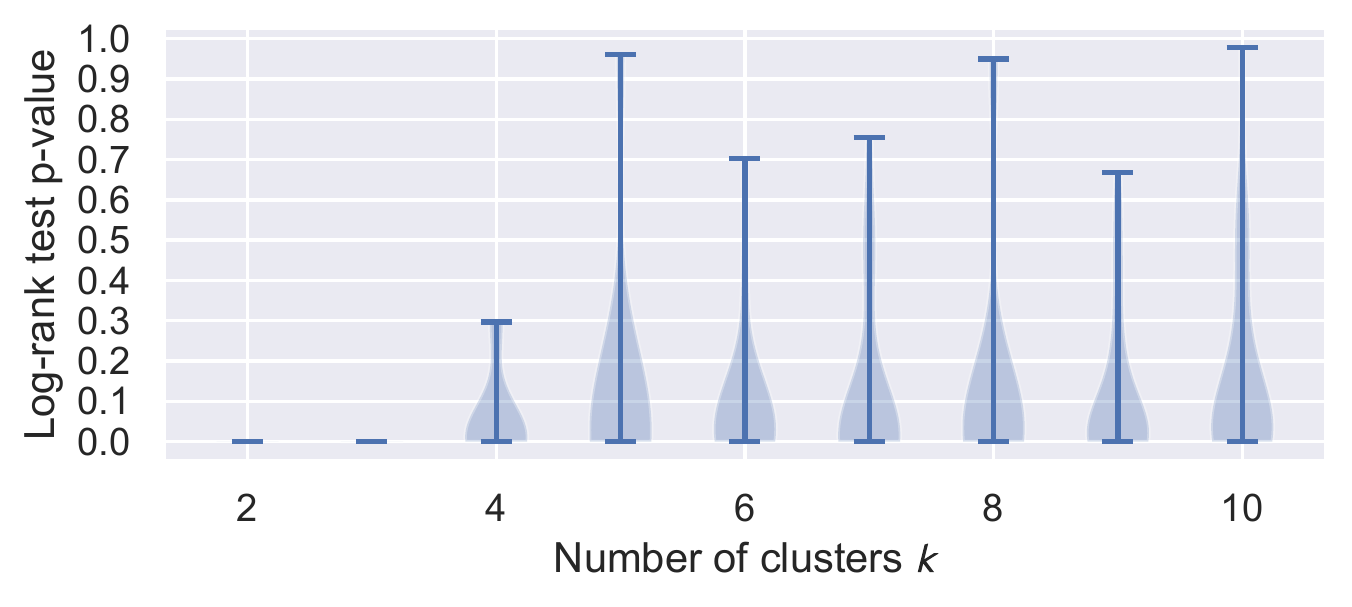}
\vspace{-1.25em}
\caption{SUPPORT dataset: a violin plot to help select the number of clusters (and thus the number of anchor directions) to use with a clustering model. Here, the encoder used is from a DeepSurv model that has a Euclidean norm~1 constraint, and the clustering model is a Gaussian mixture~model.}
\label{fig:support-logrank-helper-hypersphere-gmm}
\vspace{-2em}
\end{figure}

\begin{figure}[p!]
\centering
\includegraphics[width=.825\linewidth]{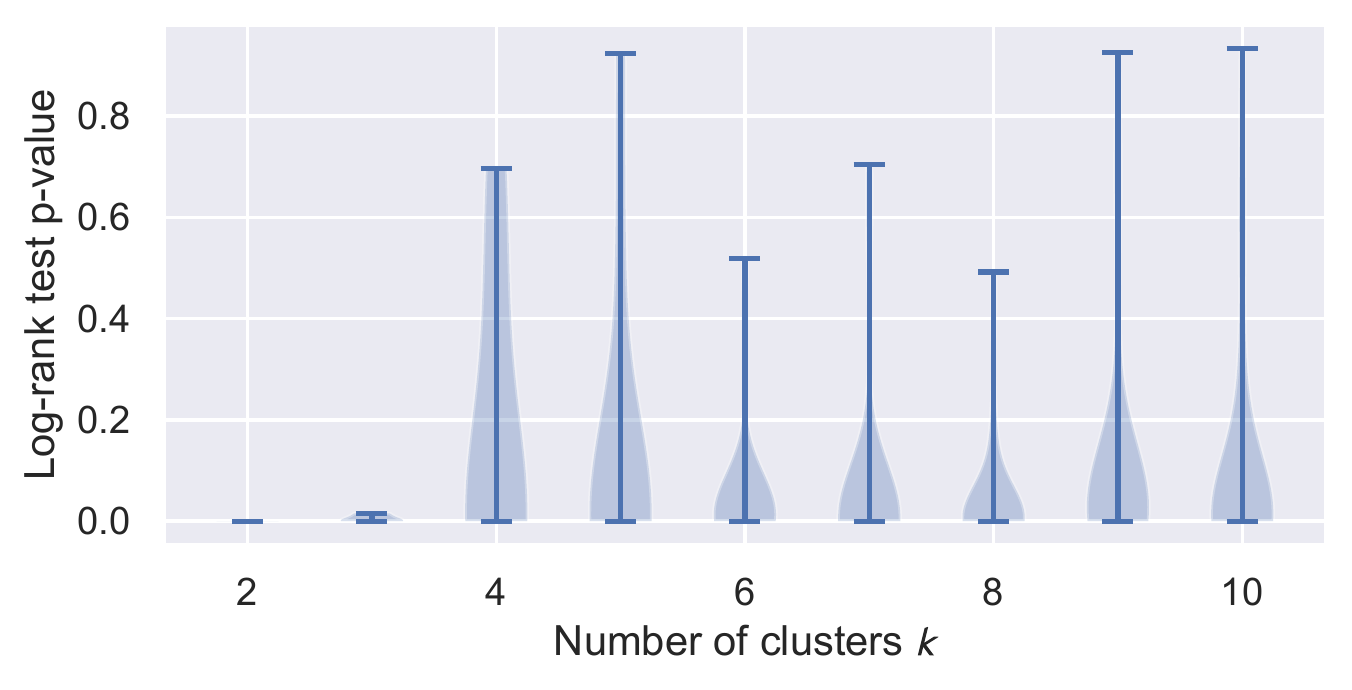}
\vspace{-1.25em}
\caption{SUPPORT dataset: a violin plot to help select the number of clusters (and thus the number of anchor directions) to use with a clustering model. Here, the encoder used is from a DeepSurv model that does \emph{not} have a Euclidean norm~1 constraint, and the clustering model is a Gaussian mixture model.}
\label{fig:support-logrank-helper-no-hypersphere}
\end{figure}

\begin{figure}[p!]
\centering
\includegraphics[width=.62\linewidth, trim={5pt 0pt 0pt 5pt}, clip]{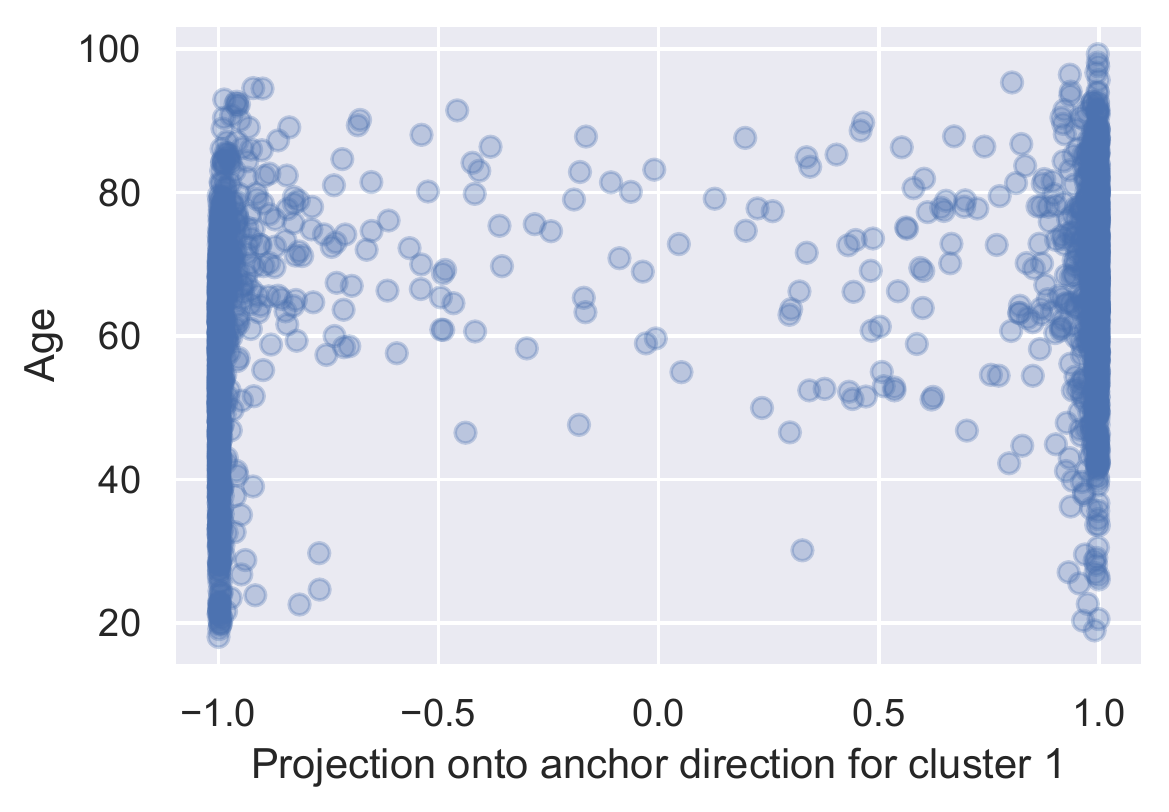}

\includegraphics[width=.62\linewidth, trim={5pt 0pt 0pt 5pt}, clip]{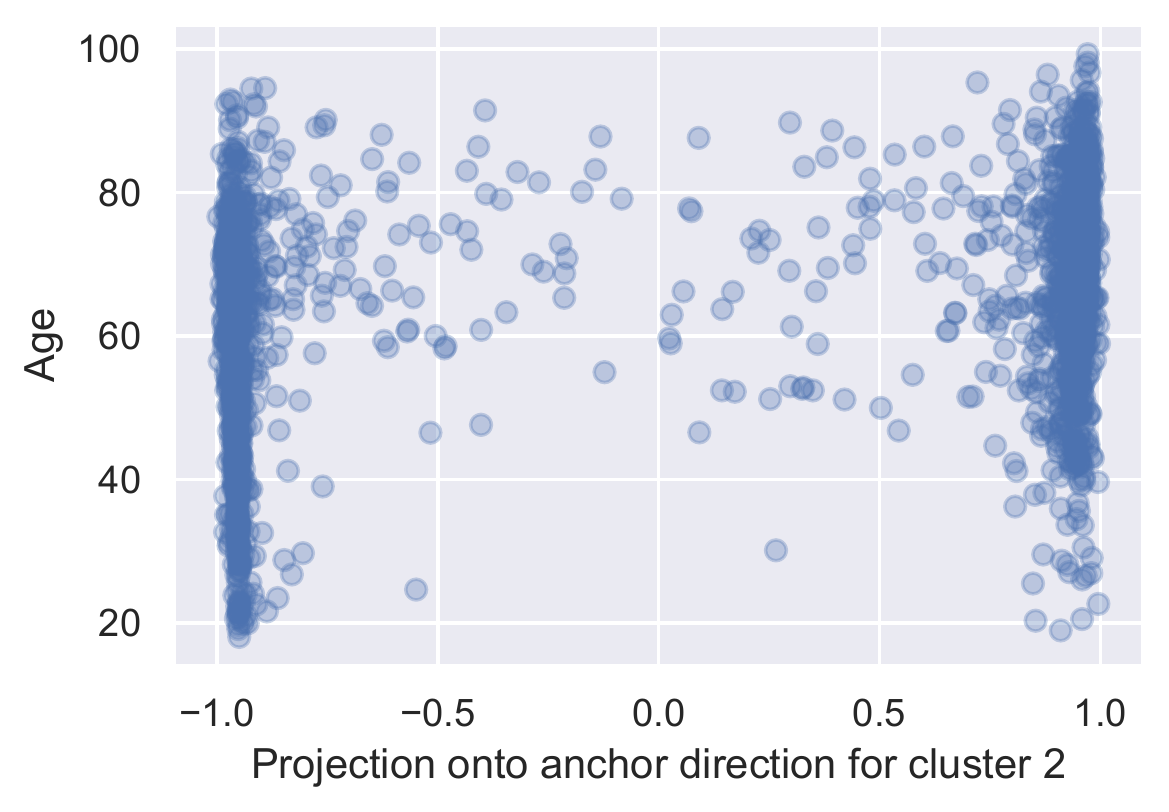}

\includegraphics[width=.62\linewidth, trim={5pt 0pt 0pt 5pt}, clip]{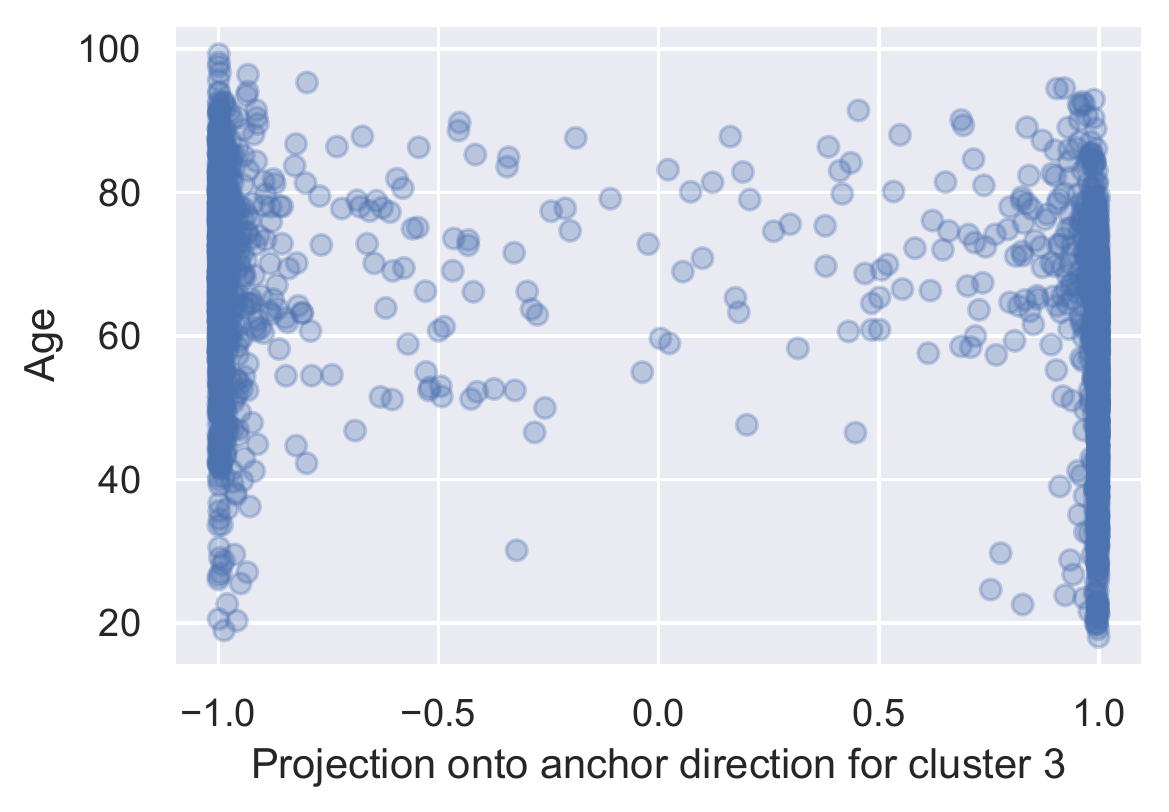}
\vspace{-1.25em}
\caption{SUPPORT dataset: scatter plots of age vs projection values for a DeepSurv model where the encoder does \emph{not} have a Euclidean norm~1 constraint. The projection values are along each cluster's anchor direction (for all 3 clusters in a 3-component Gaussian mixture model).}
\label{fig:support-age-vs-proj-no-hypersphere}
\end{figure}

\begin{figure*}[p]
\centering
\includegraphics[width=.95\linewidth]{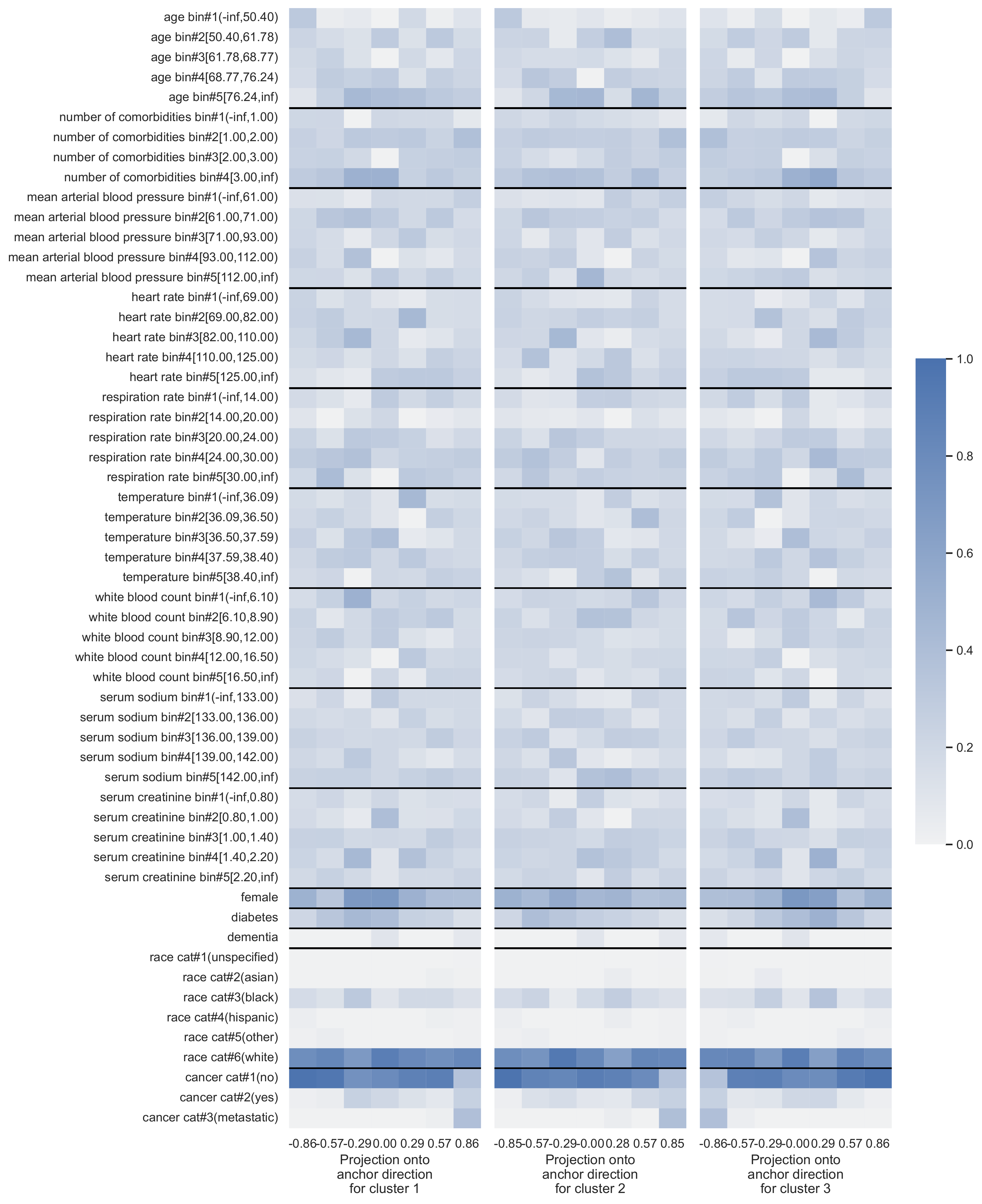}
\vspace{-1.5em}
\caption{SUPPORT dataset: raw feature probability heatmaps for a DeepSurv model where the encoder does \emph{not} have a Euclidean norm~1 constraint. These heatmaps are for each cluster's anchor direction (for all 3 clusters in a 3-component Gaussian mixture model).}
\label{fig:support-raw-feature-probability-heatmaps-no-hypersphere}
\end{figure*}

\begin{table}[p!]
\vspace{-2em}
\caption{SUPPORT dataset: rankings of raw features based on the \mbox{p-value} of Pearson's chi-squared test for a DeepSurv model where the encoder does \emph{not} have a Euclidean norm~1 constraint. These tables are for all 3 clusters in a 3-component Gaussian mixture model.\label{tab:support-feature-ranking-no-hypersphere}}
\vspace{-.75em}
\centering
\adjustbox{width=.85\linewidth}{
\begin{tabular}{rll}
\toprule
\multicolumn{3}{c}{Cluster 1} \\
\midrule
Rank & Feature & p-value \\ \midrule
1 & cancer & $4.35\times 10^{-176}$ \\
2 & age & $2.90\times 10^{-34}$ \\
3 & dementia & $9.15\times 10^{-15}$ \\
4 & number of comorbidities & $4.76\times 10^{-14}$ \\
5 & heart rate & $6.93\times 10^{-9}$ \\
6 & mean arterial blood pressure & $2.48\times 10^{-8}$ \\
7 & diabetes & $3.43\times 10^{-5}$ \\
8 & female & $4.77\times 10^{-5}$ \\
9 & white blood count & $5.63\times 10^{-4}$ \\
10 & temperature & $5.94\times 10^{-3}$ \\
11 & respiration rate & $1.31\times 10^{-1}$ \\
12 & serum creatinine & $3.13\times 10^{-1}$ \\
13 & serum sodium & $3.32\times 10^{-1}$ \\
14 & race & $5.28\times 10^{-1}$ \\
\bottomrule
\end{tabular}
}
\\[.75em]
\adjustbox{width=.85\linewidth}{
\begin{tabular}{rll}
\toprule
\multicolumn{3}{c}{Cluster 2} \\
\midrule
Rank & Feature & p-value \\ \midrule
1 & cancer & $3.41\times 10^{-175}$ \\
2 & age & $8.46\times 10^{-36}$ \\
3 & dementia & $6.14\times 10^{-15}$ \\
4 & number of comorbidities & $1.38\times 10^{-12}$ \\
5 & heart rate & $7.64\times 10^{-10}$ \\
6 & mean arterial blood pressure & $1.98\times 10^{-8}$ \\
7 & diabetes & $6.49\times 10^{-5}$ \\
8 & female & $9.97\times 10^{-4}$ \\
9 & temperature & $8.66\times 10^{-3}$ \\
10 & white blood count & $1.71\times 10^{-2}$ \\
11 & serum sodium & $2.14\times 10^{-2}$ \\
12 & race & $2.45\times 10^{-1}$ \\
13 & serum creatinine & $2.45\times 10^{-1}$ \\
14 & respiration rate & $4.42\times 10^{-1}$ \\
\bottomrule
\end{tabular}
} \\[.75em]
\adjustbox{width=.85\linewidth}{
\begin{tabular}{rll}
\toprule
\multicolumn{3}{c}{Cluster 3} \\
\midrule
Rank & Feature & p-value \\ \midrule
1 & cancer & $1.34\times 10^{-175}$ \\
2 & age & $4.95\times 10^{-34}$ \\
3 & dementia & $9.15\times 10^{-15}$ \\
4 & number of comorbidities & $3.29\times 10^{-14}$ \\
5 & heart rate & $1.48\times 10^{-8}$ \\
6 & mean arterial blood pressure & $8.29\times 10^{-8}$ \\
7 & diabetes & $6.40\times 10^{-6}$ \\
8 & female & $1.25\times 10^{-4}$ \\
9 & white blood count & $1.49\times 10^{-3}$ \\
10 & temperature & $1.31\times 10^{-2}$ \\
11 & respiration rate & $1.59\times 10^{-1}$ \\
12 & serum creatinine & $1.60\times 10^{-1}$ \\
13 & race & $2.18\times 10^{-1}$ \\
14 & serum sodium & $3.04\times 10^{-1}$ \\
\bottomrule
\end{tabular}
}
\end{table}

\begin{figure*}[p!]
\centering
\includegraphics[scale=.6]{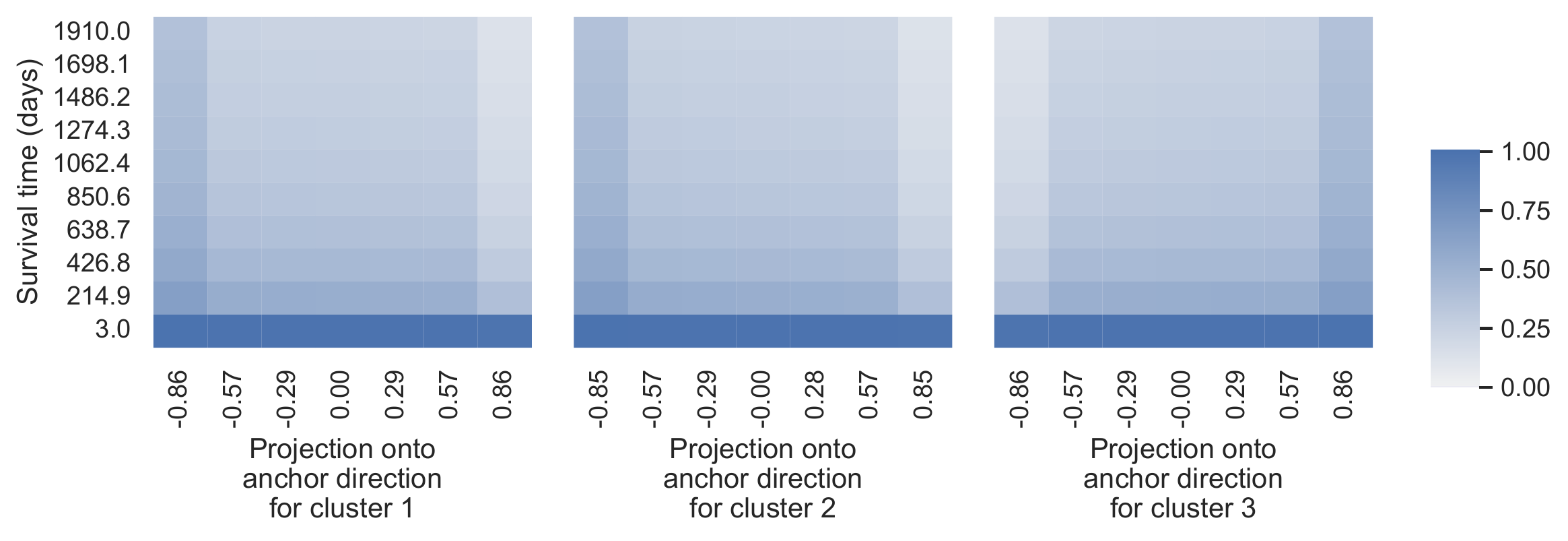}
\vspace{-1.5em}
\caption{SUPPORT dataset: survival probability heatmaps for a DeepSurv model where the encoder does \emph{not} have a Euclidean norm~1 constraint. These heatmaps are for each cluster's anchor direction (for all 3 clusters in a 3-component Gaussian mixture model).}
\label{fig:support-survival-probability-heatmaps-no-hypersphere}
\end{figure*}

\abovesubsectionskip
\subsection{Visualizations Using an Encoder \emph{Without} the Euclidean Norm~1 Constraint}
\label{sec:support-additional-results-no-hypersphere}
\belowsubsectionskip

We now present results using the exact same setup as described in Section~\ref{sec:tabular} and detailed in Appendices~\ref{sec:support-hyperparameter-grid} and~\ref{sec:support-additional-results}, where the only differences are that: (i) the final nonlinear activation layer in the encoder~$\phi$ is ReLU instead of dividing the intermediate representation by its Euclidean norm, and (ii) the clustering model used in the embedding space is a Gaussian mixture model. For reference, this model achieves a test set concordance index of 0.615, which is close to what was achieved with the model that includes the Euclidean norm~1 constraint.

The violin plot for selecting the number of clusters is shown in \figureref{fig:support-logrank-helper-no-hypersphere}, where we choose the number of clusters to be $k=3$. For this 3-component Gaussian mixture model, we find the anchor directions corresponding to its 3 clusters and then show scatter plots of age vs anchor projections of all 3 clusters in \figureref{fig:support-age-vs-proj-no-hypersphere}, raw feature probability heatmaps for all 3 clusters in \figureref{fig:support-raw-feature-probability-heatmaps-no-hypersphere}, raw feature rankings in Table~\ref{tab:support-feature-ranking-no-hypersphere}, and survival probability heatmaps for all 3 clusters in \figureref{fig:support-survival-probability-heatmaps-no-hypersphere}.

A key point we want to emphasize is that in the scatter plots (\figureref{fig:support-age-vs-proj-no-hypersphere}), we can see the ``clumping up'' artifact we mentioned in Section~\ref{sec:information-loss} that indicates that there is likely a lot of magnitude information lost (specifically, a lot of the points in the scatter plot ``clump up'' around $-1$ an $1$).

In this case, the top features across clusters are largely the same (Table~\ref{tab:support-feature-ranking-no-hypersphere}). From looking at the %
raw feature probability heatmaps (\figureref{fig:support-raw-feature-probability-heatmaps-no-hypersphere}), focusing on the largest projection bin per heatmap, we find that clusters 1 and 2 actually appear quite similar: both appear to be associated with older patients who often have cancer and at least one comorbidity. In contrast, cluster 3 appears to be associated with younger patients without cancer. Meanwhile, the survival probability heatmap (\figureref{fig:support-survival-probability-heatmaps-no-hypersphere}) indicates that indeed clusters 1 and 2 are associated with survival functions that decay quickly (indicative of the survival times tending to be shorter) whereas cluster 3 is associated with a survival function that decays slowly. These are the same main findings as for the version of the model that used the Euclidean norm~1 constraint. In fact, even if we used 2 clusters, we get the same main findings.

We point out that the major challenge when there is a lot of information loss due to magnitudes being ignored is that our visualization heatmaps will end up each consisting of essentially only two projection bins along the x-axis (that have enough data in them: one that contains projection values around $-1$ and the other that contains projection values around $1$). In this case, we could still of course find interesting relationships of how a raw feature changes with respect to an anchor direction but we would only be seeing what this change looks like (if there is any) at two x-axis values. %
We would only be able to check for monotonic trends between how a raw feature relates to two projection values along a specific anchor direction.

\abovesectionskip
\section{Rotterdam/GBSG Experiment}
\label{sec:rotterdam-gbsg}
\belowsectionskip

\begin{table*}[p!]
\caption{Rotterdam/GBSG datasets: rankings of raw features based on the \mbox{p-value} of Pearson's chi-squared test for a DeepSurv model where the encoder has a Euclidean norm~1 constraint. These tables are for all 3 clusters from a 3-component mixture of von Mises-Fisher distributions.\label{tab:rotterdam-gbsg-feature-ranking}}
\vspace{-.75em}
\centering
\adjustbox{width=.425\linewidth}{
\begin{tabular}{rll}
\toprule
\multicolumn{3}{c}{Cluster 1} \\
\midrule
Rank & Feature & p-value \\ \midrule
1 & age & $6.22\times 10^{-50}$ \\
2 & postmenopausal & $1.36\times 10^{-40}$ \\
3 & number of positive nodes & $3.51\times 10^{-27}$ \\
4 & estrogen receptor & $2.67\times 10^{-14}$ \\
5 & progesterone receptor & $1.99\times 10^{-7}$ \\
6 & tumor size & $4.11\times 10^{-5}$ \\
7 & hormonal therapy & $8.39\times 10^{-3}$ \\
\bottomrule
\end{tabular}
} %
~
\adjustbox{width=.425\linewidth}{
\begin{tabular}{rll}
\toprule
\multicolumn{3}{c}{Cluster 2} \\
\midrule
Rank & Feature & p-value \\ \midrule
1 & number of positive nodes & $1.58\times 10^{-98}$ \\
2 & hormonal therapy & $4.32\times 10^{-14}$ \\
3 & tumor size & $2.52\times 10^{-9}$ \\
4 & progesterone receptor & $6.66\times 10^{-6}$ \\
5 & age & $1.47\times 10^{-5}$ \\
6 & estrogen receptor & $5.02\times 10^{-4}$ \\
7 & postmenopausal & $6.31\times 10^{-2}$ \\
\bottomrule
\end{tabular}
} \\[.75em]
\adjustbox{width=.425\linewidth}{
\begin{tabular}{rll}
\toprule
\multicolumn{3}{c}{Cluster 3} \\
\midrule
Rank & Feature & p-value \\ \midrule
1 & age & $1.07\times 10^{-35}$ \\
2 & postmenopausal & $5.56\times 10^{-33}$ \\
3 & number of positive nodes & $2.61\times 10^{-30}$ \\
4 & hormonal therapy & $2.47\times 10^{-10}$ \\
5 & estrogen receptor & $2.05\times 10^{-3}$ \\
6 & tumor size & $2.61\times 10^{-1}$ \\
7 & progesterone receptor & $2.62\times 10^{-1}$ \\
\bottomrule
\end{tabular}
}
\end{table*}

As mentioned in this main paper, we also have visualizations where we train on the Rotterdam dataset (using the exact same neural network architecture as we used for SUPPORT, including the Euclidean norm~1 constraint; we specifically use the same hyperparameter grid and optimization strategy as detailed in Appendix~\ref{sec:support-hyperparameter-grid}). For reference, when we test on the GBSG dataset, we get a concordance index of 0.677. For the visualizations to follow, we randomly choose 25\% of the GBSG dataset to treat as the anchor direction estimation data, and we use the feature vectors from the rest of the data as the visualization raw inputs $x_1^{\viz},\dots,x_{n^{\viz}}^{\viz}$. Note that technically how we have set up the DeepSurv model here violates the i.i.d.~assumption between training and anchor direction estimation data, as well as the assumption that the visualization raw inputs come from the same distribution as the raw inputs of the training data. However, our visualization framework still actually works when the training data are sampled differently. We discuss this in a bit more detail next before going over the resulting visualizations.

\smallskip
\noindent
\textbf{A different distribution for training data.}
In Section~\ref{sec:framework}, when we discussed statistical assumptions and sample splitting, we had, for simplicity, assumed that the training data and anchor direction estimation data were sampled i.i.d.~from the distribution (which technically is a joint distribution $\mathbb{P}_{X,Y,\Delta}$ defined for raw input $X$ with observed time $Y$ and event indicator $\Delta$; here, $X$, $Y$, and $\Delta$ are random variables), and that the visualization raw inputs are drawn from the marginal raw input distribution $\mathbb{P}_X$. In fact, the training data could be sampled differently so long as the anchor direction estimation data and visualization raw inputs are sampled in a manner that is independent of the training data, which ensures that we do not encounter the issue stated in Appendix~\ref{sec:statistical-considerations-anchor-direction-estimation}.

To be more precise, our visualization framework still holds if the anchor direction estimation data are sampled i.i.d.~from $\mathbb{P}_{X,Y,\Delta}$ and the visualization raw inputs are sampled from $\mathbb{P}_X$, but now the training data are sampled i.i.d.~from some other distribution $\mathbb{Q}_{X,Y,\Delta}$ and the training data are independent from the anchor direction estimation data and the visualization raw inputs. After all, the statistical analyses we conduct are all conditioned on the training data and the encoder $\phi$; we just needed to ensure that conditioning on the training data and~$\phi$ did not result in dependence between anchor direction estimation data or the visualization raw inputs.

\smallskip
\noindent
\textbf{Visualization results and interpretations.}
We show the violin plot for selecting the number of clusters for a mixture of von Mises-Fisher distributions in \figureref{fig:rotterdam-gbsg-logrank-helper}, where we choose the number of clusters to be $k=3$. For this 3-component mixture model, we find the anchor directions corresponding to its 3 clusters and then show raw feature probability heatmaps for all 3 clusters in \figureref{fig:rotterdam-gbsg-raw-feature-prob-heatmaps}, raw feature rankings in Table~\ref{tab:rotterdam-gbsg-feature-ranking}, and survival probability heatmaps for all 3 clusters in \figureref{fig:rotterdam-gbsg-surv-prob-heatmaps}.

The main findings from our visualizations are as follows: clusters~1 and~3 are both associated with patients who tend to have very few lymph nodes that contain cancer, which in turn is associated with longer survival times (the survival probability functions in the rightmost of the survival probability heatmaps for clusters~1 and~3 decay slowly). Interestingly, clusters~1 and~3 differ largely in that cluster 1 is for older women (being postmenopausal is flagged as being very probable for cluster 1 which is of course related to age) and cluster 3 is for younger women (who have not undergone menopause). Meanwhile, the anchor direction for cluster~2 is associated with women who have a large number of lymph nodes that have cancer and the tumor sizes tend to be large as well. This of course means that cancer has advanced quite significantly and, unsurprisingly, cluster~2 is associated with shorter survival times (its survival probability heatmap's rightmost column decays quickly).

\begin{figure*}[p!]
\floatconts
  {fig:rotterdam-gbsg}
  {\caption{Rotterdam/GBSG datasets: visualizations for a DeepSurv model where the encoder has a Euclidean norm~1 constraint. Panel $(a)$ shows a violin plot to help select the number of clusters (and thus the number of anchor directions) for use with a mixture of von Mises-Fisher distributions, where we choose $k=3$ for the subsequent panels. Panel $(b)$ shows raw feature probability heatmaps for the three clusters' estimated anchor directions. Panel $(c)$ shows the survival probability heatmaps for the same anchor directions as panel $(b)$. Details on the dataset, encoder, and clustering model are in Appendix~\ref{sec:rotterdam-gbsg}.}}
  {
    \subfigure[]{\label{fig:rotterdam-gbsg-logrank-helper}
      \includegraphics[scale=.45, trim={0pt 9pt 0pt 0pt}, clip]{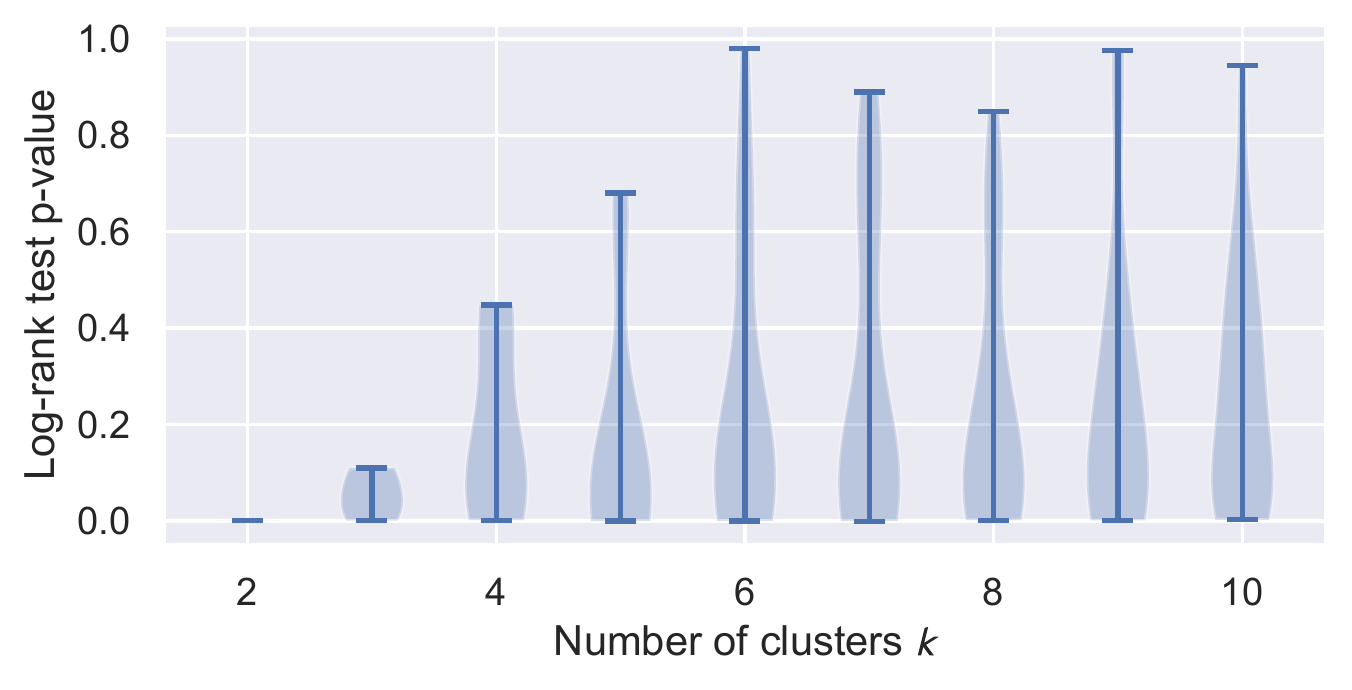}
    }
    \subfigure[]{\label{fig:rotterdam-gbsg-raw-feature-prob-heatmaps}
      \includegraphics[scale=.45, trim={0pt 9pt 0pt 0pt}, clip]{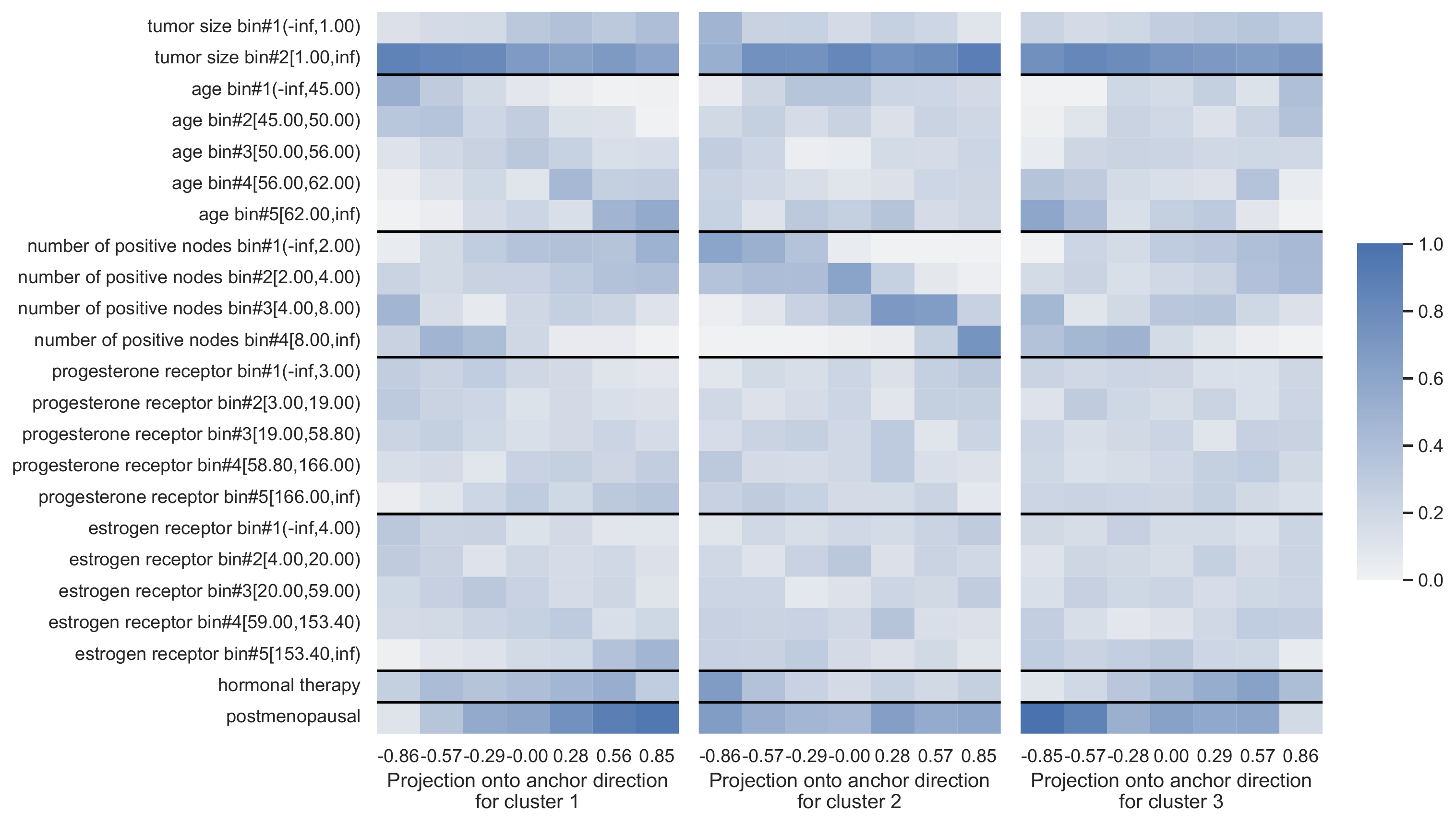}
    }
    \subfigure[]{\label{fig:rotterdam-gbsg-surv-prob-heatmaps}
      \includegraphics[scale=.45, trim={0pt 9pt 0pt 0pt}, clip]{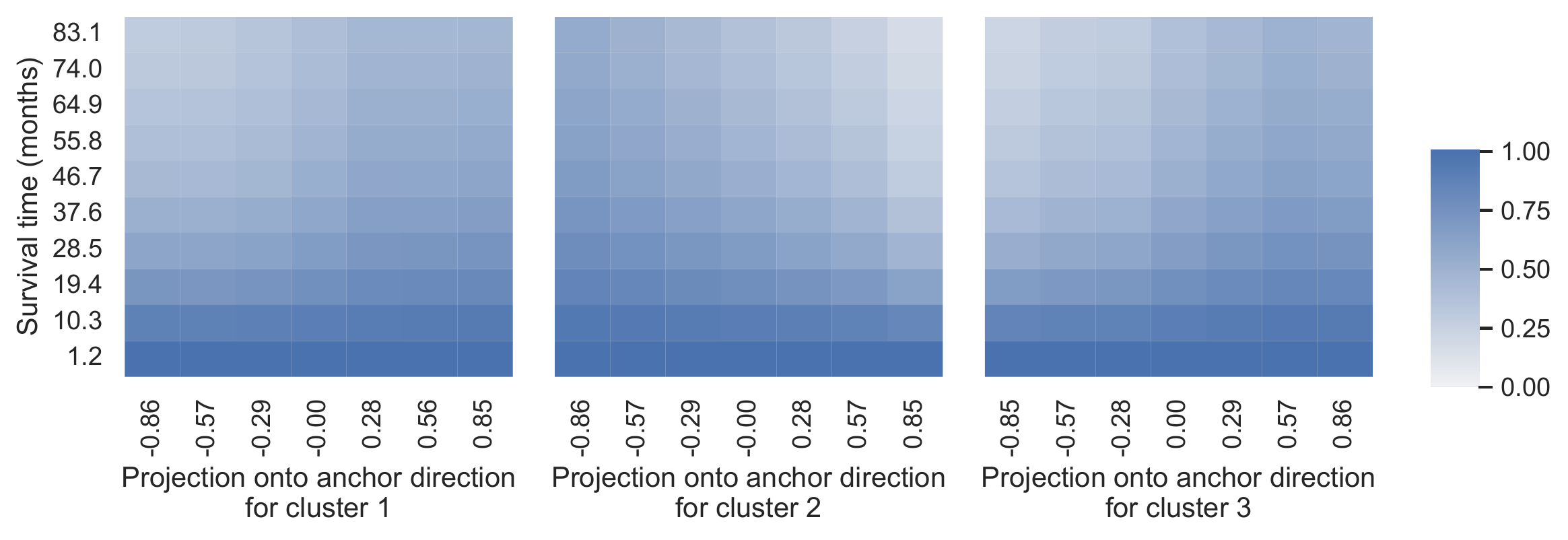}
    }
    \vspace{-1.75em}
  }
\end{figure*}

\abovesectionskip
\section{Tabular Data: Finding Interactions Between Raw Features in Predicting Projection Values Along an Anchor Direction}
\label{sec:interactions}
\belowsectionskip

\begin{figure*}[t!]
\centering
\includegraphics[width=.85\linewidth]{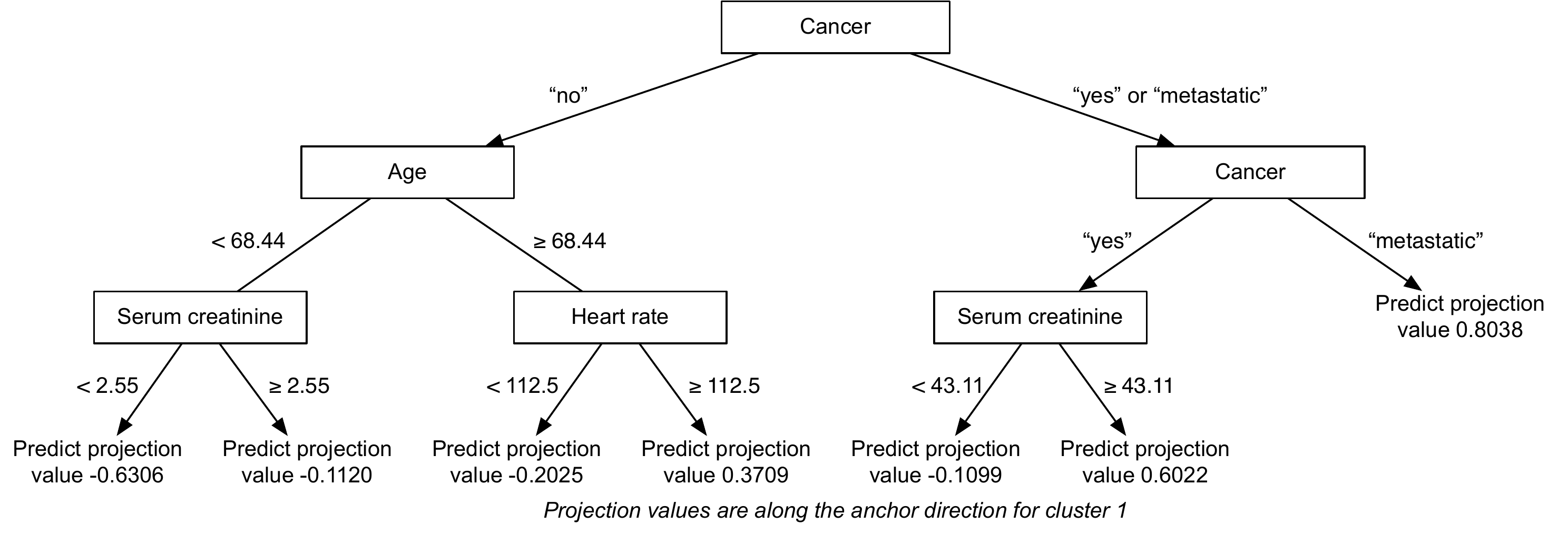}
\vspace{-1em}
\caption{SUPPORT dataset: using a DeepSurv model where the encoder has a Euclidean norm~1 constraint, we show an optimal regression tree trained using visualization raw inputs (feature vectors) to predict projection values (regression target labels) along cluster 1's anchor direction. Cluster~1 is the same cluster that we provided visualizations for throughout Section~\ref{sec:tabular} and in Appendix~\ref{sec:support-additional-results}.}
\label{fig:sparse-regression-tree}
\vspace{-1.5em}
\end{figure*}

We mention two methods that can be used to probe raw feature interactions in predicting projection values along an anchor direction~$\mu$.

\smallskip
\noindent
\textbf{Fitting a regression model that is ``easy to interpret'' and accounts for feature interactions.}
One way to find possible interactions is to fit any regression model that is straightforward to interpret and that can surface possible feature interactions, using feature vectors $x_1^{\viz},\dots,x_{n^{\viz}}^{\viz}$ and target regression labels $p_1^{\viz},\dots,p_{n^{\viz}}^{\viz}$. For example, we could fit a so-called \emph{optimal regression tree} using mixed-integer optimization \citep{dunn2018optimal}. An example of such a tree is shown in \figureref{fig:sparse-regression-tree}. The reason why this tree encodes feature interaction information is apparent when we look at any leaf. For example, the leftmost leaf with predicted projection value $-0.6306$ corresponds to the intersection of the constraints ``cancer $=$ no'', ``age $<$ 68.44'', and ``serum creatinine $<$ 2.55'', showing an interaction between the variables ``cancer'', ``age'', and ``serum creatinine'' that depends on them satisfying specific inequalities.

Note that this sort of approach should be used with care. Specifically, by using different splits of data (e.g., different train/validation splits) to train the tree, it is possible that the resulting trees look different and might suggest different raw features to matter, and the raw features that interact might vary across experimental repeats. Another issue is that such tree learning algorithms have hyperparameter(s) for controlling the tree complexity, such as the max tree depth, and for different such hyperparameter choices, the raw features that are found to be important or that are shown to interact might vary.

\smallskip
\noindent
\textbf{Archipelogo.}
As an alternative approach to finding possible feature interactions, we point out that one could use the \mbox{Archipelogo} framework \citep{tsang2020does} that is designed to find feature interactions of a black-box model (in this case, the encoder $\phi$) when provided with specific raw inputs (e.g., $x_1^{\viz},\dots,x_{n^{\viz}}^{\viz}$).

\abovesectionskip
\section{Survival MNIST Experiment}
\label{sec:survival-mnist}
\belowsectionskip

\subsection{Dataset}
\label{sec:survival-mnist-dataset}
\belowsubsectionskip

The Survival MNIST dataset builds off of the original MNIST classification dataset \citep{lecun2010mnist}, which consists of 60,000 training images and 10,000 test images. All images are 28-by-28 pixel grayscale images of handwritten digits. Each image has a target label corresponding to which of the 10 digits the image corresponds to. \citet{polsterl2019survival} modified the MNIST dataset so that the labels for training and test images are instead survival labels (observed times and event indicators) that are synthetically generated. There is some flexibility in this synthetic generation process. We specifically use the same synthetic survival label generation procedure as \citet{goldstein2020x}, who also use the Survival MNIST dataset. Specifically, for each digit $j\in\{0,1,\dots,9\}$, we let $m_j$ denote the true mean survival time for digit~$j$, where:
\begin{itemize}[leftmargin=*,itemsep=0pt,parsep=0pt,topsep=1pt]
\item $m_0 = 11.25$
\item $m_1 =  2.25$
\item $m_2 =  5.25$
\item $m_3 =  5.0$
\item $m_4 =  4.75$
\item $m_5 =  8.0$
\item $m_6 =  2.0$
\item $m_7 = 11.0$
\item $m_8 =  1.75$
\item $m_9 = 10.75$
\end{itemize}
These mean survival times are also shown in \figureref{fig:survival-mnist-ground-truth-mean-survival-times}. The digits are ranked (in increasing order of mean survival time) as: 8, 6, 1, 4, 3, 2, 5, 9, 7, 0. Note that the 10 digits are grouped into four ground truth risk groups: %
$\{0,7,9\},\{1,6,8\},\{2,3,4\},\{5\}$; within each risk group, the digits in the group have very similar ground truth mean survival times.

For each training image $x_i\in\mathcal{X}$ with digit label $\eta_i\in\{0,1,\dots,9\}$, we sample its true survival time $t_i$ from a Gamma distribution with mean $m_{\eta_i}$ and variance $10^{-3}$. After we generate survival times $t_1,\dots,t_n$ for all training data, we then sample the censoring times $c_1,\dots,c_n$ i.i.d.~from a uniform distribution between $\min\{t_1,\dots,t_n\}$ and the 90th percentile value of $t_1,\dots,t_n$. Finally, we set the training data's observed times and event indicators to be $y_i=\min\{t_i,c_i\}$ and $\delta_i=\ind\{t_i\le c_i\}$ respectively. This results in a censoring rate $(1 - \frac{1}{n}\sum_{i=1}^n \delta_i)$ of approximately 50\%. The test data are generated separately from the training data but in the same manner.

As a reminder, when training a survival analysis model with the Survival MNIST dataset, the training data are $\{(x_i,y_i,\delta_i)\}_{i=1}^n$. The true digit labels $\eta_1,\dots,\eta_n$ are not available to the training procedure. For test data, we also have access to their digit labels.

\begin{figure}[t!]
\centering
\includegraphics[width=.9\linewidth]{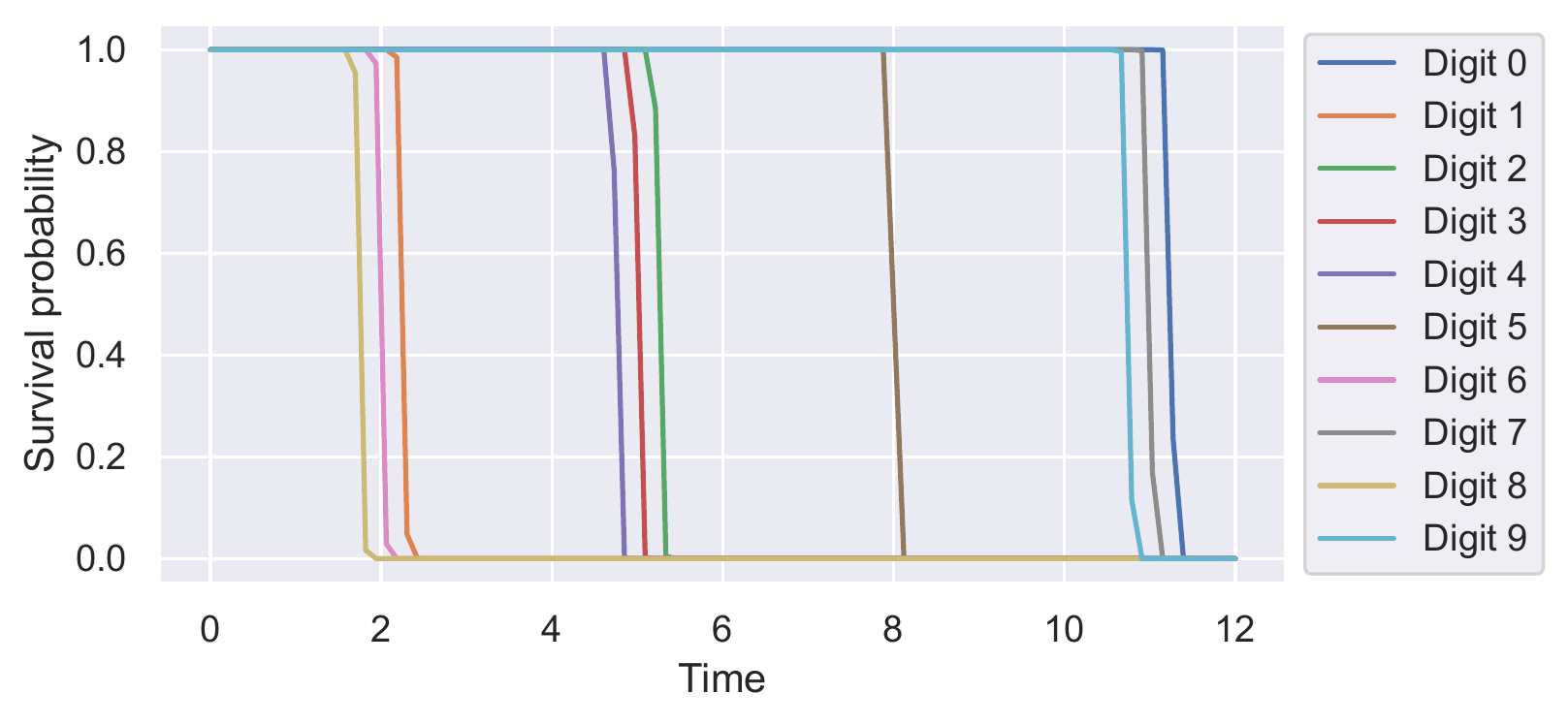}
\vspace{-1.5em}
\caption{Ground truth survival functions for the different digits.\label{fig:survival-mnist-ground-truth-survival-curves}}
\vspace{-2em}
\end{figure}

Some important remarks regarding this dataset are in order:
\begin{itemize}[leftmargin=*,itemsep=0pt,parsep=0pt,topsep=1pt]
\item \emph{The proportional hazards assumption does not hold for the underlying survival distributions of the different digits.} One way to see this is that by plotting the true survival functions of the different digits (per digit, we can get the true survival function by looking at 1~minus the CDF of the Gamma distribution associated with the digit), these survival functions are not powers of one another; we show these ground truth survival functions in \figureref{fig:survival-mnist-ground-truth-survival-curves}. (When the proportional hazards assumption holds, all survival functions are powers of an underlying ``baseline survival function''.) This means that if we fit a model with a proportional hazards assumption, such as the DeepSurv model, we should not expect the predicted conditional survival functions to be correct although these should be able to ``correctly order'' the survival times of the different digits. Recall that the digits are ordered by true mean survival time (in increasing order) as 8,~6,~1,~4,~3,~2,~5,~9,~7,~0. We would like the ``average predicted survival function for digit 8'' to be lower than that of digit 6, which should be lower than that of digit 1, and so forth. %
\item \emph{By how the censoring mechanism is set up, digits with higher mean survival times have higher censoring rates.} This is because the uniform distribution for censoring times has its maximum set to be the 90th percentile value of randomly generated survival times, so digits with high ground truth mean survival times get censored more often. For example, for the test data we generated, we have the following censoring rates for the different digits (we have ordered the digits in increasing order of ground truth mean survival time):
\begin{itemize}[leftmargin=*,itemsep=0pt,parsep=0pt,topsep=1pt]
\item Digit 8: 1.23\%
\item Digit 6: 2.92\%
\item Digit 1: 5.64\%
\item Digit 4: 34.42\%
\item Digit 3: 35.25\%
\item Digit 2: 37.98\%
\item Digit 5: 71.19\%
\item Digit 9: 96.53\%
\item Digit 7: 99.22\%
\item Digit 0: 100.00\%
\end{itemize}
Digits 0, 7, and 9 have censoring rates higher than 96\%, with digit~0's censoring rate at 100\% (in the training and test sets we generated, digit~0 is always censored). This means that that their randomly generated observed times and event indicators often look identical. Thus, when learning a neural survival analysis model using the training data, the learned embedding space (that we aim to visualize) would likely have trouble distinguishing between these three digits. Having part of the embedding space correspond only to digit~0 and not~7 or~9 would be particularly difficult.
\end{itemize}

\abovesubsectionskip
\subsection{Neural Survival Analysis Model and Encoder Setup}
\label{sec:survival-mnist-encoder}
\belowsubsectionskip

We set the base neural network $f$ to be a convolutional neural network (CNN) consisting of the following sequence of layers:
\begin{itemize}[leftmargin=*,itemsep=0pt,parsep=0pt,topsep=1pt]
\item Conv2D layer with 32 filters (each 3-by-3)
\item Nonlinear activation: ReLU
\item MaxPool2D layer (2-by-2)
\item Conv2D layer with 16 filters (each 3-by-3)
\item Nonlinear activation: ReLU
\item MaxPool2D layer (2-by-2)
\item Flatten
\item Fully-connected layer (that maps to $d$ outputs)
\item Nonlinear activation: Divide each vector by its Euclidean norm %
\item Fully-connected layer (map $d$ inputs to 1 output)
\end{itemize}
We take the encoder $\phi$ to be everything excluding the last fully-connected layer.

Note that the above choice of CNN is somewhat arbitrary. The goal of our paper here is not to provide visualizations for the very best CNN possible for Survival MNIST. Rather, we just aim to show that for a choice of CNN that is straightforward to implement, we can readily provide visualizations for one of its intermediate representations.

Just as in the tabular data setup, we train the neural network using minibatch gradient descent with at most 100 epochs and early stopping (no improvement in the validation concordance index after 10 epochs). We use Adam to optimize, and we sweep over the following hyperparameters:
\begin{itemize}[leftmargin=*,itemsep=0pt,parsep=0pt,topsep=1pt]
\item Batch size: 64, 128
\item Learning rate: 0.01, 0.001
\item Embedding dimension $d$: 10, 20, 30, 40, 50
\end{itemize}
We ran the experiments for this dataset on the same compute instance mentioned in Appendix~\ref{sec:support-hyperparameter-grid}.
After training the DeepSurv model, the model achieved a test set concordance index of 0.953.

\abovesubsectionskip
\subsection{Additional Visualizations}
\label{sec:survival-mnist-additional-results}
\belowsubsectionskip

\begin{figure}[t!]
\centering
\includegraphics[width=.9\linewidth]{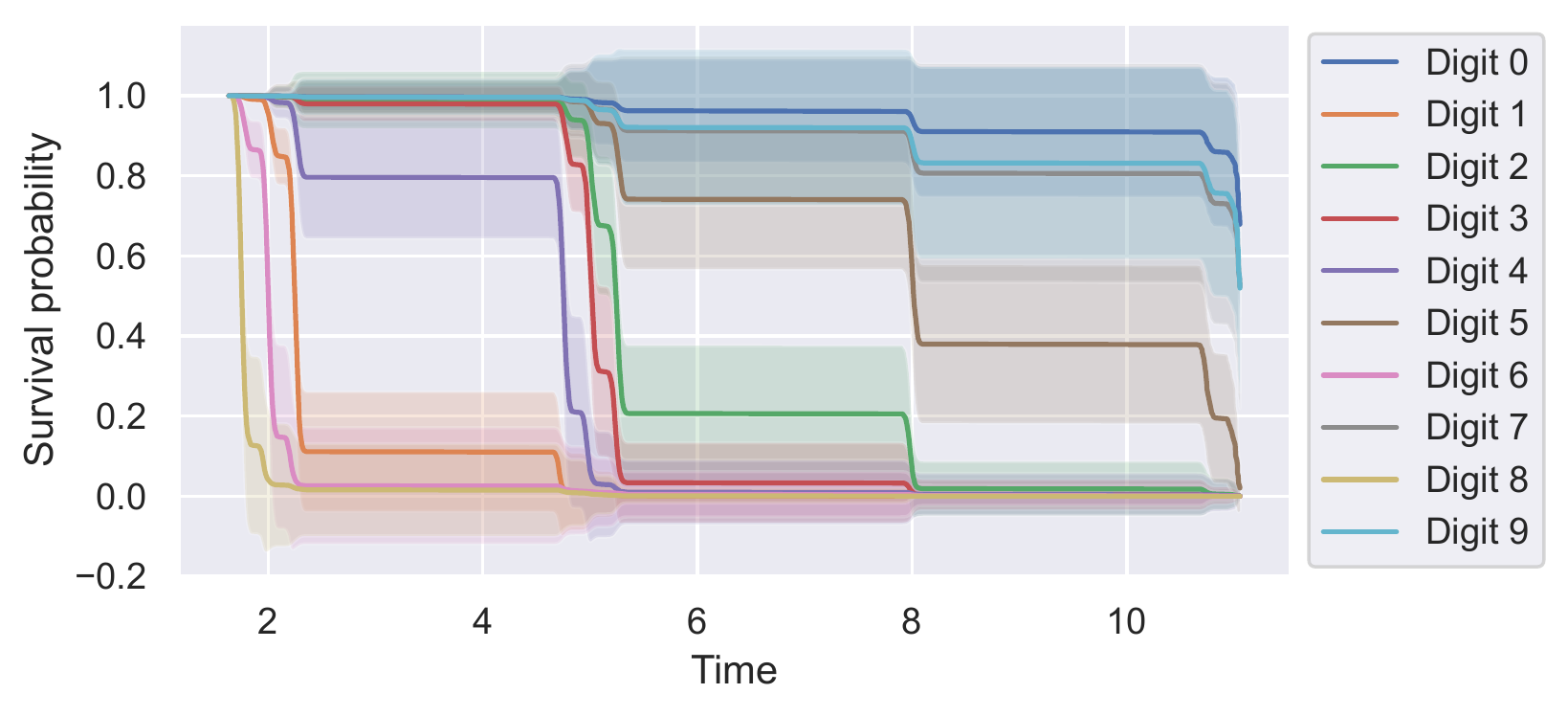}
\vspace{-1.5em}
\caption{Predicted survival functions for the different digits (mean $\pm$ standard deviation at each time point).\label{fig:survival-mnist-predicted-survival-curves}}
\vspace{-2em}
\end{figure}

\begin{figure*}[p!]
\centering
~\includegraphics[width=.49\linewidth, trim={0pt 7pt 0pt 0pt}, clip]{figures/survival-mnist-concept0-hypersphere}~\includegraphics[width=.49\linewidth, trim={0pt 7pt 0pt 0pt}, clip]{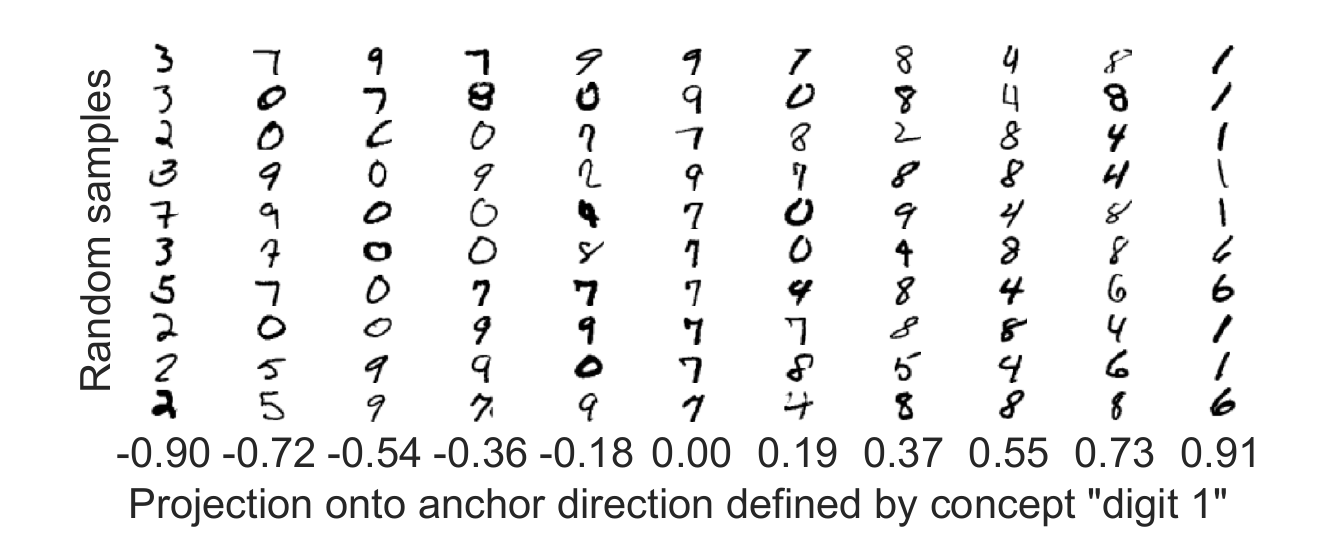} \\
~\includegraphics[width=.49\linewidth, trim={0pt 7pt 0pt 0pt}, clip]{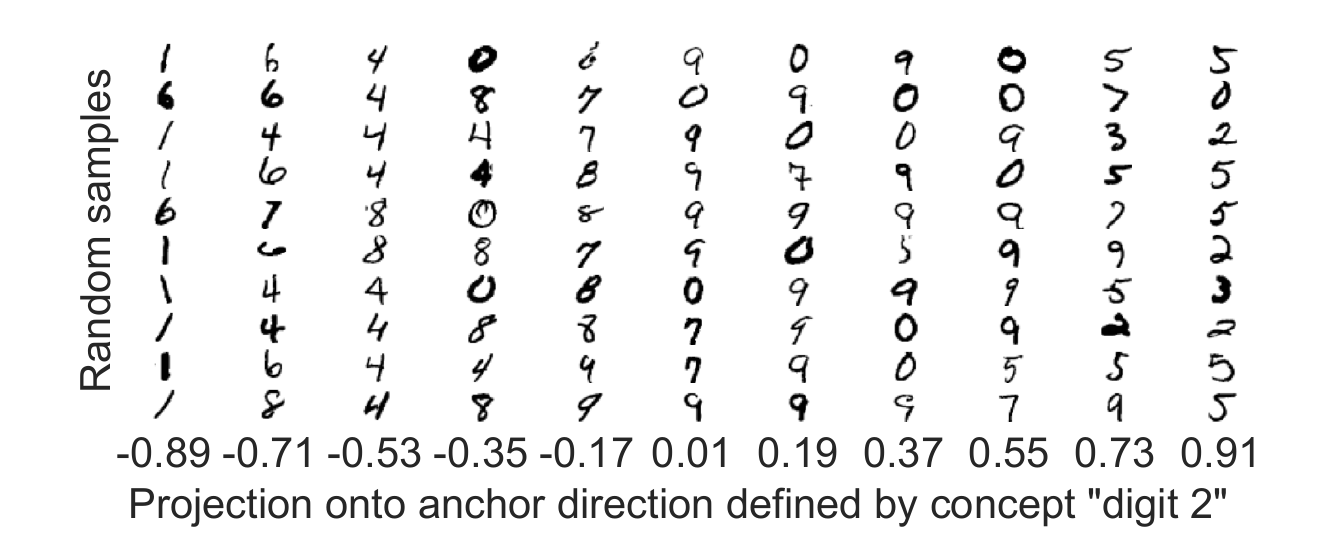}~\includegraphics[width=.49\linewidth, trim={0pt 7pt 0pt 0pt}, clip]{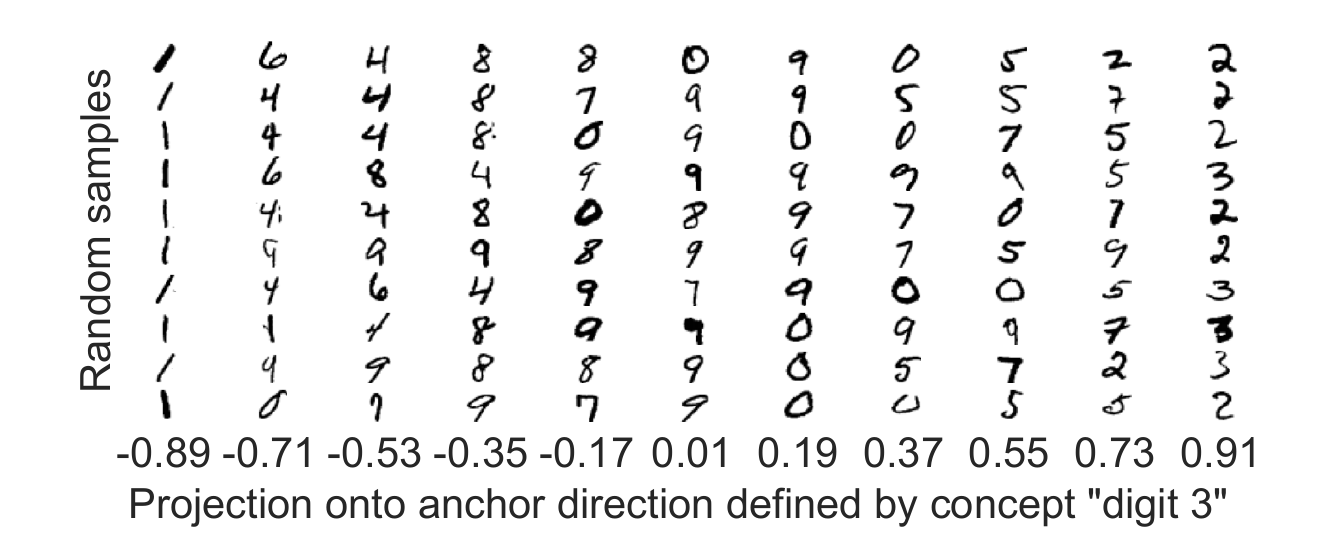} \\
~\includegraphics[width=.49\linewidth, trim={0pt 7pt 0pt 0pt}, clip]{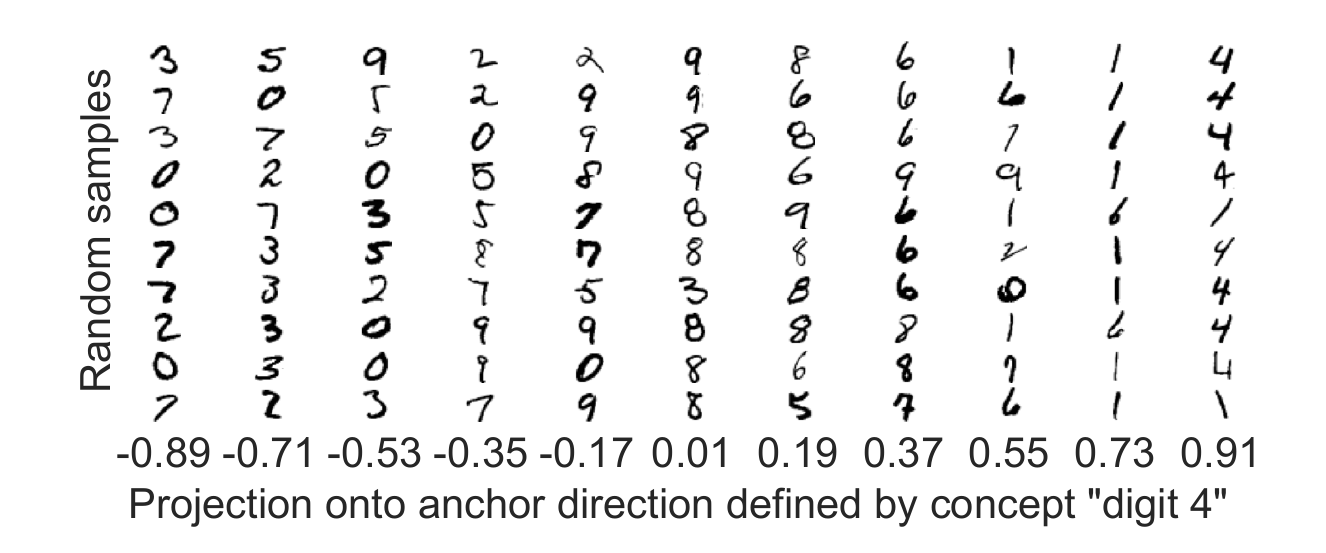}~\includegraphics[width=.49\linewidth, trim={0pt 7pt 0pt 0pt}, clip]{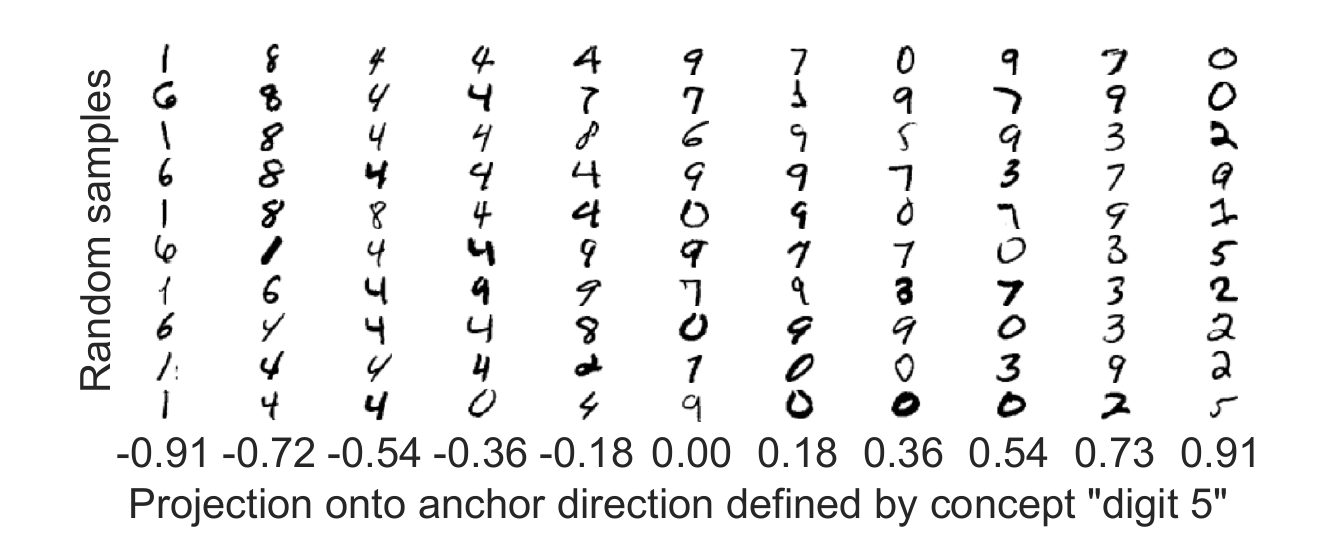} \\
~\includegraphics[width=.49\linewidth, trim={0pt 7pt 0pt 0pt}, clip]{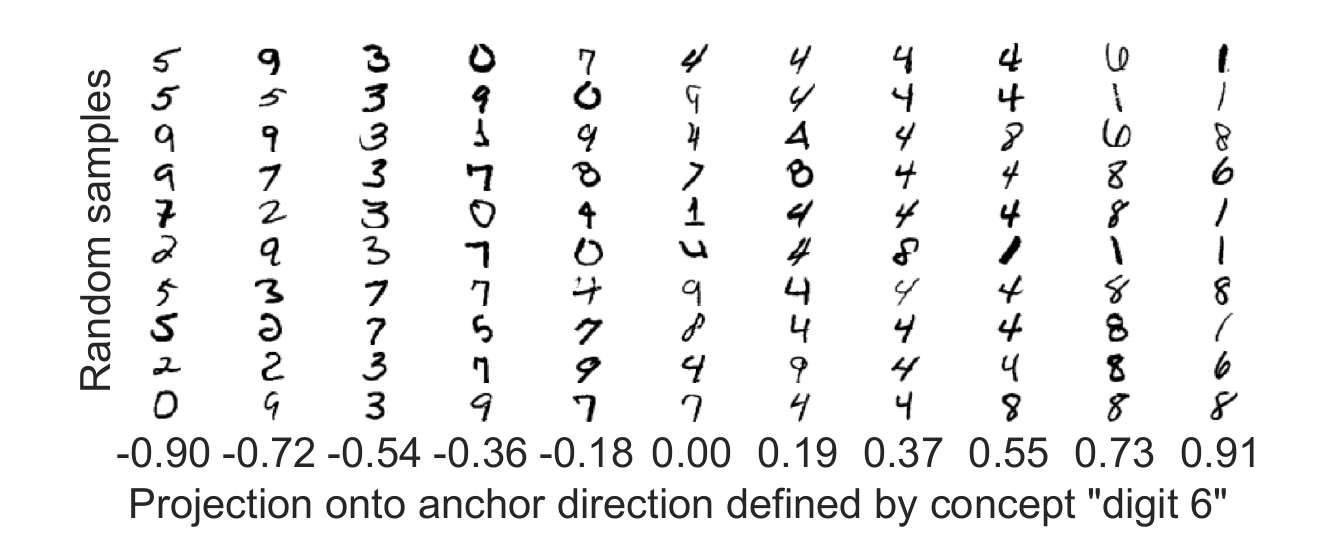}~\includegraphics[width=.49\linewidth, trim={0pt 7pt 0pt 0pt}, clip]{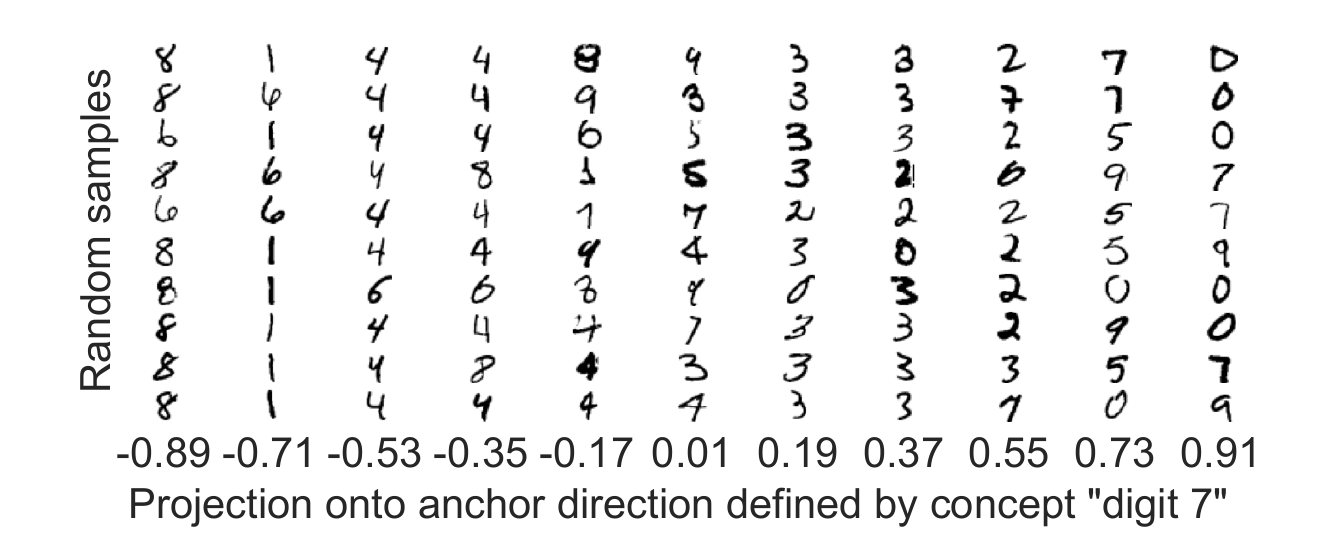} \\
~\includegraphics[width=.49\linewidth, trim={0pt 7pt 0pt 0pt}, clip]{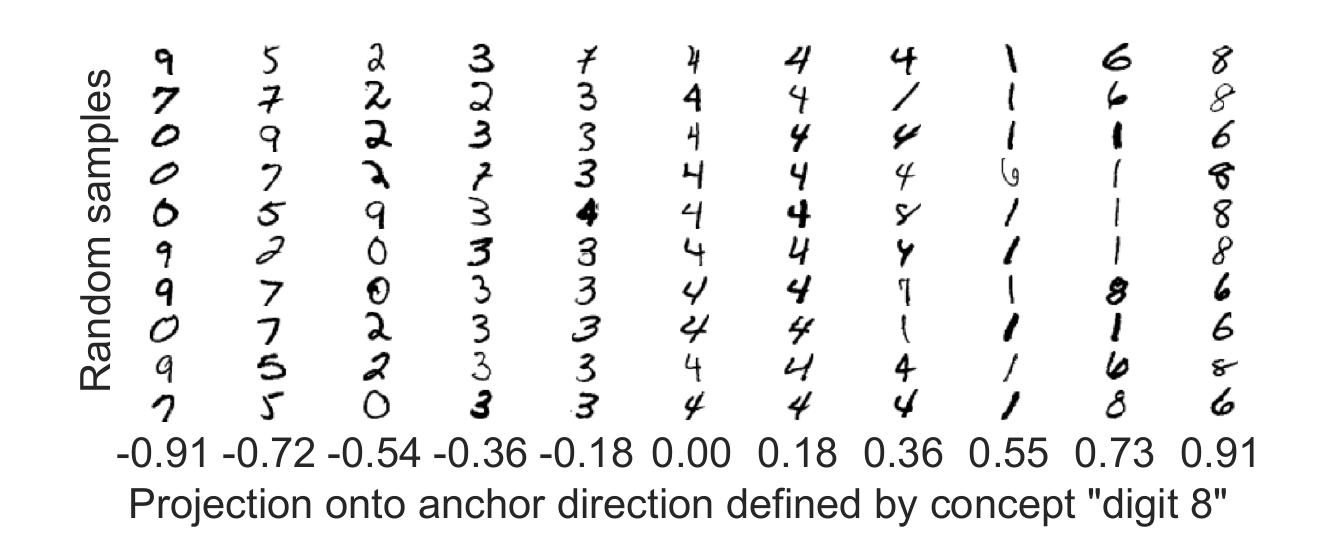}~\includegraphics[width=.49\linewidth, trim={0pt 7pt 0pt 0pt}, clip]{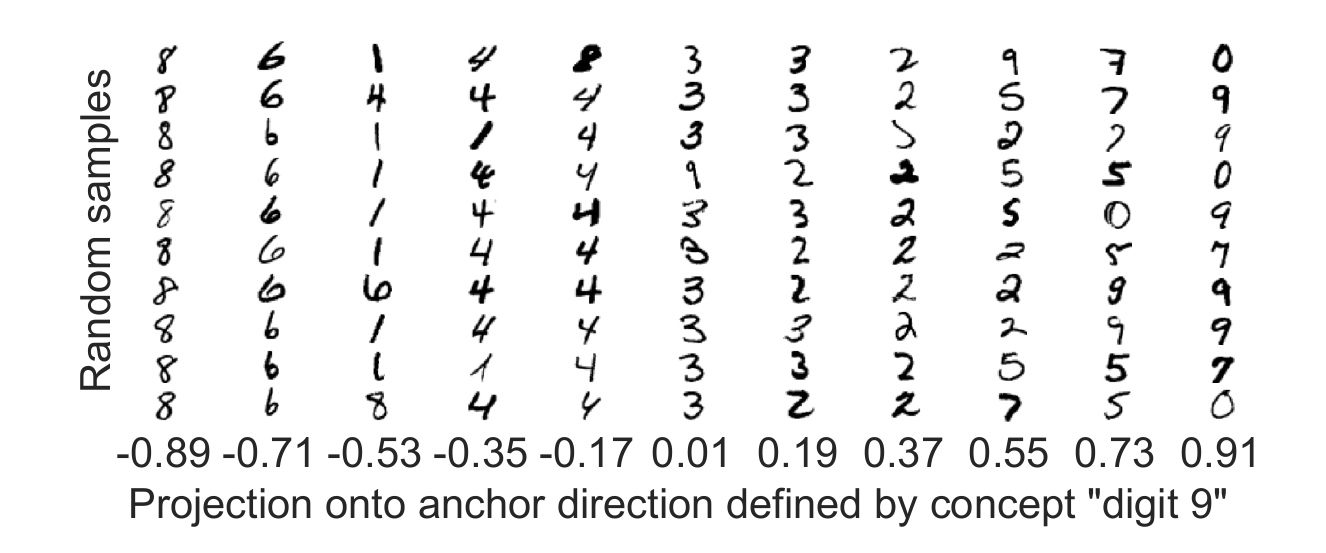}
\vspace{-.75em}
\caption{Survival MNIST: we treat each of the 10 digits as a concept that we compute an anchor direction for, and then we produce random input vs projection plots %
for the 10 anchor directions. For each plot, per projection bin, we sample 10 random visualization raw inputs.\label{fig:survival-mnist-random-samples-vs-projection}}
\vspace{-2.6em}
\end{figure*}

\begin{figure*}[t!]
\centering
\includegraphics[width=.97\linewidth, trim={0pt 7pt 0pt 0pt}, clip]{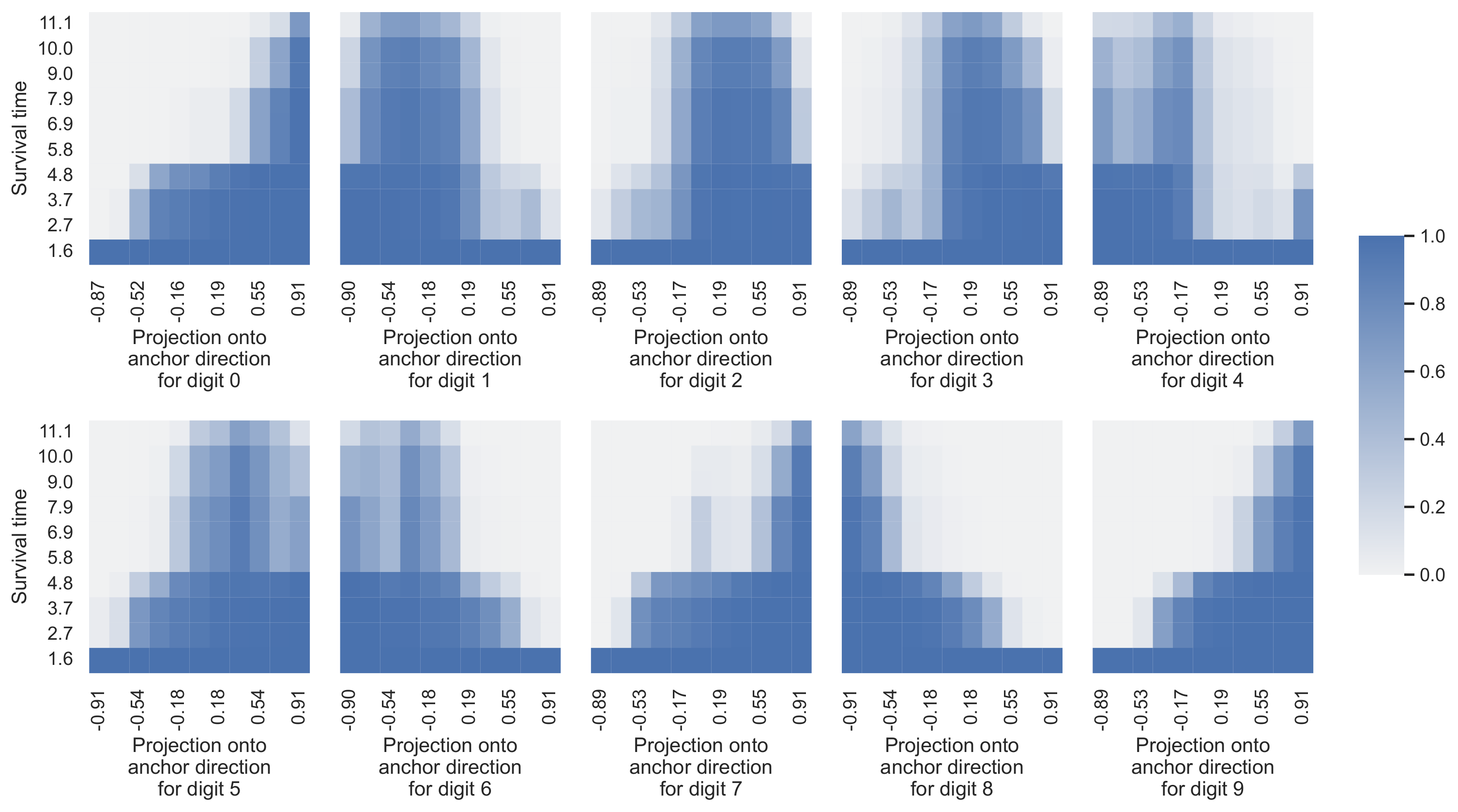}
\vspace{-.75em}
\caption{Survival MNIST: we treat each of the 10 digits as a concept that we compute an anchor direction for, and then we plot survival probability heatmaps for the 10 anchor directions.\label{fig:survival-mnist-digit-concepts-surv-prob-heatmaps}}
\vspace{-1.5em}
\end{figure*}

\subsubsection{Predicted Survival Functions}

We begin with a visualization that is not actually related to our anchor direction visualization framework and instead just looks at how well the trained DeepSurv model predicts survival functions. For the $i$-th visualization raw input $x_i^{\viz}$, we denote its true digit label as $\eta_i^{\viz}\in\{0,1,\dots,9\}$. Then for digit~$j$, we can compute the mean predicted survival function
\[
\widehat{S}_{\text{digit~\!}j}(t)
\triangleq
\frac{\sum_{i=1}^{n^{\viz}}\ind\{\eta_i^{\viz}=j\} \widehat{S}(t|x_i^{\viz})}
     {\sum_{i=1}^{n^{\viz}}\ind\{\eta_i^{\viz}=j\}}.
\]
We could also compute its standard deviation
\begin{align*}
&\widehat{S}_{\text{digit~\!}j}^{\text{std}}(t) \\
&\quad\triangleq
 \sqrt{
   \frac{\sum_{i=1}^{n^{\viz}}\ind\{\eta_i^{\viz}=j\} \big(\widehat{S}(t|x_i^{\viz}) - \widehat{S}_{\text{digit~\!}j}(t)\big)^2}{\sum_{i=1}^{n^{\viz}}\ind\{\eta_i^{\viz}=j\}}
}.
\end{align*}
We plot each mean predicted survival function $\widehat{S}_{\text{digit~\!}j}(t)$ with error bars given by $\widehat{S}_{\text{digit~\!}j}^{\text{std}}(t)$ in \figureref{fig:survival-mnist-predicted-survival-curves}. As  expected, these survival functions do not resemble the ground truth ones as the neural survival analysis model fitted assumes a proportional hazards model. However, the ranking of the digits is approximately correct: looking at the mean predicted survival functions, the ranking of these (going from lower to higher) is: 8,~6,~1,~4,~3,~2,~5,~7,~9,~0. The only error in this ranking is that digits 7 and 9 are swapped. As a reminder, digits~0, 7, and 9 are more difficult as they have censoring rates over 96\%.

Importantly, the survival functions $\widehat{S}_{\text{digit~\!}j}$ are estimated with the help of ground truth digit labels. The DeepSurv model in this case never received ground truth digit labels and, in particular, its embedding space (the output space of encoder $\phi$) was not explicitly trained to be able to distinguish between digits. That said, we can try to understand to what extent this embedding space captures information regarding the 10 digits. We proceed to do this next.

\subsubsection{Treating Each Digit as a Concept for Anchor Direction Estimation}
\label{sec:digit-as-concept}

We now show random input vs projection plots (like the one in \figureref{fig:survival-mnist-concept0-random-samples-vs-projection}) for all 10 digits in Figure~\ref{fig:survival-mnist-random-samples-vs-projection}. From these plots, we see that as the projection value gets large for digit $j\in\{0,1,\dots,9\}$, the raw inputs that achieve these large projection values for digit~$j$ tend to be of digit~$j$ itself or of other ``adjacent'' digit(s), where by ``adjacent'', we mean one(s) ordered next to digit~$j$ in terms of the ranking of ground truth mean survival times.

Using these same anchor directions, we produce the survival probability heatmaps (like the ones in \figureref{fig:support-anchor1-survival-heatmap} and \figureref{fig:support-anchor1through5-survival-heatmap}) in \figureref{fig:survival-mnist-digit-concepts-surv-prob-heatmaps}. Note that these heatmaps convey information similar to what is shown in \figureref{fig:survival-mnist-predicted-survival-curves}. For instance, when we look at the rightmost column of the survival probability heatmap for digit 0's anchor direction, we see that the survival function barely decays for all the observed times, indicative of the survival time tending to be large, as expected. This rightmost column's survival function resembles the predicted survival curve for digit~0 in \figureref{fig:survival-mnist-predicted-survival-curves}. A similar finding holds for the other digits.

\smallskip
\noindent
\textbf{The embedding space does not appear to capture the ground truth risk groups.}
We previously pointed out that the digits are grouped into four risk groups (each risk group has ground truth mean survival times that are very close by to each other; see \figureref{fig:survival-mnist-ground-truth-mean-survival-times}). It is not the case, however, that only the digits within the same risk group end up with high projection values for each other's anchor directions (see \figureref{fig:survival-mnist-concept0-random-samples-vs-projection}). We \emph{do} see this happen for the risk group with the lowest mean survival times (consisting of digits 1, 6, and 8) as well as the risk group with the highest mean survival times (consisting of 0, 7, and 9) but this does not entirely hold for the other risk groups.

To give a concrete example of how the embedding space does not correctly ``capture'' a risk group,
consider digit 5. %
The ground truth has digit~5 in its own risk group: no other digit's mean survival time is very close to that of digit~5. However, when we look at digit 5's random input vs projection plot in \figureref{fig:survival-mnist-random-samples-vs-projection}, when we look at the rightmost two projection bins, we see that many digits (that are not the digit~5) have high projection values for digit~5. %

We suspect that what is causing the problem is censoring. We said that the digits in risk group $\{1,6,8\}$ tend to have high projection values for each other's anchor directions, and similarly for the digits in the risk group $\{0,7,9\}$. Note that all digits in risk group $\{1,6,8\}$ have censoring rates below 6\%. All digits in risk group $\{0,7,9\}$ have censoring rates over 96\%. In contrast, the digit~5 has a censoring rate of about 71\%, which is neither very low nor very high. The digits with high projection values for digit~5 are ones that all have censoring rates over 35\%. %
Basically, although the embedding space does not appear to be capturing risk groups well, it seems to recognize what censoring means. We examine this next.

\subsubsection{Treating ``Censored'' as a Concept for Anchor Direction Estimation}

\begin{figure}[t!]
~\includegraphics[width=.98\linewidth, trim={0pt 7pt 0pt 0pt}, clip]{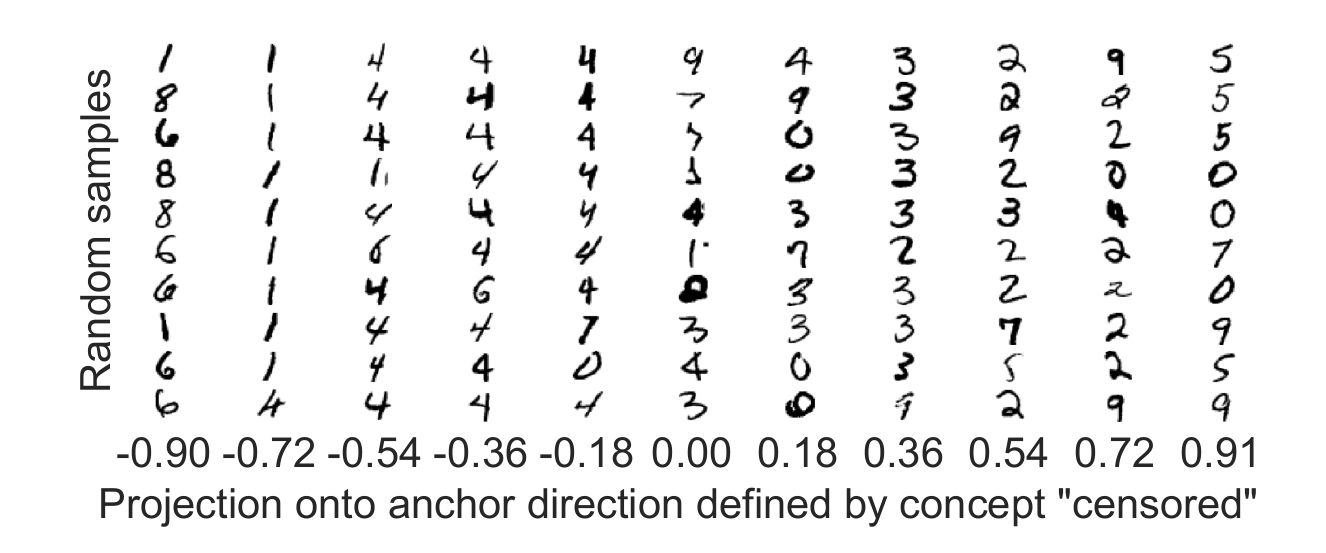}
\vspace{-2.2em}
\caption{Survival MNIST: after computing an anchor direction for the concept ``censored'', we produce a random input vs projection plot for this anchor direction.\label{fig:survival-mnist-censored-random-samples-vs-projection}}\vspace{-2em}
\end{figure}

We treat anchor direction estimation data that are censored (i.e., their event indicator variables are equal to 0) as a concept, which we then compute the anchor direction for using equation~\eqref{eq:concept-vs-all}. We produce a random input vs projection plot for this ``censored'' concept's anchor direction in \figureref{fig:survival-mnist-censored-random-samples-vs-projection}. From the plot, as we progress from the most negative projection values to the most positive, the random samples clearly correspond to digits going from the lowest to the highest ground truth mean survival times, which also corresponds to going from the lowest to the highest censoring rates.

\subsubsection{Estimating Anchor Directions via Clustering}

\begin{figure}[t!]
\centering
\includegraphics[width=.825\linewidth]{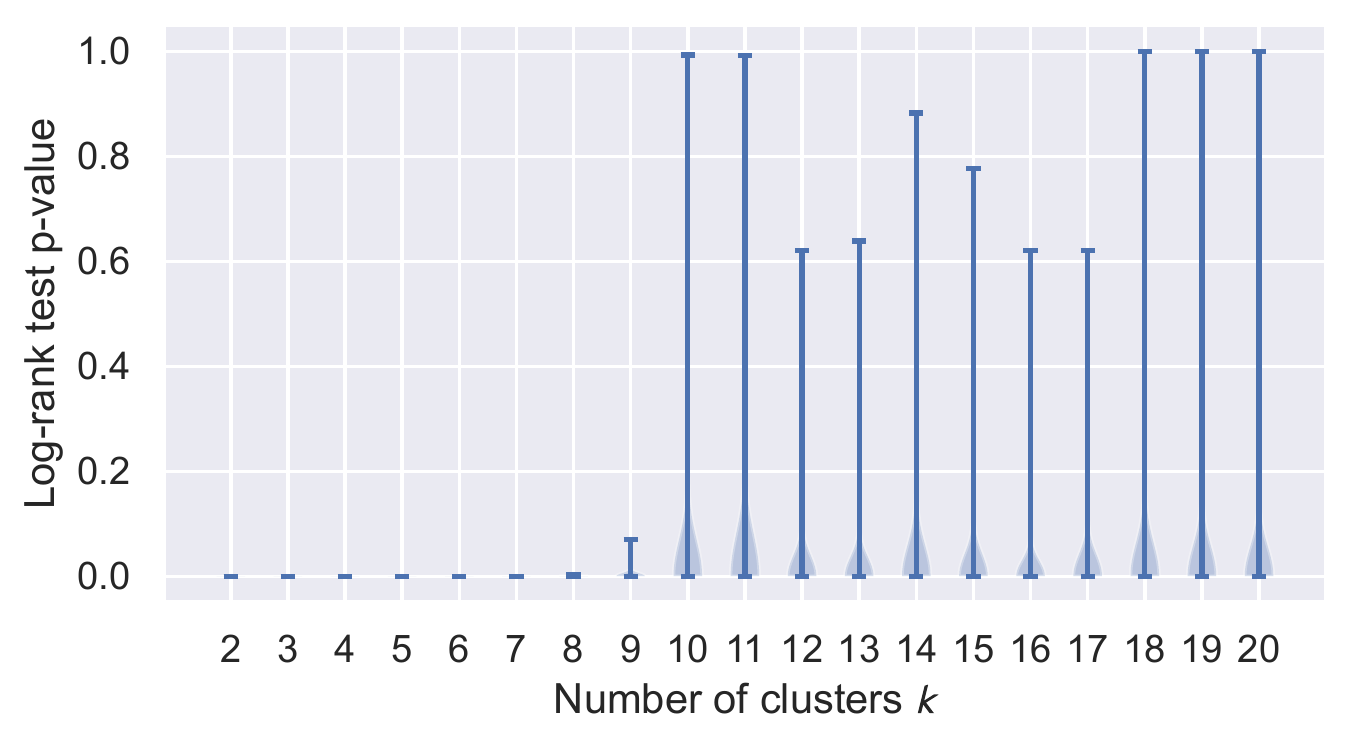}
\vspace{-1.5em}
\caption{Survival MNIST: a violin plot to help select the number of clusters for use with a mixture of von Mises-Fisher distributions (which determines the number of anchor directions to use).}
\label{fig:survival-mnist-logrank-helper}
\vspace{-2em}
\end{figure}

\begin{figure}[t!]
\centering
\includegraphics[width=\linewidth]{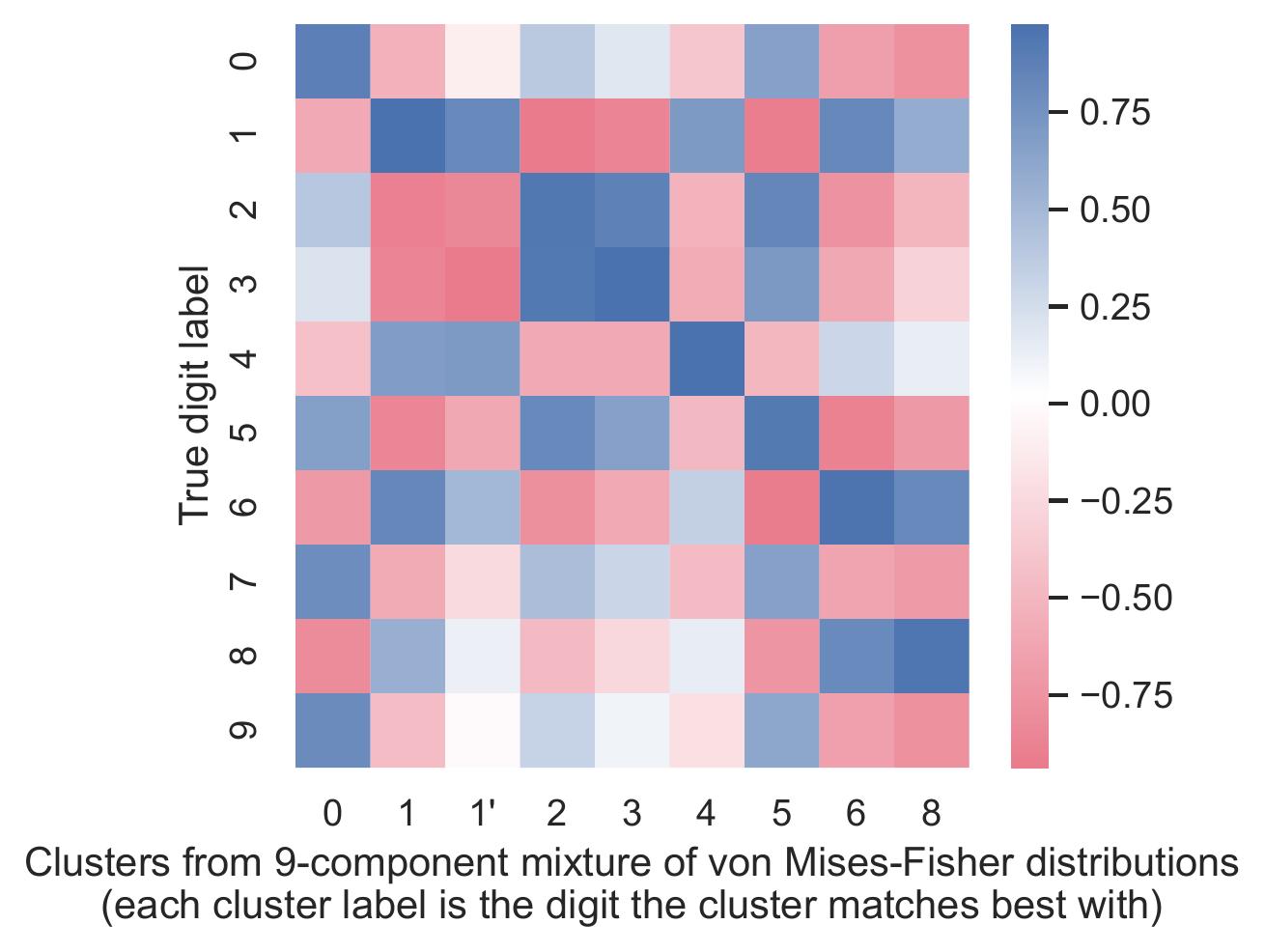}
\vspace{-2.5em}
\caption{Survival MNIST: average projection heatmap for seeing how well a clustering assignment with 9 clusters (using a mixture of von Mises-Fisher distributions) aligns with the 10 ground truth digit labels. The intensity at the $i$-th row and \mbox{$j$-th} column corresponds to average projection value along the \mbox{$j$-th} cluster's anchor direction across visualization data with ground truth digit label~$i$. Note that the clusters with labels 1 and 1' both match best with digit 1.}
\label{fig:survival-mnist-nclusters9-cluster-vs-digit-average-projection-heatmap}
\vspace{-2em}
\end{figure}

Lastly, we consider estimating anchor directions via clustering, where we use a mixture of von Mises-Fisher distributions as the clustering model. We show the violin plot for selecting the number of clusters in \figureref{fig:survival-mnist-logrank-helper}. From this violin plot, we see that the log-rank test p-values have a sharp increase after 9 clusters. We examine the clustering results using 9 clusters and, separately, also using 4 clusters and 10 clusters. The reason we look at the 4 clusters case is because there are 4 underlying risk groups, and we can check to what extent these can be recovered from a clustering solution with 4 clusters. As for looking at a model with 10 clusters, this is because we know that in reality there are 10 digits, each with its own ground truth survival function.

\begin{figure*}[p!]
\centering
~\includegraphics[width=.49\linewidth, trim={0pt 7pt 0pt 0pt}, clip]{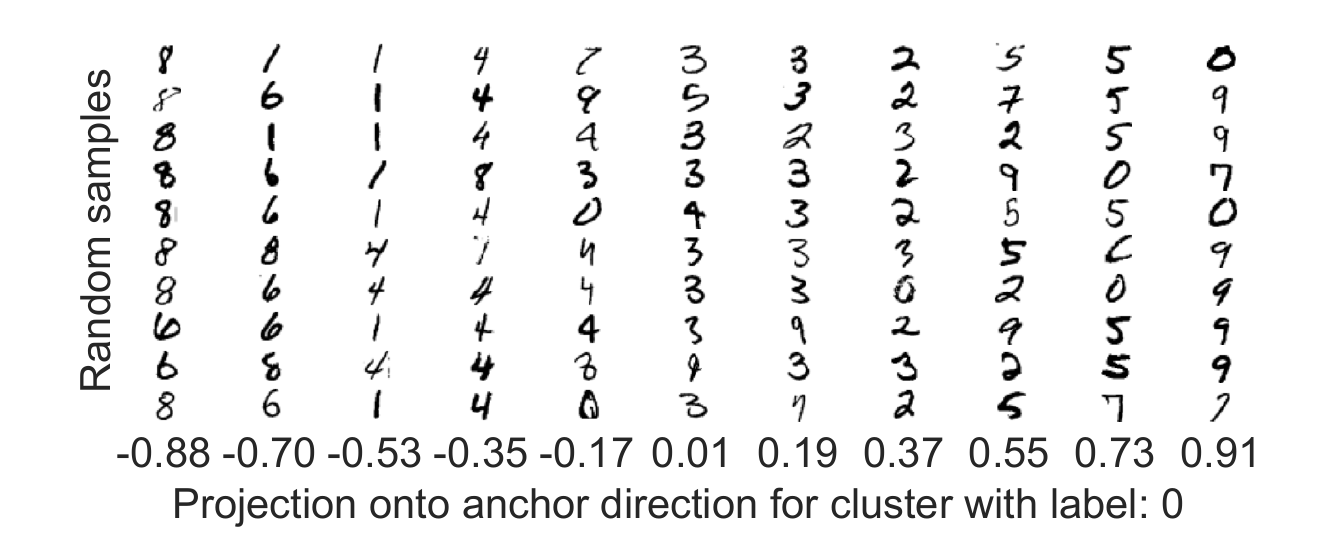}~\includegraphics[width=.49\linewidth, trim={0pt 7pt 0pt 0pt}, clip]{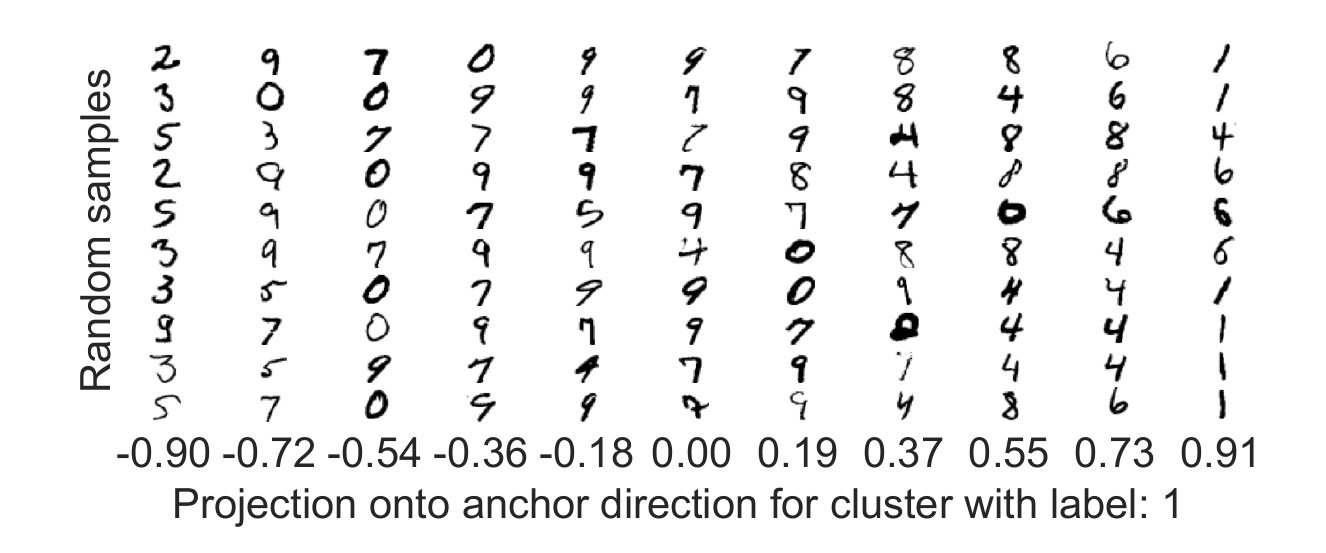} \\
~\includegraphics[width=.49\linewidth, trim={0pt 7pt 0pt 0pt}, clip]{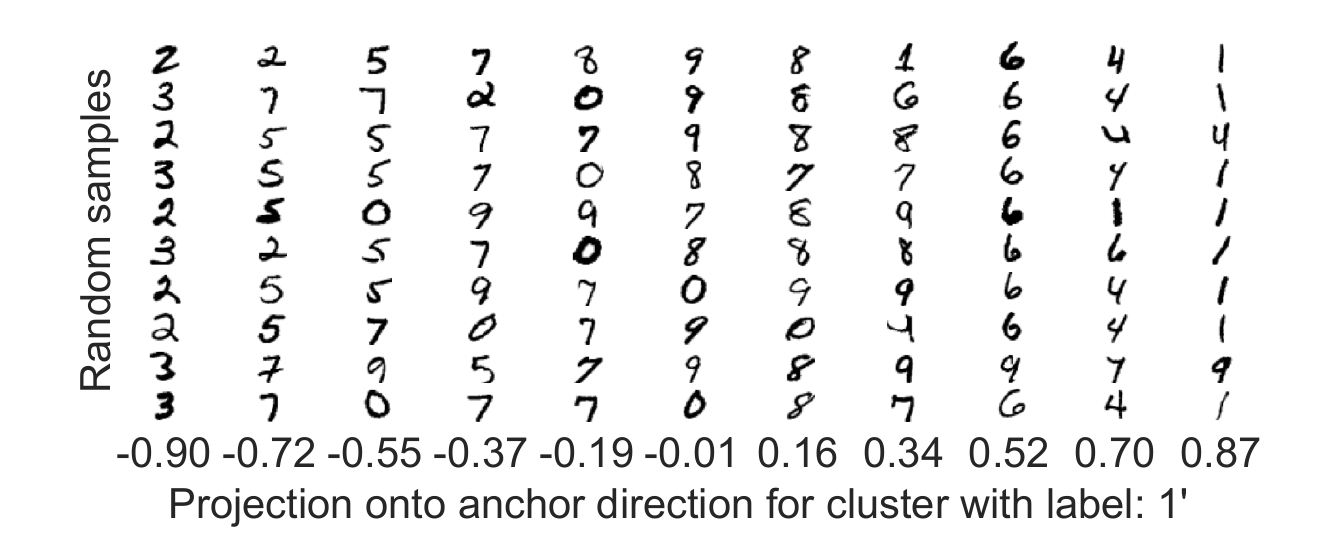}~\includegraphics[width=.49\linewidth, trim={0pt 7pt 0pt 0pt}, clip]{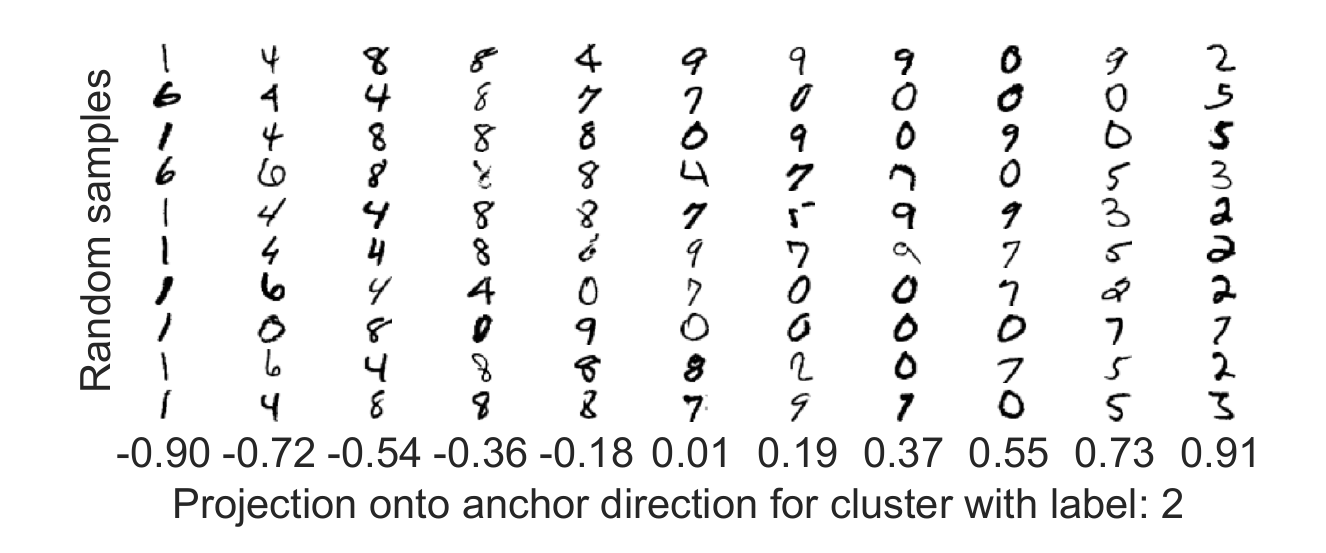} \\
~\includegraphics[width=.49\linewidth, trim={0pt 7pt 0pt 0pt}, clip]{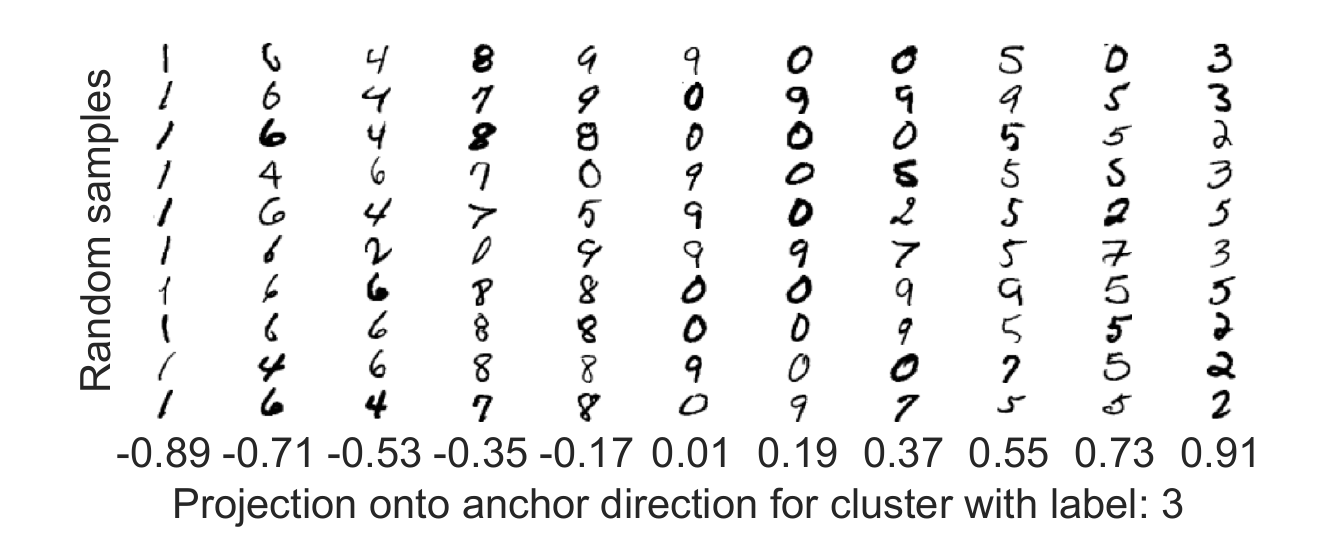}~\includegraphics[width=.49\linewidth, trim={0pt 7pt 0pt 0pt}, clip]{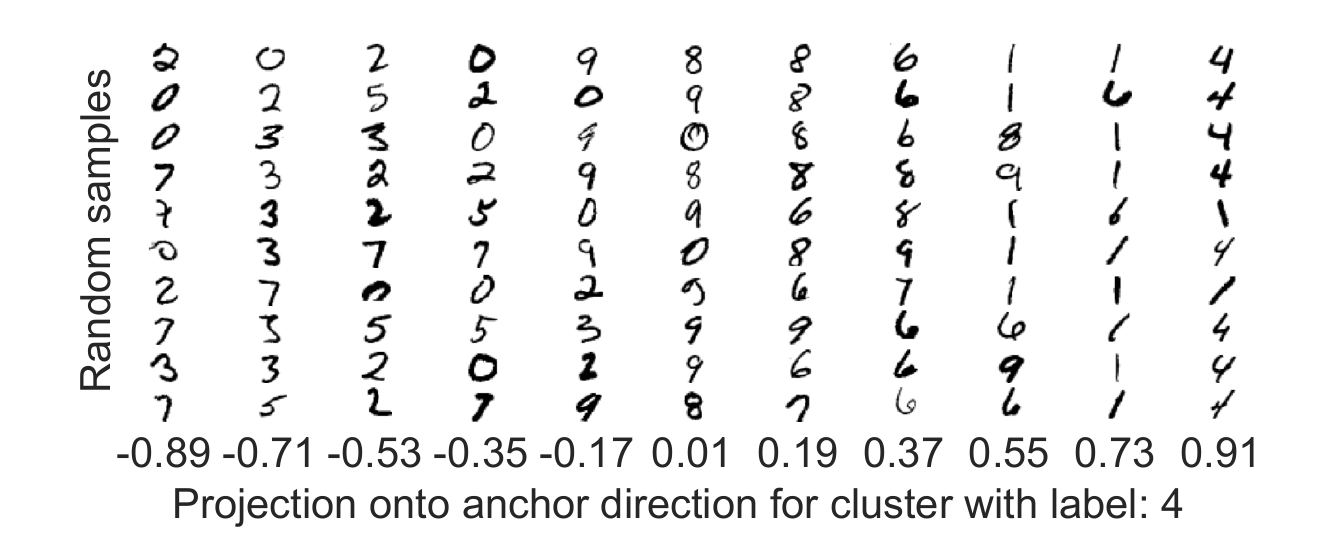} \\
~\includegraphics[width=.49\linewidth, trim={0pt 7pt 0pt 0pt}, clip]{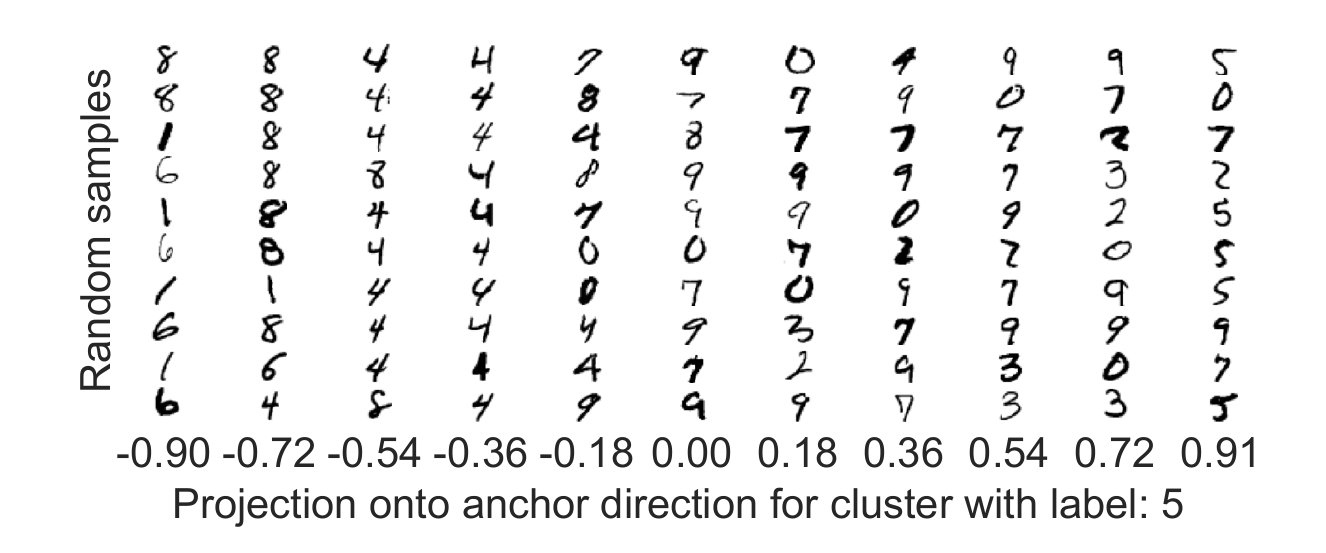}~\includegraphics[width=.49\linewidth, trim={0pt 7pt 0pt 0pt}, clip]{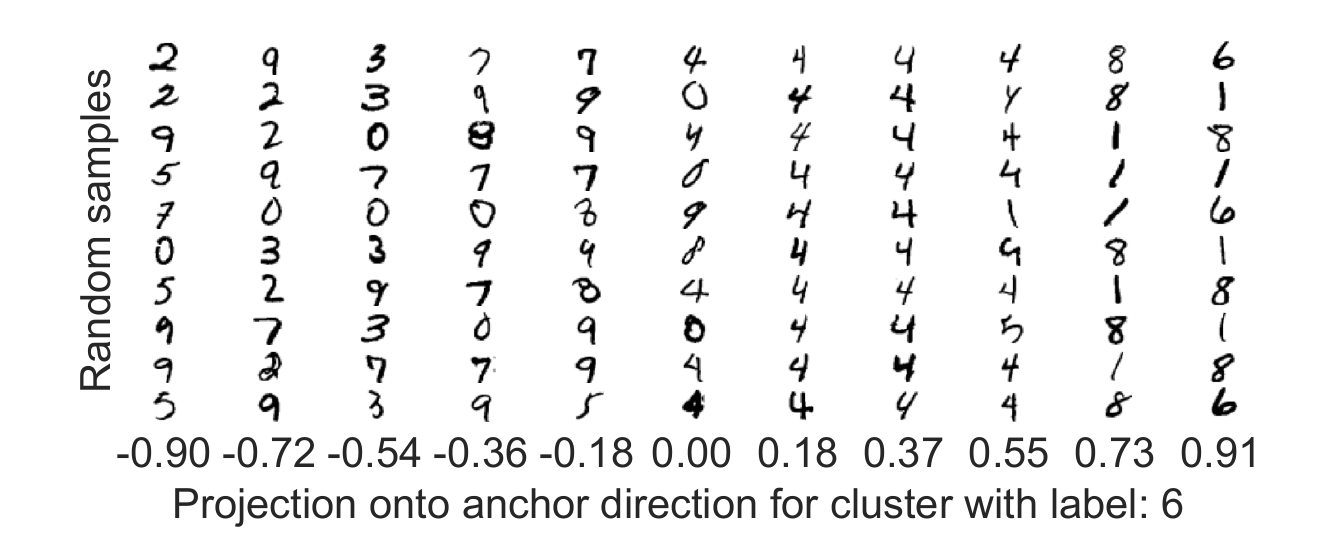} \\
~\includegraphics[width=.49\linewidth, trim={0pt 7pt 0pt 0pt}, clip]{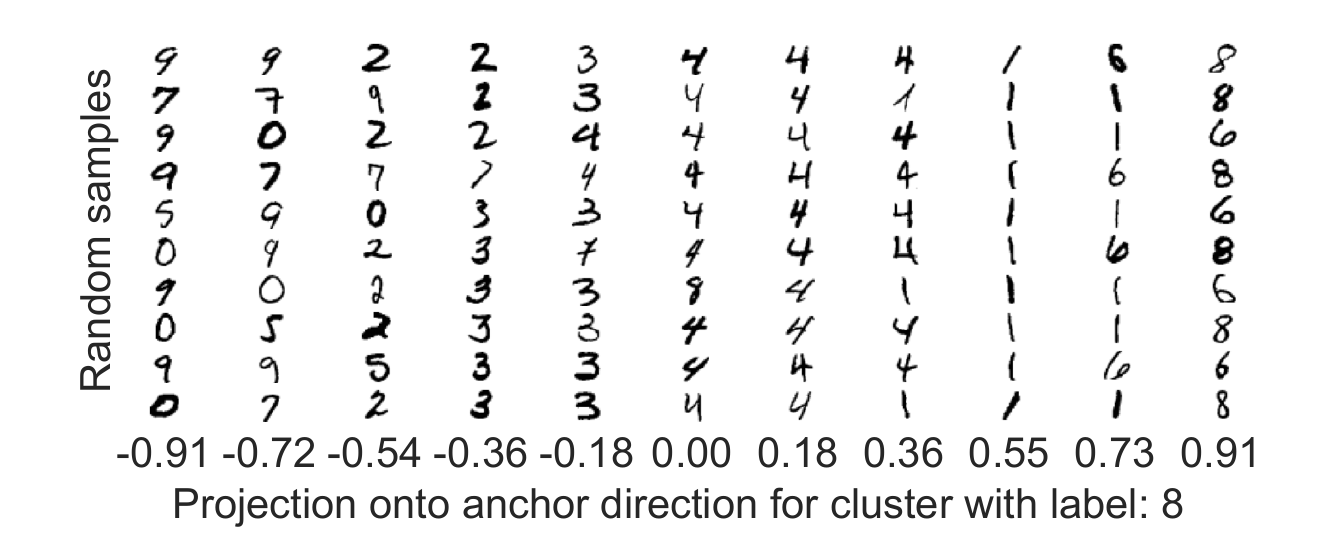}
\vspace{-.75em}
\caption{Survival MNIST: using anchor directions estimated from a clustering model with 9 clusters, we produce random input vs projection plots for the resulting 9 anchor directions. The cluster labels are the same as the ones along the x-axis of \figureref{fig:survival-mnist-nclusters9-cluster-vs-digit-average-projection-heatmap}.\label{fig:survival-mnist-nclusters9-random-samples-vs-projection}}
\vspace{-2.6em}
\end{figure*}

\begin{figure*}[p!]
\centering
\includegraphics[width=.6\linewidth]{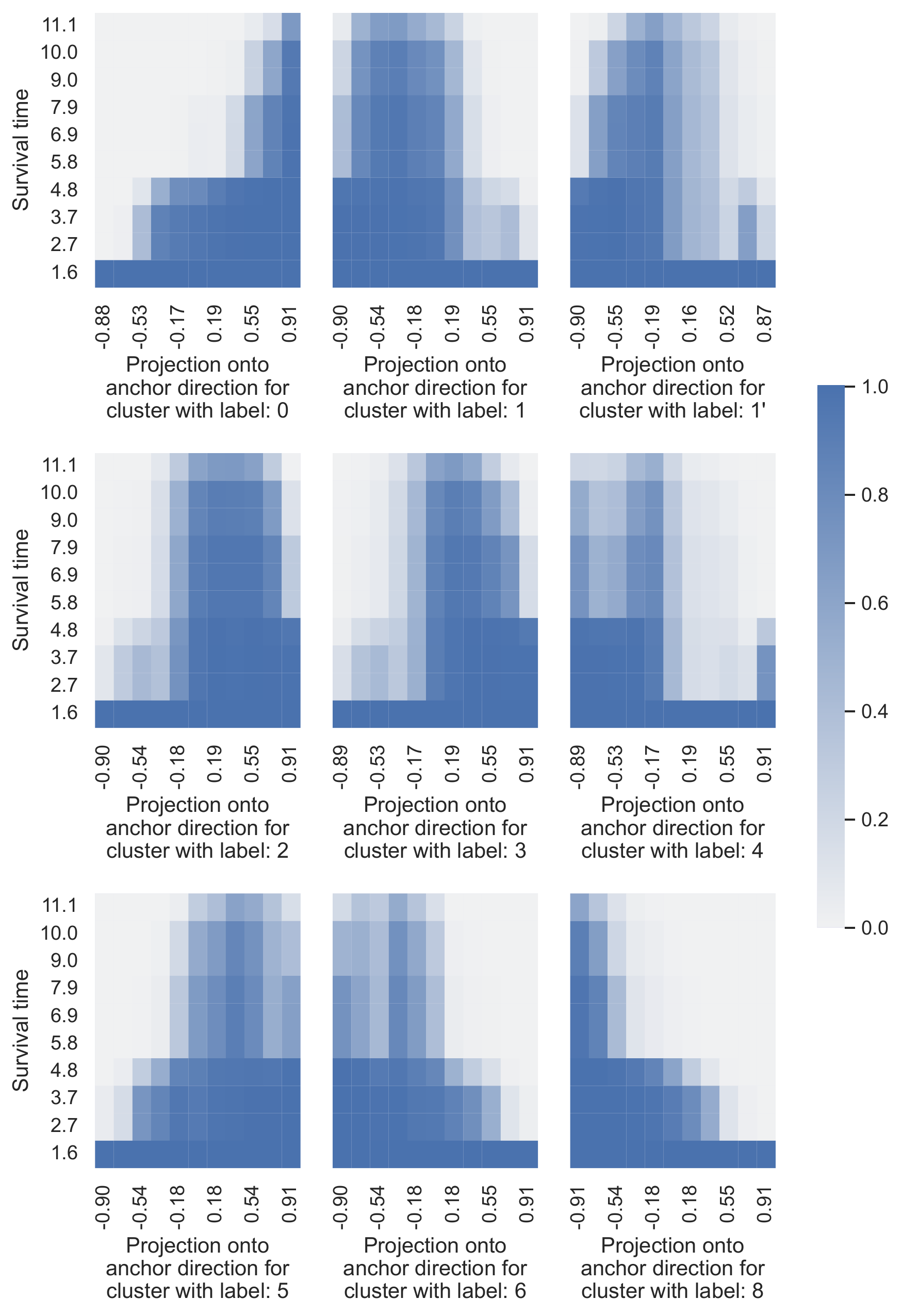}
\vspace{-1em}
\caption{Survival MNIST: using anchor directions estimated from a clustering model with 9 clusters, we produce survival probability heatmaps for the resulting 9 anchor directions. The cluster labels are the same as the ones along the x-axis of \figureref{fig:survival-mnist-nclusters9-cluster-vs-digit-average-projection-heatmap}.\label{fig:survival-mnist-nclusters9-surv-prob-heatmaps}}
\vspace{-2.6em}
\end{figure*}

\begin{figure}[p!]
\centering
\includegraphics[width=\linewidth]{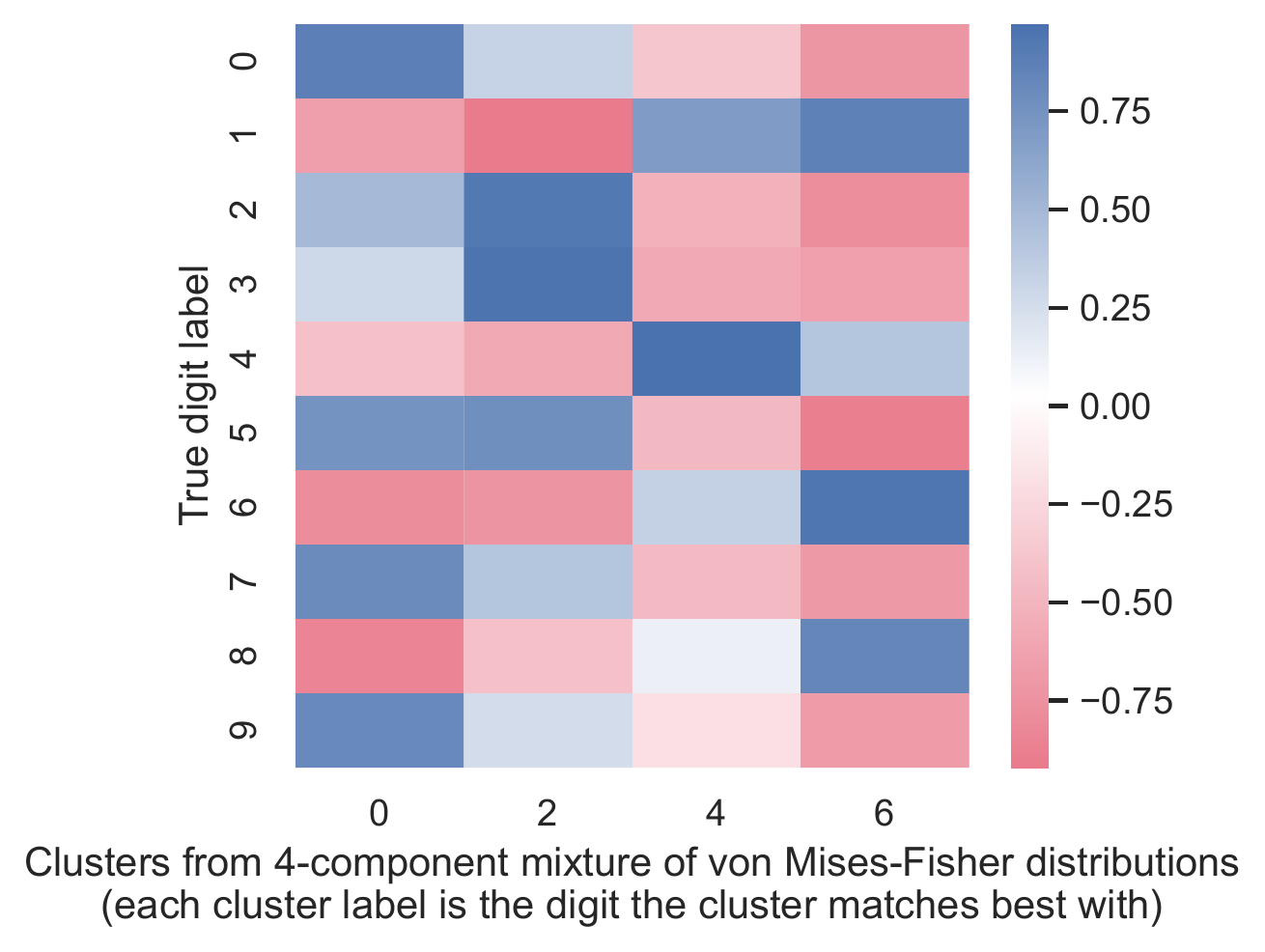}
\vspace{-2.5em}
\caption{Survival MNIST: average projection heatmap for a clustering assignment with 4 clusters (using a mixture of von Mises-Fisher distributions). The intensity at the $i$-th row and \mbox{$j$-th} column corresponds to average projection value along the \mbox{$j$-th} cluster's anchor direction across visualization data with ground truth digit label~$i$. Just as in \figureref{fig:survival-mnist-nclusters9-cluster-vs-digit-average-projection-heatmap}, we label each cluster based on the single digit that it best matches to.} %
\label{fig:survival-mnist-nclusters4-cluster-vs-digit-average-projection-heatmap}
\vspace{-1em}
\end{figure}

\begin{figure}[p!]
\centering
\includegraphics[width=\linewidth]{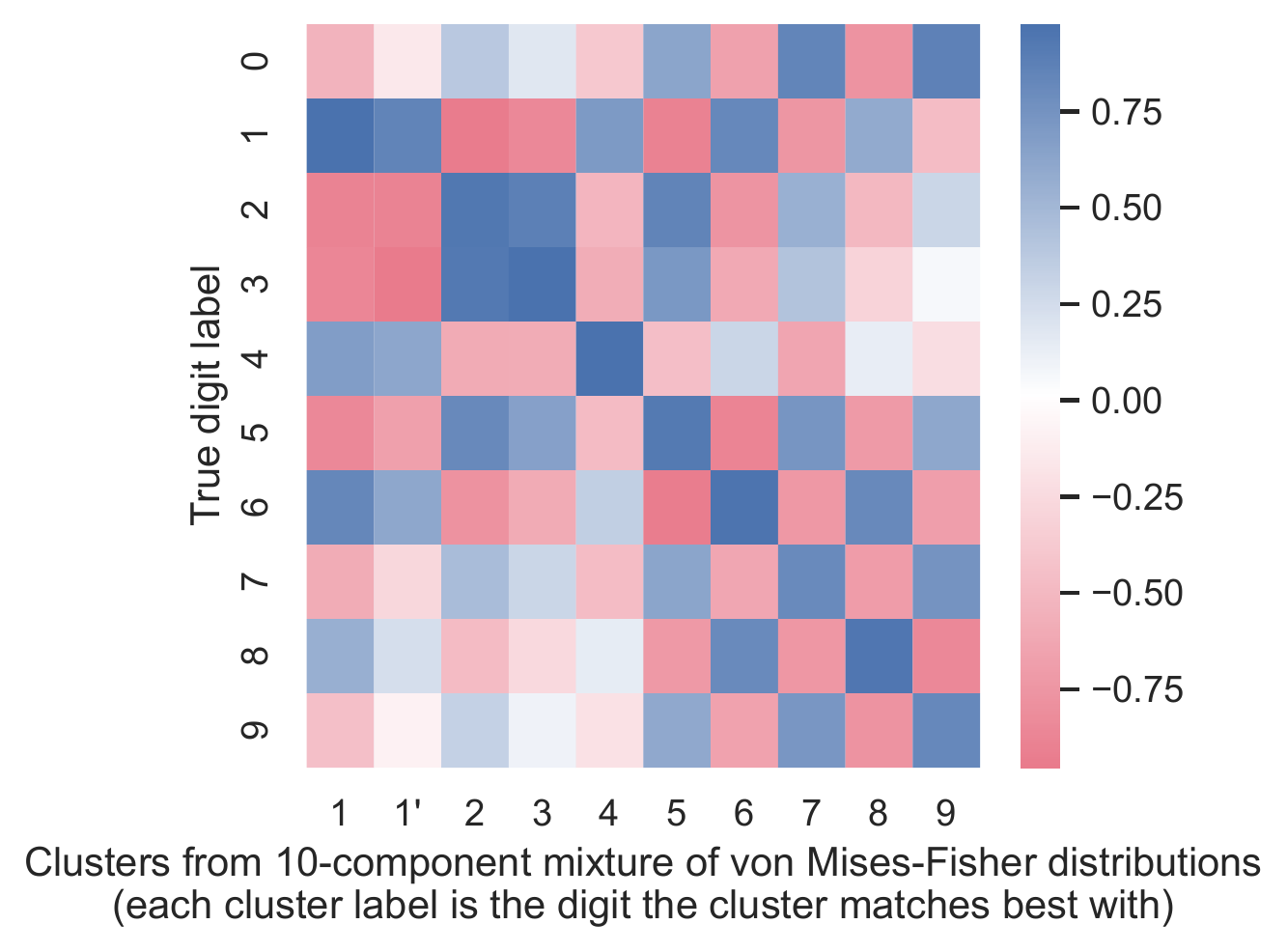}
\vspace{-2.5em}
\caption{Survival MNIST: average projection heatmap for a clustering assignment with 10 clusters (using a mixture of von Mises-Fisher distributions). The intensity at the $i$-th row and \mbox{$j$-th} column corresponds to average projection value along the \mbox{$j$-th} cluster's anchor direction across visualization data with ground truth digit label~$i$. Just as in \figureref{fig:survival-mnist-nclusters9-cluster-vs-digit-average-projection-heatmap}, we label each cluster based on the single digit that it best matches to.} %
\label{fig:survival-mnist-nclusters10-cluster-vs-digit-average-projection-heatmap}
\vspace{-1em}
\end{figure}

\smallskip
\noindent
\textbf{Results for a clustering model with 9 clusters.}
We first use anchor directions from a 9-component mixture of von Mises-Fisher distributions, where each component is treated as a cluster. We check how well the 9 clusters' anchor directions align with the anchor directions of the digit concepts. %
For visualization purposes, we label each cluster with the digit that the cluster matches best to, where we determine a match as follows: for the \mbox{$j$-th} cluster, we find whichever digit's anchor direction (computed using equation~\eqref{eq:concept-vs-all}; these were the anchor directions used in Appendix~\ref{sec:digit-as-concept}) is most similar to the \mbox{$j$-th} cluster's anchor direction (computed using equation~\eqref{eq:cluster-vs-all}) according to cosine similarity. It is possible that multiple clusters match best with the same digit. Then, we can create what we call an \emph{average projection heatmap} where the entry at the $i$-th row and \mbox{$j$-th} column corresponds to the average projection value along the \mbox{$j$-th} cluster's anchor direction across visualization data that have the ground truth digit label~$i$. We show the resulting heatmap in \figureref{fig:survival-mnist-nclusters9-cluster-vs-digit-average-projection-heatmap}. From this heatmap, we see that the cluster with label 0 has high projection values for digits 0, 5, 7, and 9, which are the most censored digits (digits 0, 7, and 9 in particular tend to have the highest projection values for cluster 0); digit 5, however, also has its own cluster that it is matched to whereas digits 7 and 9 do not. Another observation is that digits 1, 6, and 8 tend to have high projection values for clusters with labels 1, 6, and 8 (although the cluster with label 1 has highest projection values for digit 1, and similarly for the clusters with labels 6 and 8). %

We create a random input vs projection plot for each of these 9 clusters' anchor directions in \figureref{fig:survival-mnist-nclusters9-random-samples-vs-projection}, where we see the same sort of phenomenon we had pointed out in Appendix~\ref{sec:digit-as-concept}: when we look at an anchor direction roughly corresponding to digit~$j$, then images of digit~$j$ as well as images of digits that have a true mean survival time adjacent to that of digit~$j$ will often have high projection values. We also plot survival probability heatmaps for these anchor directions in \figureref{fig:survival-mnist-nclusters9-surv-prob-heatmaps}. Each cluster's survival probability heatmap is similar to the one for the digit that the cluster best matches to (see \figureref{fig:survival-mnist-digit-concepts-surv-prob-heatmaps} for comparison).

\smallskip
\noindent
\textbf{Results for a clustering model with 4 clusters.}
We next use anchor directions from a 4-component mixture of von Mises-Fisher distributions, again treating each component as a cluster. We show an average projection heatmap in \figureref{fig:survival-mnist-nclusters4-cluster-vs-digit-average-projection-heatmap}. Note that although we use the same way of labeling each cluster as we did when we used 9 clusters, here we would actually like each cluster to correspond to the ground truth risk groups (so that each cluster does not necessarily only correspond to a single digit). %

From \figureref{fig:survival-mnist-nclusters4-cluster-vs-digit-average-projection-heatmap}, we see that the cluster with label~0 is most like the ground truth risk group $\{0,7,9\}$ but also includes digit~5 (which also has a relatively high censoring rate among the different digits). Meanwhile, the cluster with label 6 is most like the risk group $\{1,6,8\}$. The clusters with labels 2 and 4 together correspond to the ground truth risk group $\{2,3,4\}$. Overall, the ground truth risk groups are not correctly recovered although the clusters found, qualitatively, pick up on some of the ground truth structure. We suspect that the difficulty in determining that digit 5 should be in its own cluster has to do with how often it is censored (over 70\%). The risk group corresponding to digits 1, 6, and 8 should be the easiest to recover as these three digits have the lowest censoring rates (all below 6\%).

We omit random input vs projection plots and survival probability heatmaps for the different clusters' anchor directions since the findings from these visualizations are qualitatively similar to the findings we just pointed out from looking at Figures~\ref{fig:survival-mnist-nclusters4-cluster-vs-digit-average-projection-heatmap} and~\ref{fig:survival-mnist-nclusters9-surv-prob-heatmaps} (note that for each cluster in this 4-cluster model, the digits that the cluster matches well with tend to have very similar survival probability heat maps).

\smallskip
\noindent
\textbf{Results for a clustering model with 10 clusters.}
When we use 10 clusters for the mixture of von Mises-Fisher distributions, we obtain the average projection heatmap in \figureref{fig:survival-mnist-nclusters10-cluster-vs-digit-average-projection-heatmap}. In this case, when we match each cluster to a digit, the only digit that does not get matched is digit~0, although digit 0 itself has high projection values for the clusters with labels~7 and~9. %
Qualitatively, the clustering result here is not too different from the one where we used 9 clusters. The main difference now is that we can somewhat distinguish better between the digits in the risk group $\{0,7,9\}$ consisting of digits with the highest censoring rates. We omit the random input vs projection plots and the survival probability heatmaps for the 10 different clusters' anchor directions as these visualizations do not provide additional insight at this point over the other findings we have already reported.

\abovesectionskip
\section{Theoretical Result on Projection Values When Information Content in Embedding Vectors is All in Magnitudes}
\label{sec:proposition-proof}
\belowsectionskip

\begin{proposition}
\label{prop:toy-example}
(\emph{Extreme example where the embedding space information is all in magnitudes})
Suppose that the embedding vectors (of anchor direction estimation and visualization data) are i.i.d.~of the form $(\Xi,0,\dots,0)\in\mathbb{R}^d$ (i.e., all coordinates are 0 except the first), where $\Xi$ is a continuous random variable with positive variance. In this setup, the only direction in the embedding space that matters is along the vector $\mu=(1,0,\dots,0)\in\mathbb{R}^d$, which we can take to be the anchor direction of interest. Then the only possible projection values $p_i^{\viz}$ are $-1$ or $1$; projection values in the open interval $(-1, 1)$ are not possible.
\end{proposition}

\begin{proof}
Let $\mu=(1,0,\dots,0)$; we treat this as the anchor direction as it is the only direction in which the embedding vectors even vary in this proposition's setup.
Our visualizations involve plugging in the visualization raw inputs $x_1^{\viz},\dots,x_{n^{\viz}}^{\viz}$ into~$\phi$. We denote $u_i^{\viz}\triangleq\phi(x_i^{\viz})$, so that the projection value~$p_i^{\viz}$ defined in equation~\eqref{eq:projection} is equal to
\begin{equation}
p_i^{\viz} = \text{proj}_\mu(x_i^{\viz}) %
= \Big\langle \frac{u_i^{\viz} - \overline{u}^{\anchor}}{\|u_i^{\viz} - \overline{u}^{\anchor}\|}, \mu \Big\rangle,
\label{eq:proof-helper}
\end{equation}
where we have used the fact that $\|\mu\|=1$.

The key observation is that we can write each $u_i^{\viz}$ as $u_i^{\viz}=(\Xi_i^{\viz},0,\dots,0)$ and similarly each $u_i^{\anchor}$ as $u_i^{\anchor}=(\Xi_i^{\anchor},0,\dots,0)$ where the $\{\Xi_i^{\viz}\}_{i=1}^{n^{\viz}}$ and $\{\Xi_i^{\anchor}\}_{i=1}^{n^{\anchor}}$ are all i.i.d.~continuous random variables with positive variance. Therefore,
\begin{align*}
u_i^{\viz} - \overline{u}^{\anchor}
&= u_i^{\viz} - \frac{1}{n^{\anchor}}\sum_{i=1}^{n^{\anchor}} u_i^{\anchor} \\
&= \Big( \underbrace{\Xi_i^{\viz} - \frac{1}{n^{\anchor}}\sum_{i=1}^{n^{\anchor}} \Xi_i^{\anchor}}_{\spadesuit}, 0, \dots, 0 \Big),
\end{align*}
where $\spadesuit$ is a sum of independent continuous random variables with positive variance, so $\spadesuit$ itself is still a continuous random variable with positive variance (note that the variance of the sum of two independent variables is the sum of their variances). This implies that $\spadesuit$ is 0 with probability 0, which in turn implies that with probability 1, $\|u_i^{\viz} - \overline{u}^{\anchor}\|$ is nonzero, so $\frac{u_i^{\viz} - \overline{u}^{\anchor}}{\|u_i^{\viz} - \overline{u}^{\anchor}\|}$ is well-defined and, in particular, it is either $\mu$ or $-\mu$. Then by using equation~\eqref{eq:proof-helper}, $p_i^{\viz}$ is either equal to $\langle\mu,\mu\rangle=1$ or $\langle-\mu,\mu\rangle=-1$.%
\end{proof}

\abovesectionskip
\section{Handling a Large Number of Clusters/Anchor Directions With the Help of Ranking}
\label{sec:handling-many-anchor-directions}
\belowsectionskip

When using our heuristic from Section~\ref{sec:anchor-directions-via-clustering} for choosing the number of clusters to use, it is possible that the number of clusters could be very large --- so large that examining visualizations for anchor directions corresponding to all the clusters would be too tedious for model debugging purposes. Of course, one could simply choose to not set the number of clusters to be so large. Put another way, when using our violin plot visualization to help select the number of clusters, one could simply choose a smaller p-value threshold, which would result in fewer clusters being used. However, if for whatever reason, one wants to use a number of clusters that is larger with the goal of having clusters that are more ``fine-grain'', we now suggest an approach for handling this situation as to reduce the amount of clusters to look at. Note that this approach actually applies more generally to the setting where there are many anchor directions that are under consideration, where the anchor directions need not be estimated from our clustering approach. For example, the anchor directions could be computed based on concepts as in Section~\ref{sec:anchor-directions-via-concepts}, where there is a very large number of concepts under consideration.

The basic idea is to rank the anchor directions. If we have a ranking of the anchor directions, then we could focus on, for instance, a few of the highest and a few of the lowest ranked anchor directions. Alternatively, we could also, for instance, take anchor directions that are ``diverse'' across ranks: for example, we could choose the 0th percentile-ranked anchor direction (lowest ranked), the 25th percentile, the 50th percentile (median), the 75th percentile, the 100th percentile (highest ranked). In this manner, we can focus on visualizing only a subset of all the anchor directions. We describe one approach to rank anchor directions next.

\smallskip
\noindent
\textbf{Ranking anchor directions based on predicted median survival times.}
One heuristic approach is to compute a median survival time estimate for each anchor direction, and then rank anchor directions based on this median survival time estimate.

Per anchor direction $\mu$, we first determine the visualization data that are in the top $\alpha$ fraction of the projection values along~$\mu$ (e.g., if $\alpha=0.1$, then this means that we consider data points with projection values that are within the top 10\%). Formally, this set of visualization data points can be written as follows. First, recall that the visualization data have projection values $p_i^{\viz} = \text{proj}_\mu(x_i^{\viz})$, for $i=1,\dots,n^{\viz}$. Suppose that we sort these projection values and denote the sorted projection values as $p_{(1)}^{\viz}<p_{(2)}^{\viz}<\cdots<p_{(n^{\viz})}^{\viz}$. Then the top $\alpha$ projection value can be estimated by
\[
q_{\alpha} \triangleq p_{(\lceil (1-\alpha) n^{\viz}\rceil)}^{\viz}.
\]
Then the visualization data with projection values in the top $\alpha$ percentile of projection values along~$\mu$ are the ones in the set
\[
\mathcal{I}_{\mu}^{\text{~\!top~\!}\alpha}
\triangleq\{ i \in \{1,2,\dots,n^{\viz}\}\text{ s.t.~}p_i^{\viz} \ge q_{\alpha}\}.
\]
Note that this equation is similar to that of equation~\eqref{eq:projection-bin-points}. We then compute the survival curve for the data points in $\mathcal{I}_{\mu}^{\text{~\!top~\!}\alpha}$ using an equation analogous to equation~\eqref{eq:survival-curve-bin}:
\[
\widehat{S}_{\mu}^{\text{~\!top~\!}\alpha}(t)
\triangleq \frac{1}{|\mathcal{I}_{\mu}^{\text{~\!top~\!}\alpha}|}\sum_{i\in\mathcal{I}_{\mu}^{\text{~\!top~\!}\alpha}} \widehat{S}(t|x_i^{\viz}).
\]
By a standard result in survival analysis, the time~$t$ where the survival curve $\widehat{S}_{\mu}^{\text{~\!top~\!}\alpha}(t)$ crosses 1/2 corresponds to a median survival time estimate (see, for instance, \citealt{reid1981estimating}). In particular, we denote this median survival time estimate as
\[
\widehat{\text{med}}_{\mu}^{\text{top~\!}\alpha}
\triangleq \inf\{t\ge0 \text{ s.t.~} \widehat{S}_{\mu}^{\text{~\!top~\!}\alpha}(t) \le 1/2 \}.
\]
In practice, to compute the infimum, commonly a discrete time grid is used, such as using all the unique observed times in the training data (i.e., the unique $Y_i$ values), and if the survival curve never crosses 1/2 over this discrete time grid, then for simplicity we just take the median survival time estimate to be a special value specifying that it is greater than the maximum observed time in the training data.

Note that what we stated above is for any anchor direction~$\mu$. Thus, if we have $k$ anchor directions denoted as $\mu_1,\dots,\mu_k$, then we can rank these anchor directions by $\widehat{\text{med}}_{\mu_1}^{\text{top~\!}\alpha},\dots,\widehat{\text{med}}_{\mu_k}^{\text{top~\!}\alpha}$.

\smallskip
\noindent
\textbf{SUPPORT dataset example.} Consider the data and setup from Section~\ref{sec:tabular}, which is the same setting as the additional results in Appendix~\ref{sec:support-additional-results}. By using the above approach for ranking clusters/anchor directions based on estimated median survival times and setting $\alpha=0.1$, we get the following ranking of the five clusters (in ascending order of estimated median survival times):
\begin{enumerate}[leftmargin=*,topsep=1pt,parsep=0pt,itemsep=0pt]
\item Cluster 1: median survival time estimate 46~days
\item Cluster 5: median survival time estimate 105~days
\item Cluster 2: median survival time estimate 236~days
\item Cluster 3: median survival time estimate 452~days
\item Cluster 4: median survival time estimate 1895~days
\end{enumerate}

\abovesectionskip
\section{Baseline Visualization Strategy: Use Dimensionality Reduction to Plot the Embedding Space}
\label{sec:baseline-visualization}
\belowsectionskip

\begin{figure*}[p!]
\floatconts
  {fig:baseline-visualizations}
  {\caption{2D PCA and t-SNE plots of the visualization data's embedding vectors from a DeepSurv model (with a norm~1 constraint) for the $(a)$ SUPPORT, $(b)$ Rotterdam/GBSG, and $(c)$ Survival MNIST datasets. The colors indicate estimated median survival times.}}
  {
    \subfigure[]{\label{fig:baseline-visualizations-support}
      \includegraphics[scale=.5]{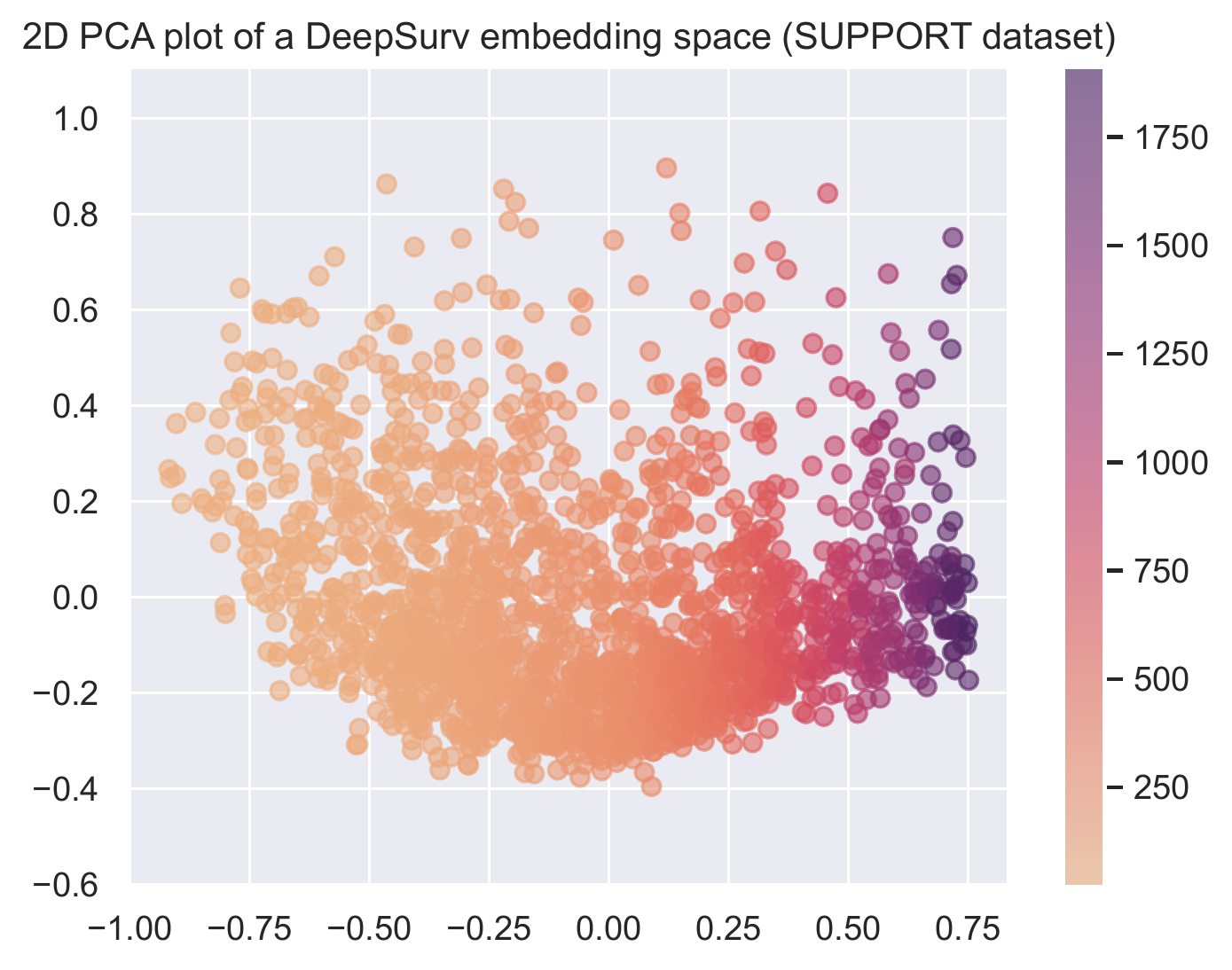}~
      \includegraphics[scale=.5]{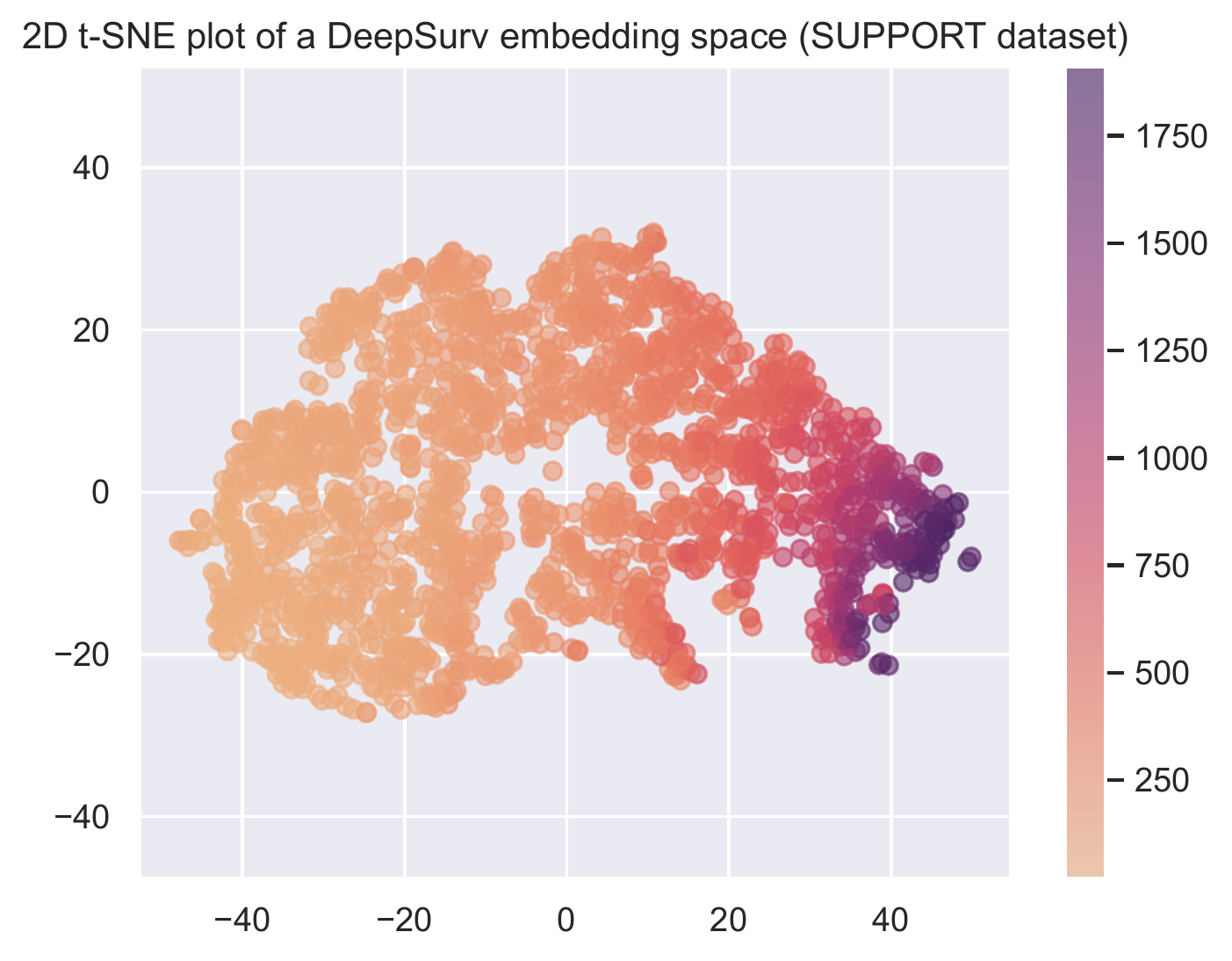}
    \vspace{-1.5em}
    }
    \vspace{.5em} \\
    \subfigure[]{\label{fig:baseline-visualizations-rotterdam-gbsg}
      \includegraphics[scale=.5]{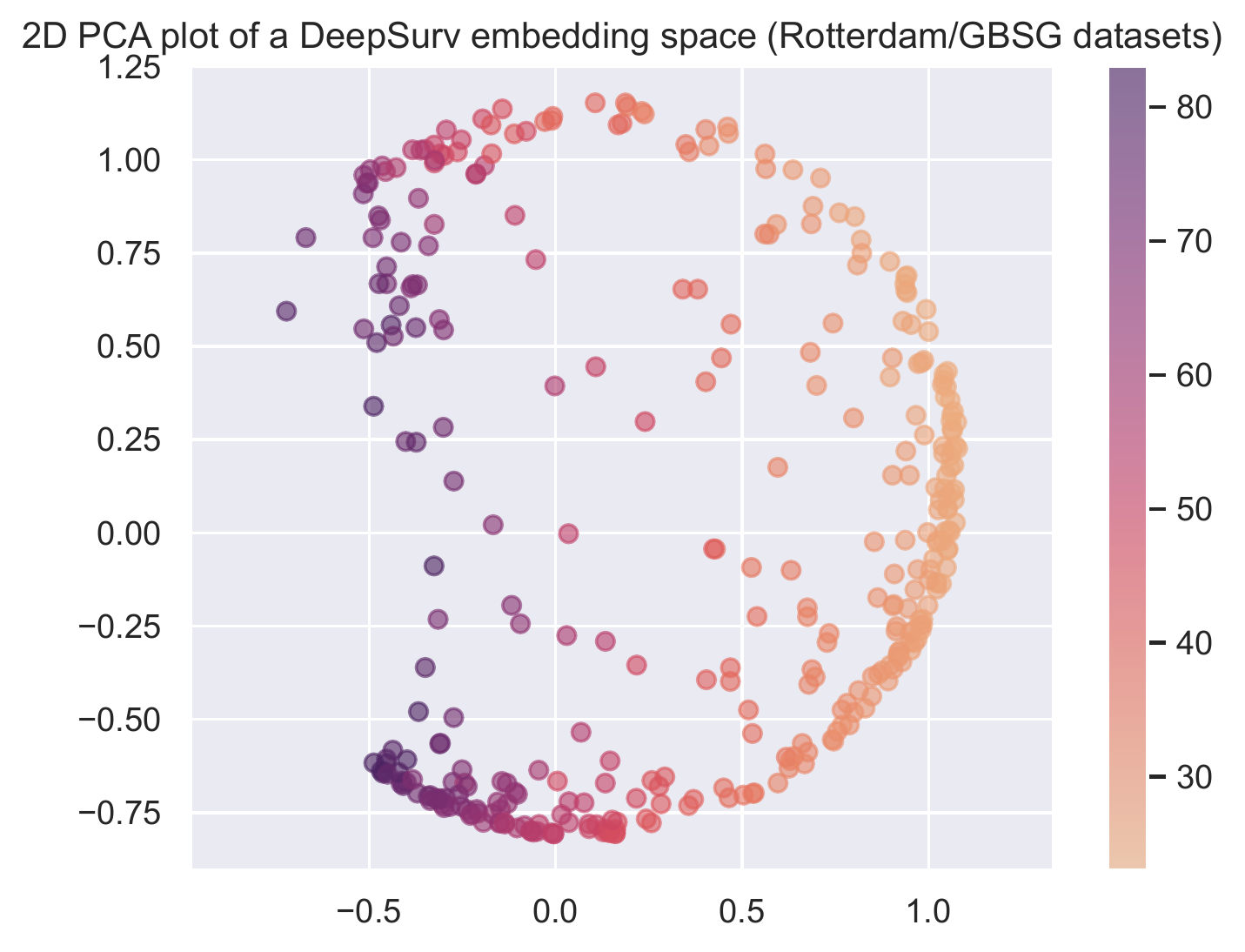}~
      \includegraphics[scale=.5]{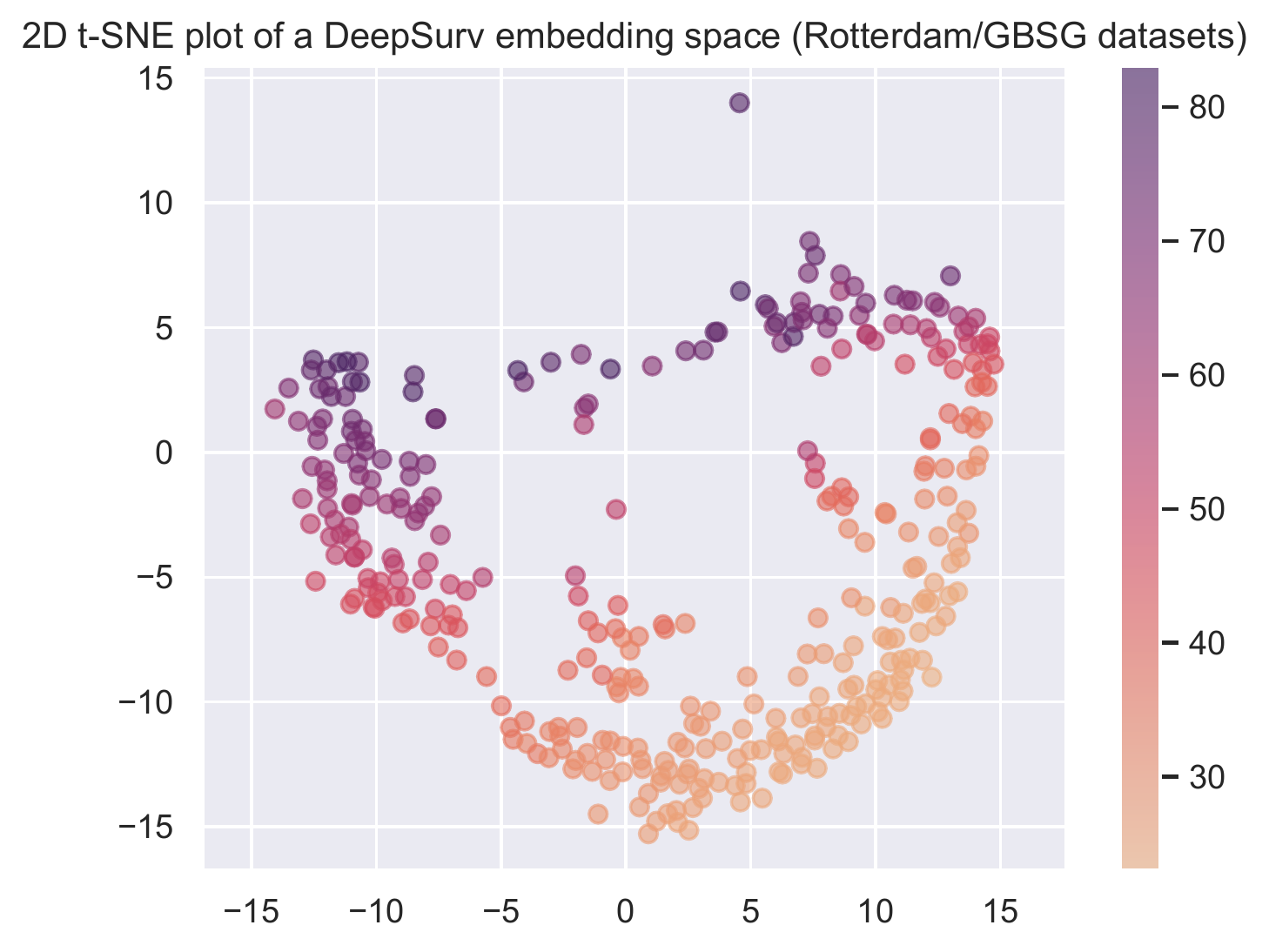}
    \vspace{-1.5em}
    }
    \vspace{.5em} \\
    \subfigure[]{\label{fig:baseline-visualizations-survival-mnist}
      \includegraphics[scale=.5]{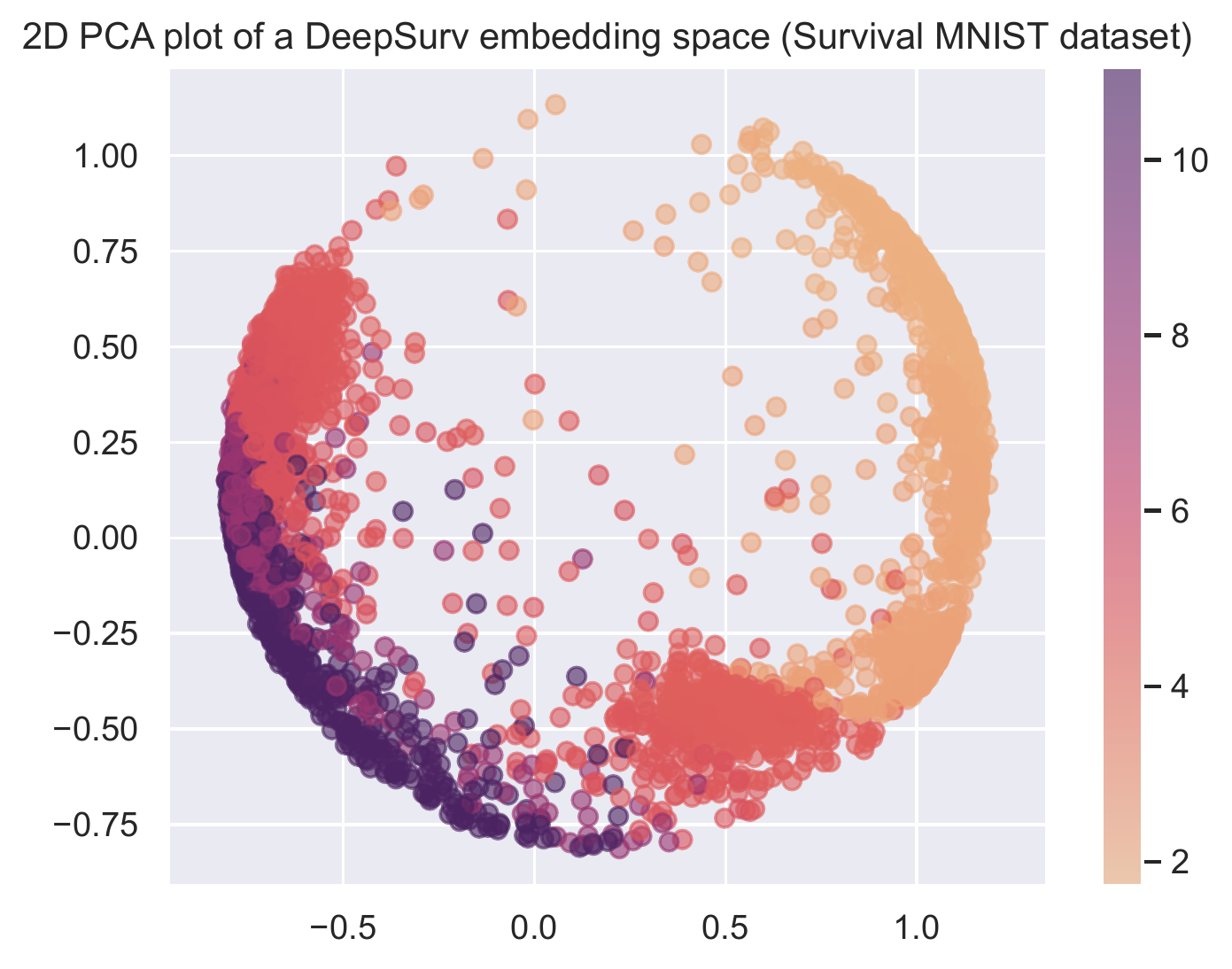}~
      \includegraphics[scale=.5]{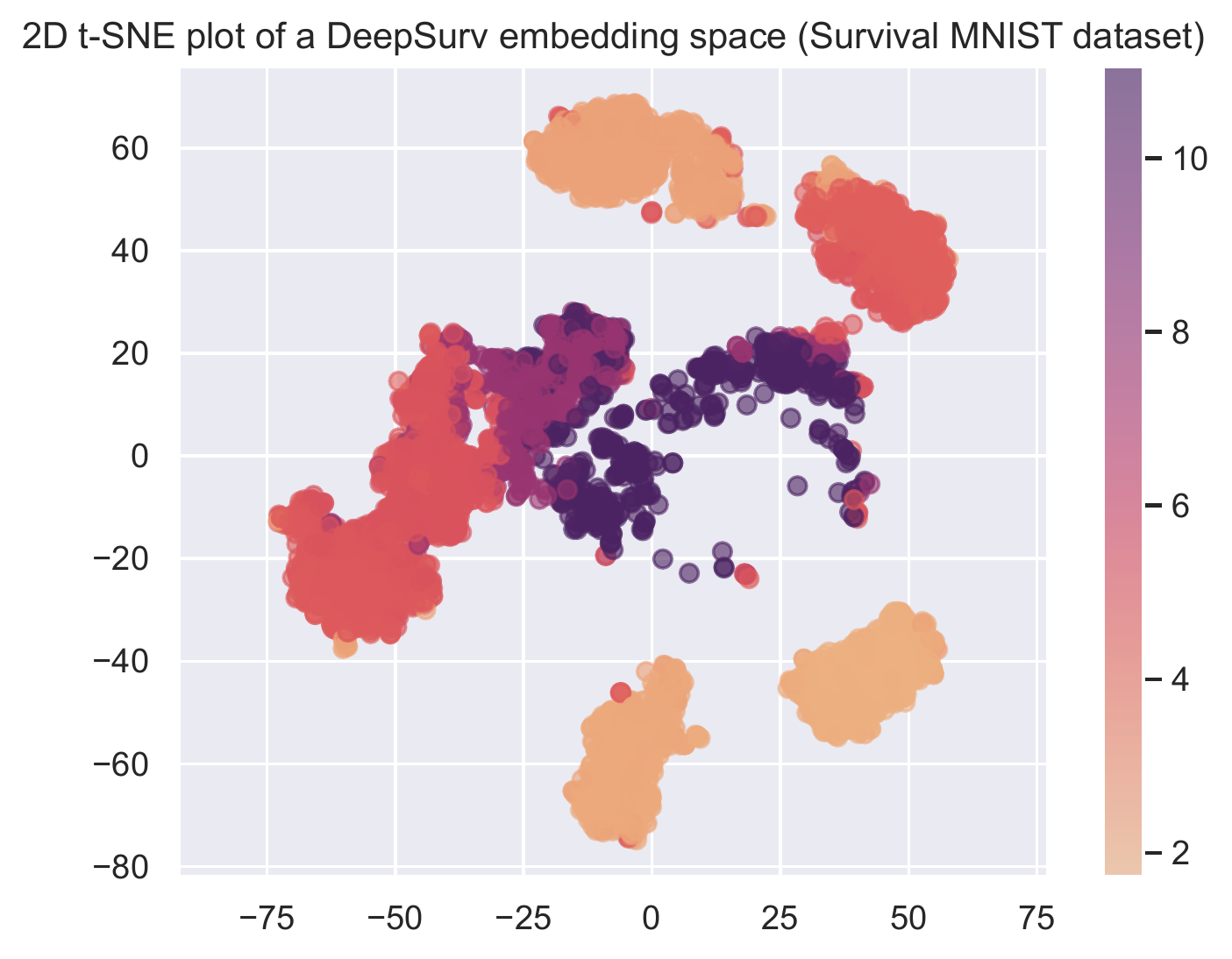}
    \vspace{-1.5em}
    }
    \vspace{-1.5em}
  }
\end{figure*}

We present PCA and t-SNE plots using the baseline visualization strategy described at the end of Section~\ref{sec:intro}. We show plots for the DeepSurv embedding space for the SUPPORT dataset (using the setup in Section~\ref{sec:tabular}/Appendix~\ref{sec:support-additional-results}) in \figureref{fig:baseline-visualizations-support}, the Rotterdam/GBSG datasets (using the setup in Appendix~\ref{sec:rotterdam-gbsg}) in \figureref{fig:baseline-visualizations-rotterdam-gbsg}, and the Survival MNIST dataset (using the setup in Section~\ref{sec:images}/Appendix~\ref{sec:survival-mnist}) in \figureref{fig:baseline-visualizations-survival-mnist}. In particular, the embedding spaces under examination all have a norm~1 constraint. Each scatter plot is made using the visualization data (and not the training data used to train the neural survival analysis model nor the anchor direction estimation data). For each visualization data point $x_i^{\viz}$, we compute its median survival time estimate by looking at the time $t$ where $\widehat{S}(t|x_i^{\viz})$ crosses 1/2 (similar to what we had discussed in Appendix~\ref{sec:handling-many-anchor-directions}), and we color scatter plot points based on these median survival times. From these scatter plots, we can get a rough sense of the geometry of the embedding space. For instance, whereas there are clear clusters of points that show up for Survival MNIST (in fact, one could check that these correspond to different groups of digits; again, as we pointed out in Section~\ref{sec:images}/Appendix~\ref{sec:survival-mnist}, the embedding space does not appear to disentangle all 10 digits neatly), we do not see this clustering behavior for the SUPPORT and Rotterdam/GBSG DeepSurv embedding spaces.

Importantly, as we already pointed out in Section~\ref{sec:intro}, these scatter plots from dimensionality reduction do not tell us how the embedding space relates to raw features. Even PCA, which is easier to interpret than nonlinear dimensionality reduction methods (e.g., t-SNE), does not relate the embedding space to raw features in this setting since PCA here is directly applied to vectors from the embedding space (and not vectors from the raw feature space). While the t-SNE plot for Survival MNIST shows clustering behavior, note that t-SNE itself does not actually estimate cluster assignments for different data points, i.e., t-SNE is inherently \emph{not} a clustering algorithm.

Note that the PCA plots can actually give us a sense of whether information in the embedding space is stored more in magnitudes vs more in angles. As a reminder, Euclidean vectors with norm~1 reside on what is called the ``unit hypersphere'' ${\mathcal{S}^{d-1}\triangleq\{v\in\mathbb{R}^d\text{ s.t.~}\|v\| = 1\}}$. When we take data on the unit hypersphere and plot their 2D PCA plot, the resulting 2D PCA plot will always look like points that are within a 2D circle (since PCA is a linear dimensionality reduction method, it retains the hyperspherical structure but projects down to 2D, where points can be projected inside the circle rather than only along th shell of the circle). This plot could be helpful. We can readily tell if the data appear uniformly distributed over a hypersphere or not. For example, for the SUPPORT dataset's 2D PCA plot in \figureref{fig:baseline-visualizations-support}, the points largely bunch up on one side of a circle, meaning that in the embedding space (that in this case is actually 10-dimensional), the points largely are concentrated around a hyperspherical cap (i.e., the embedding vectors are largely all pointed in a similar direction). In contrast, the 2D PCA plots for the Rotterdam/GBSG datasets (\figureref{fig:baseline-visualizations-rotterdam-gbsg}) and the Survival MNIST dataset (\figureref{fig:baseline-visualizations-survival-mnist}) clearly show more of a circle shape, indicating that the DeepSurv embedding vectors are more uniformly distributed for Rotterdam/GBSG and Survival MNIST (i.e., they have information stored more in angles than in magnitudes) than for \mbox{SUPPORT}.

We remark that it is possible to color the scatter plot points using other quantitative values. For example, we could use the mean (instead of the median) survival time estimate, which corresponds to the area under a data point's predicted survival curve, we could use an indicator value for whether the point is censored or not (to get a sense of whether some parts of the embedding space correspond to more censored points), or we could use cluster labels as estimated using any of the clustering approaches we had used to estimate anchor directions.

\end{document}